\documentclass[twoside]{article}

\usepackage{subcaption}
\usepackage{algorithm}
\usepackage{algorithmic}
\newcommand{\OUTPUT}{\item[\textbf{Output:}]}
\usepackage[utf8]{inputenc} % allow utf-8 input
\usepackage[T1]{fontenc}    % use 8-bit T1 fonts
\usepackage{hyperref}       % hyperlinks
\usepackage{url}            % simple URL typesetting
\usepackage{booktabs}       % professional-quality tables
\usepackage{amsfonts}       % blackboard math symbols
\usepackage{nicefrac}       % compact symbols for 1/2, etc.
\usepackage{microtype}      % microtypography
\usepackage{xcolor}         % colors
\usepackage{graphicx}
\usepackage{amsmath}
\usepackage{amssymb}
\usepackage{amsthm}
\theoremstyle{plain}
\newtheorem{theorem}{Theorem}[section]

\theoremstyle{definition}
\newtheorem{definition}[theorem]{Definition}

\theoremstyle{remark}

\def\X{\mathcal{X}}
\def\F{\mathcal{F}}
\def\DR{\textrm{DR}}
\def\L{\mathcal{L}}
\def\S{\mathcal{S}}
\def\d{\mathcal{\gamma}}
\def\kNN{\textit{k}\textrm{NN}}
\def\PaCMAP{$\text{PaCMAP}_{\textrm{param}\text{ }}$}
\def\UMAP{$\text{UMAP}_{\textrm{param}}$}
\usepackage{bbm}
\usepackage[authoryear]{natbib}
\floatstyle{plain}
\restylefloat{figure}

\usepackage[accepted]{aistats2026}
\begin{document}



\twocolumn[

%\aistatstitle{The Rashomon Effect for Dimension Reduction for Data Visualization}
\aistatstitle{The Rashomon Effect for Visualizing High-Dimensional Data}

\aistatsauthor{ Yiyang Sun$^*$ \And Haiyang Huang$^{*+}$ \And  Gaurav Rajesh Parikh$^*$ \And Cynthia Rudin  }

\aistatsaddress{ Duke University \And  Duke University \And Duke University \And Duke University} ]

    \def\thefootnote{*}\footnotetext{These authors contributed equally to this work}\def\thefootnote{+}\footnotetext{Now at Google.}\def\thefootnote{\arabic{footnote}}

\begin{abstract}
Dimension reduction (DR) is inherently non-unique: multiple embeddings can preserve the structure of high-dimensional data equally well while differing in layout or geometry. In this paper, we formally define the Rashomon set for DR—the collection of `good' embeddings—and show how embracing this multiplicity leads to more powerful and trustworthy representations. Specifically, we pursue three goals. First, we introduce PCA-informed alignment to steer embeddings toward principal components, making axes interpretable without distorting local neighborhoods. Second, we design concept-alignment regularization that aligns an embedding dimension with external knowledge, such as class labels or user-defined concepts. Third, we propose a method to extract common knowledge across the Rashomon set by identifying trustworthy and persistent nearest-neighbor relationships, which we use to construct refined embeddings with improved local structure while preserving global relationships. By moving beyond a single embedding and leveraging the Rashomon set, we provide a flexible framework for building interpretable, robust, and goal-aligned visualizations.
\end{abstract}

\section{Introduction}
\label{sec:introduction}
% - There are a lot of possibly good DRs. Typically, DR methods just produce one of them. This leaves out the uncertainty that is inherent to DR that the user might want to know about and/or leverage. If we could access that uncertainty, what would we do with it? We would (1) use it to incorporate domain knowledge to reduce the uncertainty and get DR that agrees with both data and domain knowledge, and (2) use it to figure out what is not uncertain; what knowledge is common to all or most of the viable DRs.
% This ties into the notion of Rashomon sets, the set of all good models. 

Dimension reduction (DR) is the cornerstone of modern data analysis, visualization, and representation learning, enabling researchers to explore and interpret complex high-dimensional datasets in low-dimensional spaces. A wide range of DR techniques have been developed, from linear methods such as PCA \citep{pca} to non-linear techniques such as t-SNE \citep{tsne}, UMAP \citep{umap_algorithm,umap_parametric}, and PaCMAP \citep{pacmap,paramrepulsor} — that aim to preserve various aspects of the original data structure, such as local and global neighborhoods. A critical limitation of most DR methods is that they return a \textit{single} embedding, often driven by stochastic optimization, heuristic design choices, or algorithmic randomness.

These single output embeddings mask a fundamental fact: there is rarely a unique ``correct'' embedding of high-dimensional data. Instead, many diverse low-dimensional representations can preserve the data structure nearly equally well. Different runs of the same DR algorithm -- with variations in random seed, initialization, bootstrapped samples, or pairwise constraints -- can yield embeddings that differ substantially in global layout or local relationships, even when they achieve similar objective scores or visual quality \citep{kobak2021initialization,pacmap}. Such variability introduces epistemic uncertainty that is seldom quantified or used in downstream analysis.

Importantly, this variability also manifests in the \textit{mobility of clusters} within the embedding space. A cluster that appears compact and well-separated in one embedding might shift position, rotate, or stretch in another -- while still preserving its internal neighborhood structure. These transformations do not necessarily indicate a failure of the DR method to optimize its objective but rather reflect the geometric flexibility inherent to the mapping. In other words, as long as the cluster itself is maintained, it can translate, rotate, or reorient in the embedding space while still potentially yielding a valid outcome of dimension reduction.
% For instance, Figure \ref{fig:boostrap_pacmap_embeddings} has shown that the entire clusters are supposed to drift or reorient with the entire clusters, so long as pairwise proximities within the cluster are maintained. As a result, global positioning—such as the relative layout of clusters—can vary substantially across embeddings without compromising their quality.

% \begin{figure}[ht]
%     \centering
%     \includegraphics[width=0.6\linewidth]{fig/boostrap_embeddings.png}
%     \caption{ParamPaCMAP embeddings clusters aligned differently through boostrap datasets}
%     \label{fig:boostrap_pacmap_embeddings}
% \end{figure}

\begin{figure*}[th]
    \centering
    \includegraphics[width=0.9\linewidth]{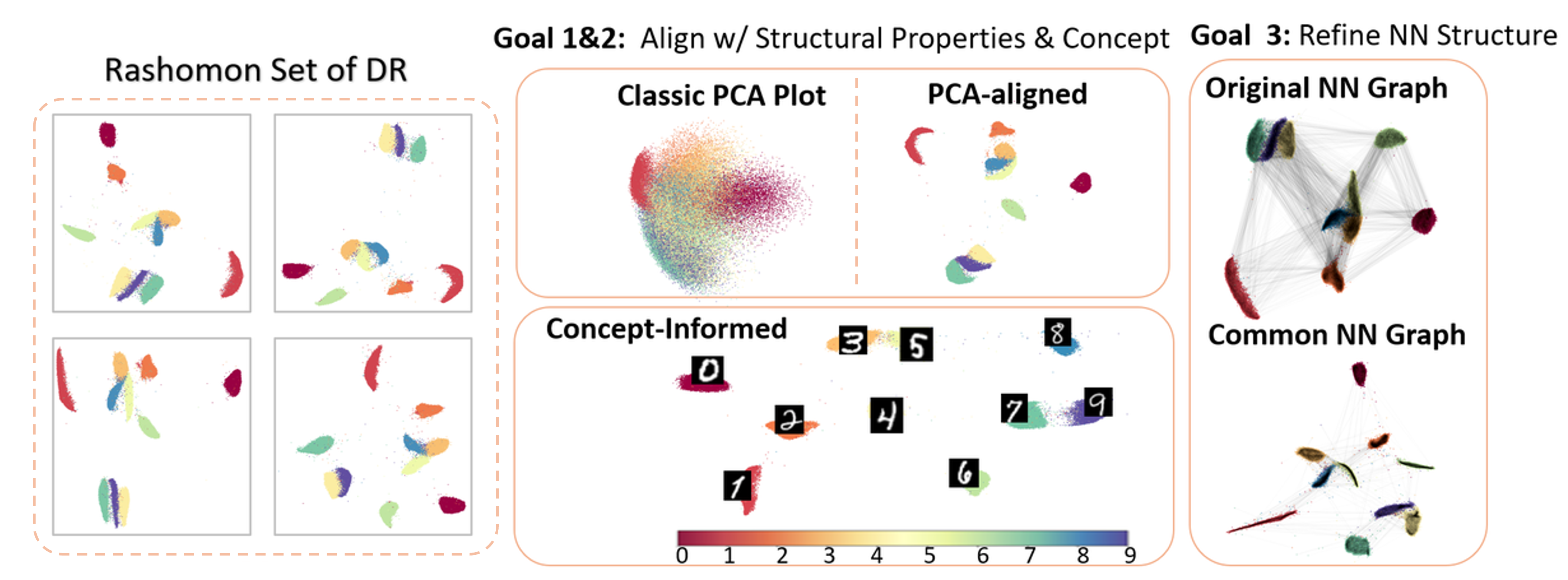}
    \caption{Three goals for generating and exploring the Rashomon set for dimension reduction}
    \label{fig:overview}
    
\end{figure*}

In this work, we \textit{extend the notion of the Rashomon effect from supervised learning \citep{rashomon} to dimension reduction}. Here, the \textit{Rashomon effect} is the phenomenon that many structurally valid embeddings can be found for the same high-dimensional dataset. Even without labels, most DR methods are guided by well-defined objectives -- such as preserving local neighborhoods or global geometry -- and thus allow us to identify a set of embeddings that perform about equally well. Rather than seeing the resulting variability as a nuisance, we instead treat it as a source of opportunity and insight, as we show in Fig.~\ref{fig:overview}. First, we can steer embeddings toward desirable properties, such as \textbf{alignment with principal components} (Goal 1) or \textbf{prior knowledge} (Goal 2), to enhance interpretability without compromising structural fidelity.
Second, by comparing many good embeddings, we can uncover structural relationships -- such as cluster boundaries or neighborhood graphs -- that are stable \textbf{across alignment goals, random seeds, or bootstrapped samples} and that appear in the results of many DR methods (Goal 3). These recurring patterns represent a form of unsupervised consensus, capturing the underlying structure in the data that is robust to algorithmic or sampling choices.

\section{Related Work}

\paragraph{Dimension Reduction}
% Early approaches to DR preserved global geometric structure, as seen in methods like PCA \cite{pca}, MDS \cite{Torgerson52} and NMF \cite{lee1999learning}.
Early dimension reduction methods focused on preserving global geometric structure. MDS \citep{Torgerson52} preserves pairwise distances, while PCA \citep{pca} maintains variance along principal directions, and NMF \citep{lee1999learning} preserves part-based representations. While effective at maintaining the overall data structure, these techniques often fail to capture local neighborhoods and clusters. To address this limitation, \textit{local} DR methods were developed. Examples include Isomap \textcolor{black}{\citep{Tenenbaum00}}, LLE \textcolor{black}{\citep{Roweis00}}, Laplacian Eigenmap \textcolor{black}{\citep{Belkin01}}, and more recent Neighborhood Embedding algorithms such as t-SNE \textcolor{black}{\citep{tsne}}, UMAP \textcolor{black}{\citep{umap_algorithm}}, PaCMAP \textcolor{black}{\citep{pacmap}}, LocalMAP \textcolor{black}{\citep{localmap}}, and many others \citep{Tang16, TriMAP, artemenkov2020ncvis, sarfraz2022hnne, damrich2023from, van2022probabilistic, zu2022spacemap}. To support online and continual learning scenarios, \textit{parametric} DR methods learn a mapping function that approximates the DR objective. Notable examples include Paramt-SNE \textcolor{black}{\citep{van2009learning}}, ParamUMAP \textcolor{black}{\citep{umap_parametric}}, InfoNC-t-SNE \textcolor{black}{\citep{damrich2023from}}, ParamRepulsor \textcolor{black}{\citep{paramrepulsor}} and others \citep{bohm2022unsupervised,moor2020topological,nazari2023geometric}.

We note that all of these algorithms are designed to produce a \textbf{single} embedding that captures the underlying structure of the data. However, in this work, we observe that multiple distinct embeddings can satisfy the same DR objectives, and we show that constructing a diverse set of such valid embeddings can be beneficial for improving the robustness and alignment of the embedding.

\paragraph{Embedding Improvement}
Numerous studies have sought to improve DR results by modifying algorithmic behavior and examining the influence of various components. 
Some focus on parameter tuning
\textcolor{black}{\citep{wattenberg2016use,cao2017automatic,nguyen2019ten,Belkina19OptSNE, kobak2021initialization}}, while others explore the impact of loss functions \textcolor{black}{\citep{pacmap, bohm2022attraction}} and graph reweight/update strategies \textcolor{black}{\citep{colange2020steering,pacmap,dalmia2021clustering,localmap}}. In contrast, our work introduces a novel perspective: we investigate the \textit{invariant} components of embeddings across multiple runs and analyze which elements contribute to preserving DR performance. This insight enables us to create embeddings that %aligned with specific desired features and to 
extract shared structural information, leading to more robust and reliable representations.

\paragraph{Consensus Embeddings}

\textcolor{black}{A consensus embedding combines multiple embeddings of the same dataset into a single, unified embedding that captures shared structural information  \citep{viswanath2012consensus}. Our work on Rashomon DR can help facilitate consensus embeddings by finding many valid embeddings to combine, and we study this in Goal 3.}
%While consensus embedding and Rashomon DR both deal with multiple embeddings of the same data, they differ in emphasis: consensus embedding focuses on fusing given representations, whereas our Rashomon DR framework characterizes the set of valid embeddings, so that we can extract useful embedding or find common knowledge.}

Existing consensus embedding methods, such as Median Consensus Embedding \citep{tomo2025median} and C-LLE \citep{tiwari2008consensus}, predominantly rely on assumptions of linear relationships between input embeddings. We know these assumptions are too strong, and most of them require storing a pairwise $n\times n$ matrix, which is extremely large and difficult to work with. Other approaches—often applied in image segmentation \citep{viswanath2012consensus} or network alignment \citep{li2022consensus}—typically perform fusion through subspace projections or global alignment, implicitly assuming a shared global structure. These assumptions are incompatible with scenarios where embeddings disagree, e.g., where individual clusters may move independently or where local geometries vary across embeddings—common traits within the Rashomon set. Most critically, prior work in consensus embeddings has largely ignored the \textit{trustworthiness of neighborhood graphs}, assuming they are always trustworthy when we know they are not \citep{localmap}. In this work, we aim to extract trustworthy nearest neighbor (NN) relationships from the Rashomon set by identifying the local structure that remains stable across multiple high-quality embeddings. Hence, our method reinterprets consensus not as coordinate fusion but as the extraction of consistent local relationships across embeddings. By doing so, we move toward a structure-aware, graph-level consensus that is better aligned with the goals of both interpretability and robust representation learning.

\section{Defining the Rashomon Set for Dimension Reduction}

% \color{black}
Most DR algorithms lack robustness, often producing different embeddings for the same dataset in different runs \citep{pacmap,kobak2021initialization,umap_algorithm}. This variability poses challenges for interpretation. While prior work has treated such inconsistency as a drawback, we observe that it parallels the concept of the Rashomon set in supervised learning \citep{rashomon}, which captures the existence of multiple models that perform equally well on the same data. Motivated by this connection, we extend the Rashomon set framework to the context of DR. We begin with the conventional definition based on the loss threshold and subsequently adapt it to a DR-specific formulation that considers the preservation of pairwise relationships between data points. In this paper, we primarily use parametric dimensionality reduction (DR) methods as examples since they are easier to implement on large datasets through stochastic gradient descent. However, all of the definitions introduced here are equally applicable to non-parametric DR methods.
% We formally define the Rashomon set for DR. 
%We adopt two distinct but connected perspectives: defining the Rashomon set in terms of a loss function, as well as defining it in terms of the graph structure of the pairwise relationships between points. 

Let $\X = {x_1, \dots, x_n} \subset \mathbb{R}^p$ be a high-dimensional dataset, and let $\F_{\theta}: \mathbb{R}^p \to \mathbb{R}^d$ be a parametric dimension reduction method with learnable parameters $\theta$, such that $\F_{\theta}(\X) = y \in \mathbb{R}^{n \times d}$. The set of all $\F_{\theta}$ functions is denoted by $\F$. Let $\L_{\DR}(\X, \F_\theta)$ denote the loss function that measures the quality of the embedding, specifically measuring how much information $\F_\theta(\X)$ preserves about $\X$.
Let $\Theta$ be the parameter space for $\F_\theta$, and let $\theta^*$ be an optimal or reference solution, i.e.,
$$
\theta^* \in \arg\min_{\theta\in \Theta}\L_{\DR}(\X,\F_\theta).
$$

\begin{definition}[Rashomon set of Dimension Reduction from a Loss Perspective]\label{def:rashomonlossperspective}
We define the Rashomon set of Dimension Reduction from a Loss Perspective, denoted $\mathcal{R}_{\textrm{loss}}(\X, \F_{\theta}, \delta,\L_\DR)$, as the set of all parameter values $\theta$ such that the corresponding embedding $y = \F_{\theta}(\X)$ achieves a loss close to the reference:
$\mathcal{R}_{\textrm{loss}}(\X,\F,\delta, \L_\DR) :=$
$$ \{\theta\in\Theta | \L_{\DR}(\X,\F_\theta) \le \L_{\DR}(\X,\F_{\theta^*})+ \delta\}.$$
Embeddings within this set are all produced by models whose loss values are within a tolerance $\delta$ of the minimum.
\end{definition}

While most widely used DR methods are framed as loss minimization problems \citep{umap_algorithm,tsne,pacmap,artemenkov2020ncvis}, their training objectives can also be interpreted through the lens of nearest neighbor graph preservation. Specifically, these methods often define their loss functions based on the high-dimensional nearest neighbor graph, aiming to minimize discrepancies between pairwise affinities in the original and embedded spaces. Consequently, the optimization process can be viewed as implicitly selecting a graph structure. In this view, a Rashomon set based on loss, which comprises embeddings with low loss, can be interpreted as preserving approximately the same high-dimensional neighbor graph. Embeddings within the same Rashomon set are likely to induce mostly similar weighted graphs over a fixed edge set. This insight motivates a graph-theoretic reinterpretation of the Rashomon set in the context of DR, introduced below.

Suppose we have a candidate set of embeddings, denoted as $\S=\{y^{(1)},...,y^{(|\S|)}\}$, which is of size $|\S|$. $\S$ can consist of embeddings from the Rashomon set from a loss perspective from Definition \ref{def:rashomonlossperspective}.  
Define the k-nearest neighbors of a point $i$ in $\X$ by $\kNN(i)$ \textcolor{black}{; these are asymmetric, i.e., $i \in \kNN(j)$ does not imply $j \in \kNN(i)$. We note that high-dimensional Euclidean nearest neighbors may not correspond to neighbors along the data manifold. However, a sufficiently large initial $k$ ensures that true manifold neighbors are included within the candidate set, even if some false positives are also present. The reduction of false positives is handled subsequently through our scoring-based selection in Section~\ref{sec:common_knowledge}, which identifies the most consistently close pairs across multiple embeddings.}

% using the high-dimensional Euclidean distance, which does not necessarily reflect distances along the data manifold. Nevertheless, points along the data manifold will be included in the \kNN{} as long as $k$ is sufficiently large. (In Section \ref{sec:common_knowledge}, we show how to improve the neighborhood construction.)
To evaluate how an embedding $y$ reproduces the graph structure, we define $W^y \in \mathbb{R}^{n \times n}$ as a weight matrix that shows the strength between all pairs of points. 
\textcolor{black}{The construction of $W^y$ is a hybrid of both the high-dimensional and low-dimensional spaces: the adjacency structure (which pairs $(i,j)$ are considered neighbors) is determined by $\kNN$ in the high-dimensional space $\X$, while the edge weights are computed from distances in the low-dimensional embedding $y$.}
If $i\notin \kNN(j)$, then the entry of $W^y_{ij}$ is 1. 
% Otherwise, the entry $W^y_{ij}$, for data points $i$ and $j$, is 
Otherwise, entry $W^y_{ij}$ is
\[W^y_{ij} = \frac{(\|y_i-y_j\|_2^2 + \d)}{(\|y_i-y_j\|_2^2+\d)+1},\] where $\d$ is the scaling constant for numerical stability. Nearer points have lower weight, and farther points have higher weight. This function is designed to saturate as points get farther from each other. \textcolor{black}{Thus, $W^y$ depends on the specific embedding $y$ being evaluated: while the adjacency list is fixed by the high-dimensional $\kNN$ graph of $\X$, the weights vary with the candidate low-dimensional embedding.}

\begin{definition}[Rashomon set of Dimension Reduction from Graph Perspective]
\label{def:rashomongraphperspective}
We define the Rashomon set of Dimension Reduction from a Graph Perspective, denoted $\mathcal{R}_{\textrm{graph}}(\mathcal{Y}; \kNN;\S)$, as the set of embeddings in $\mathcal{Y}$ that induce approximately the same $W$ among embeddings in $\S$, that is:
{
\small
\[
\mathcal{R}_{\textrm{graph}}(\mathcal{Y}; \kNN;S) := \left\{ y \in \mathcal{Y} \left|\frac{1}{|S|}\sum_{i=1}^{|S|}d(W^{y},W^{{y^{(i)}}})\right. \leq \epsilon \right\},
\]
}
where $W^{y}$ is the induced low-dimensional edge weight matrix constructed using \kNN, $W^{y^{(i)}}$ is the edge weight matrix of embedding $y^{(i)} \in \mathcal{S}$ constructed using \kNN, and $d(\cdot,\cdot)$ is a distance between two weight matrices defined below. 
\end{definition}

\begin{definition}[Soft Jaccard Distance Between Weighted Matrices]\label{def:soft_jaccard}
Let $W^{y^1}, W^{y^2} \in \mathbb{R}^{n \times n}$ be two weighted matrices corresponding to embeddings $y^1, y^2 \in \mathcal{Y}$. We define the \emph{Soft Jaccard Distance} as:
{\small
$$
d(W^{y^1}, W^{y^2}) := 1-
\frac{\sum_{i,j} {\min \left( W^{y^1}_{ij}, W^{y^2}_{ij} \right)}/({W^{y^1}_{ij} + W^{y^2}_{ij}})}
     {\sum_{i,j} {\max \left( W^{y^1}_{ij}, W^{y^2}_{ij} \right)}/
     ({W^{y^1}_{ij} + W^{y^2}_{ij})}}.
$$}
This value lies within the range of $[0,1]$, where it would be 0 if the two weight matrices are identical and 1 if they are completely different. To ensure scale invariance and comparability across embeddings, each embedding $y \in \mathcal{Y}$ is first standardized such that $\frac{1}{n} \sum_{i=1}^n y_i = 0$, and $\frac{1}{n} \sum_{i=1}^n \|y_i\|_2^2 = 1$, which ensures that distance magnitudes are comparable across embeddings when computing edge weights.
\end{definition}

In Section \ref{exp1} and Appendix \ref{app:missing_ratio}, we show
that in practice, the two Rashomon definitions exhibit similar behavior when $S$ consists of low-loss embeddings. It is also important to note that several evaluation metrics (mentioned in Section \ref{sec:experiments} and Appendix \ref{app:metrics}) remain largely consistent within the Rashomon set. This consistency suggests that our Rashomon sets are stable across both definitions and evaluation criteria.

\section{Knowledge Alignment for DR within the Rashomon Set}\label{sec:alignment}

% % To explore how domain constraints or auxiliary objectives affect the Rashomon set, we augment the DR loss function with various task-specific regularization terms, depending on the experimental setting:
To investigate how domain knowledge can guide the selection of embeddings within the Rashomon set, we introduce alignment-based regularization terms that nudge a base DR method toward interpretable structure without significantly altering DR loss ($\L_\DR$). Our goal is to identify embeddings that remain within the Rashomon set\textemdash i.e., that preserve $\L_\DR$\textemdash while also aligning with external knowledge or structure. Here we develop two useful kinds of alignment.

\subsection{PCA-Informed Alignment}\label{subsec:pca_aligned}

% % In this setup, we encourage the embedding space to align the relative geometric structure of samples . 
In this setup, we choose the DR embedding to align with PCA embeddings while preserving local structure, which means that the embeddings' axes should be as similar as possible to the first two principal components.  To formalize this, we introduce a term that encourages the directions between point pairs in the DR embedding to align with those in the PCA embedding. 

% \textcolor{black}{In this section, we assume $d = 2$ for the purpose of clear 2D visualization; the formulation naturally generalizes to $d > 2$ by aligning the embedding dimensions with the corresponding top-$d$ principal components. We also note that PCA embeddings are only unique up to axis flips (sign ambiguity), implying that part of the alignment is inherently arbitrary. In practice, we fix a canonical sign convention before applying the alignment.}

Consider a pair of points $(i,j)$ that are not identified as nearest neighbor pairs, $i \notin \kNN (j)$ (i.e., these could be further pairs in the PaCMAP algorithm \citep{pacmap}, non-nearest neighbor pairs in other dimension reduction methods, or contrastive neighbor methods \citep{paramrepulsor,artemenkov2020ncvis}).
\textcolor{black}{We exclude nearest-neighbor pairs from this alignment term because $\L_\DR$ already
governs the local structure through NN relationships. Including NN pairs in the PCA alignment would create
a competing objective that could distort local neighborhoods. By restricting alignment to non-NN (more distant) pairs, we steer only the global layout toward the principal components, so that the local structure will not be impacted. This separation of concerns is what allows the embedding to remain within the Rashomon set while gaining interpretable axes.} For non-neighbor points $i$ and $j$, let $y_1,y_2$ be their coordinates in the DR embedding $y$, and let  $y_{\textrm{PCA},1},y_{\textrm{PCA},2}$ be their coordinates in the PCA embedding $y_{\textrm{PCA}}$. The goal of this regularization term is to align the directions of $y_1 - y_2$ and $y_{\textrm{PCA},1} - y_{\textrm{PCA},2}$ by minimizing the cosine similarity between those two vectors. Therefore, the total loss is:
{\small
\begin{align*}\small &\L_\text{total} = \L_{\DR} + \\ &\lambda_{\textrm{PCA}} \cdot \mathbb{E}_{i\notin \kNN(j)}\left[\left(1-\frac{\langle y_1 - y_2 , y_{\textrm{PCA},1} - y_{\textrm{PCA},2}\rangle}{\|y_1 - y_2\|_2\cdot\|y_{\textrm{PCA},1} - y_{\textrm{PCA},2}\|_2}\right)^2\right] \end{align*}}
where $\lambda_\textrm{PCA}$ controls the strength of the alignment constraint. This is studied in Section \ref{subsec:pca_aligned} and illustrated in Figure \ref{fig:overview}. An example of aligning PaCMAP with PCA is shown in Figure \ref{fig:pca_aligned_usps_pacmap} for the USPS dataset \citep{USPS}, which shows that the USPS embedding is flipped and rotated toward the principal components without harming the original local structure.

\begin{figure}[ht]
    \centering
    \includegraphics[width=\linewidth]{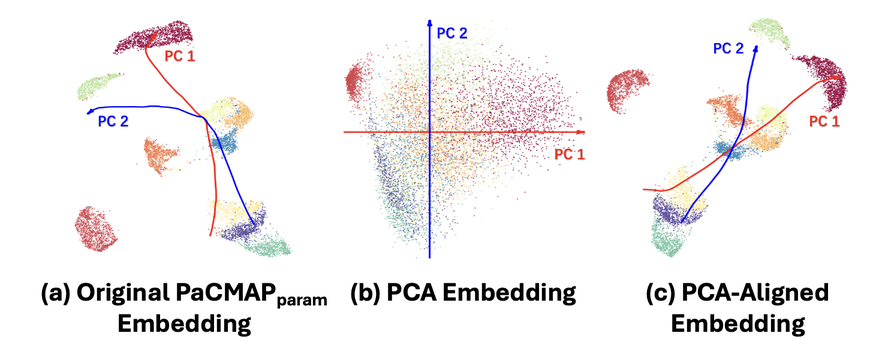}
    \caption{\PaCMAP embedding with and without PCA-Informed Alignment. \textcolor{black}{The colored curves overlaid on the embeddings are generated by applying the learned parametric DR mapping to points sampled along the first two principal component directions in the original high-dimensional space, thereby visualizing how the DR mapping transforms the PCA axes.}}
    \label{fig:pca_aligned_usps_pacmap}
\end{figure}

\subsection{Concept-Informed Alignment}\label{subsubsec:axis_aligned}

Here, we encourage the DR embedding to align with a given concept along a designated axis, which is typically the horizontal axis of the embedding space. To do this, we define an axis-informed objective term that encourages this alignment. The concept can arise from any given function, including labels from supervised or semi-supervised problems.  
Specifically, for each non-nearest-neighbor pair $(i, j)$ where both concept labels $l_i, l_j$ are available, we compute the standardized difference along the horizontal axis of the embedding, denoted by $\tilde{y}_{i,1} - \tilde{y}_{j,1}$, and compare it to the difference between their concept labels, $\tilde{l}_i - \tilde{l}_j$. The axis-informed loss penalizes squared deviations between these two quantities:
% Specifically, for each non-nearest-neighbor pair, we compute the vector differences along the x-axis and standardize these differences. Simultaneously, we compute the standardized label differences between the corresponding pairs. The loss penalizes the squared deviation between the standardized x-axis differences and standardized label differences, only considering pairs where both labels are valid:
{\small\begin{equation*} \L_\text{total} = \L_\text{DR} + \lambda_\text{Axis}\mathbb{E}_{i\notin \kNN(j)}\left[\left(\tilde{y}_{i,1} - \tilde{y}_{j,1} - (\tilde{l}_i - \tilde{l}_j)\right)^2\right] \end{equation*}}
$\lambda_\text{Axis}$ controls the strength of the alignment with concept label differences. This formulation encourages the embedding to preserve meaningful variation along an interpretable direction while remaining flexible in other dimensions. This technique is studied in Section \ref{exp1} and illustrated in Figure \ref{fig:overview}. An example aligning PaCMAP with a label is shown in Figure \ref{fig:fmnist_aligned} for the Fashion-MNIST dataset \citep{fmnist}, where labels were arranged into a pre-defined order: fashion items are arranged from head to toe.

\section{Common Knowledge Extraction from the Rashomon Set}
\label{sec:common_knowledge}

While each embedding in the Rashomon set satisfies structural constraints based on its loss function, not all local relationships within its embeddings are equally reliable. Some nearest-neighbor (NN) pairs are consistently close across embeddings, while others vary significantly in distance or may not be close in any good embedding. This observation motivates a deeper question: can we extract trustworthy structural information from the Rashomon set by identifying NN relationships that are stable and reproducible? Taking this idea one step further, can we leverage the set of trustworthy pairs to construct an embedding that better captures the essential structure of the data?

Building on the findings of \citet{localmap}, which identified the presence of false negative NN pairs in embedding neighborhoods that can negatively affect cluster identification, we aim to refine the selection of NN pairs for use in DR methods. Specifically, we leverage embeddings from the Rashomon set to assess the stability of these pairs across reasonable embeddings. We aim to identify pairs that \textit{consistently} appear as neighbors.

% For variety, we choose the set of ``good'' embeddings to come from a variety of popular DR algorithms. 

Here, we provide theoretical proof to show that stable neighbors are more likely to be trustworthy.
\textcolor{black}{Let $\mathcal{D}$ be the distribution of valid embeddings within the Rashomon set. Let $y \sim \mathcal{D}$ be a random embedding drawn from this set, and let $d_{ij}^y = \|y_i - y_j\|_2$ be the pairwise distance in embedding $y$.}
\textcolor{black}{We define fixed \emph{population} parameters for each nearest-neighbor pair $(i,j)$, $j \in \kNN(i)$:}
{\small\begin{align}
    \textcolor{black}{\mu^*_{ij}} &\textcolor{black}{= \text{Median}_{y\sim\mathcal{D}}[d^y_{ij}],} \\
    \textcolor{black}{\sigma^*_{ij}} &\textcolor{black}{= \sqrt{\mathbb{E}_{y\sim\mathcal{D}}[(d^y_{ij} - \mu^*_{ij})^2]},} \\
    \textcolor{black}{m^*_{ij}} &\textcolor{black}{= \sup_{y \in \mathcal{D}} d^y_{ij},}\\
    \textcolor{black}{M^*} &\textcolor{black}{= \text{Median}_{y\sim\mathcal{D},\, j\in\kNN(i)}[d^y_{ij}].}
    \label{eq:pop_stats}
\end{align}}
\textcolor{black}{Using these fixed population constants, we define the scoring function:}
{\small\begin{align}
    \textcolor{black}{\Psi_{ij}(y)} &\textcolor{black}{=
    \begin{cases}
        d^y_{ij} + \lambda \cdot \sigma^*_{ij}, & \text{if } d^y_{ij} \le M^* \\
        m^*_{ij} + \lambda \cdot \sigma^*_{ij}, & \text{otherwise}
    \end{cases}}
    \label{eq:edge_importance}
\end{align}}
\textcolor{black}{where $\lambda \ge 0$ is a penalty coefficient. This quantity acts as an upper confidence bound for the distance, penalizing pairs with high variance across embeddings.  Since $\Psi_{ij}(y)$ depends on the random sample $y$ through $d^y_{ij}$ while using only fixed population constants ($\mu^*_{ij}, \sigma^*_{ij}, m^*_{ij}, M^*$), it provides i.i.d.\ scores across independent embeddings.}
% where $M_{\text{global}}^{(T)}$ is the global median distance across all embeddings $\{y^{(t)}\}_t$ and nearest neighbor pairs $j\in\kNN(i)$, $ \tilde \mu_{ij}^{(T)}$, $m_{ij}^{(T)}$ and $\tilde \sigma_{ij}^{(T)}$ are the pair median, max and robust variance of distance across all embeddings.  $ \tilde{\sigma}_{ij}^{(T)}$ evaluates the stability of pairwise distance for nearby points, while $m_{ij}^{(T)}$ is a conservative (larger) upper confidence bound for points that are unlikely to be true nearest neighbor pairs. 
If $\Psi$ is low, the pair is most likely a trustworthy (stable and close) pair within the \kNN{} of $i$. In practice, we do not have access to the infinite distribution $\mathcal{D}$. Therefore, we use a \emph{calibration-test split}: population parameters ($\mu^*_{ij}, \sigma^*_{ij}, m^*_{ij}, M^*$) are estimated from a separate calibration set of embeddings, and the scoring is then applied to an independent test set of embeddings. By the law of large numbers, the empirical estimates converge to their population counterparts as the calibration set size increases. To operationalize this approach, we propose two ways of finding the $k$ most trustworthy neighbors for each point $i$ (details are in Appendix \ref{app:algo}): \textbf{Average Rank Approach} and \textbf{Mean Score Approach}. The first one ranks the neighbors $j$ by their scores in ascending order and takes the average rank, whereas the second one uses the average raw score directly. In the experiments and theorem, we have used the \textbf{Mean Score Approach} as the default, as it directly connects to the theorem and behaves similarly to the rank variant.

\textcolor{black}{We emphasize that the notion of trustworthiness here depends on \emph{both} the high-dimensional space and the embeddings, not on the embeddings alone. The candidate neighbor pairs are defined by the \kNN{} graph in the original high-dimensional space $\X$, and the scoring function $\Psi$ then evaluates whether these high-dimensional neighbors are \emph{consistently} close across multiple low-dimensional embeddings. If an embedding method maps true high-dimensional neighbors to distant locations, those pairs will exhibit high variance and large $\Psi$ scores, and will therefore \emph{not} be deemed trustworthy.}

\begin{theorem}
\label{theo:conv}
\textcolor{black}{Let $(i, j^*, j')$ be a triplet of points, and let $\{y^{(1)}, \dots, y^{(T)}\}$ be $T$ i.i.d. low-dimensional embeddings sampled from $\mathcal{D}$. Define the margin variable for the $t$-th embedding based on the population scoring function $\Psi$:}{\small$$
\textcolor{black}{Z^{(t)}_{i,j^*,j'} := \Psi_{ij^*}(y^{(t)}) - \Psi_{ij'}(y^{(t)}).}
$$}
\textcolor{black}{Since $\Psi$ depends on the random sample $y^{(t)}$ through $d^{y^{(t)}}_{ij}$ while using only fixed population constants, and the $y^{(t)}$ are i.i.d., the $Z^{(t)}_{i,j^*,j'}$ are strictly i.i.d.\ real-valued random variables. Let $\Delta = \mathbb{E}[Z^{(t)}_{i,j^*,j'}]$ be the true expected margin, $V = \text{Var}(Z^{(t)}_{i,j^*,j'})$, and $B$ a constant such that $Z^{(t)}_{i,j^*,j'} \le B$ almost surely. Define the empirical margin:}
{\small$$\textcolor{black}{\hat{\Delta}^{(T)} = \frac{1}{T}\sum_{t=1}^T Z^{(t)}_{i,j^*,j'}.}$$}
\textcolor{black}{Then with probability at least $1 - \delta$:}
{\small
\[
\textcolor{black}{\Delta \le \hat{\Delta}^{(T)} + \frac{B\log(1/\delta)}{3T} + \sqrt{\frac{2V\log(1/\delta)}{T}}.}
\]}
\textcolor{black}{We say that $j^*$ is strictly more trustworthy than $j'$ if $\hat{\Delta}^{(T)} + \frac{B\log(1/\delta)}{3T} + \sqrt{\frac{2V\log(1/\delta)}{T}} < 0$, which certifies $\Delta < 0$ with high probability.}
\end{theorem}

% \begin{proof}
% The sequence $\{Z^{(t)}\}_{t=1}^T$ consists of i.i.d. Bernoulli random variables with mean $p = \mathbb{E}[Z^{(t)}]$ and variance $\sigma^2 = p(1 - p)$. Since $Z^{(t)} \in [0,1]$, Bernstein's inequality for bounded i.i.d. random variables gives:
% \[
% \mathbb{P}\left( |\hat{p}_T - p| \ge \epsilon \right) \le 2 \exp\left( -\frac{T \epsilon^2}{2 \sigma^2 + \frac{2}{3} \epsilon} \right).
% \]

% Solving for $\epsilon$ in terms of confidence level $\delta$ gives the desired bound:
% \[
% |\hat{p}_T - p| \le \sqrt{ \frac{2 p(1 - p) \log(2/\delta)}{T} } + \frac{2 \log(2/\delta)}{3T}
% \]
% with probability at least $1 - \delta$.

% Since we assume $p > 0.5$, the law of large numbers implies that $\hat{p}_T \to p$ as $T \to \infty$. Thus, $\hat{p}_T > 0.5$ with high probability for sufficiently large $T$.
% \end{proof}

% \color{black}

Theorem~\ref{theo:conv} shows that as we sample nearest-neighbor candidates from more embeddings, trustworthy pairs with smaller $\Psi_{ij}^{(T)}$ are likely to be sampled. Based on this finding, we propose Algorithm~\ref{alg:refined_embedding_rank} and \ref{alg:refined_embedding_score} (see details in App.~\ref{app:algo}) to create a refined \kNN{} graph.
% Algorithm\ref{alg:refined_embedding}, which selects the lowest $k$ pairs for each point $i$ and create a refined \kNN{} graph for DR, has been proposed based on the defined scoring function.
% With this theorem, we are able to see that more pairs with lower scores are selected. From a practical perspective, using the scoring metric $\Psi$,
% we sample several reliable nearest-neighbor candidates for each data point and include them in our DR algorithms based on their default setup.
In Section~\ref{sub:experiment_common_knowledge}, we will show the refined graph generates better embeddings, accomplishing our Goal 3. It is worth noting that only a small subset of the Rashomon set (as few as five embeddings) is sufficient to achieve a better consensus embedding.
% These neighbors refine the embeddings, as shown in Section \ref{sub:experiment_common_knowledge}, 
% illustrated in Goal 3 in Figure \ref{fig:overview}.

% \textcolor{black}{To clarify the relationship between this section and Section~\ref{sec:alignment}: Section~\ref{sec:alignment} introduces alignment-based regularization to steer individual embeddings within the Rashomon set toward interpretable structure (Goals~1 and~2). The present section builds on top of this by leveraging the \emph{collection} of aligned and unaligned embeddings to extract common structural knowledge (Goal~3). Theorem~\ref{theo:conv} provides theoretical justification for the scoring function $\Psi$ used in Algorithm~\ref{alg:refined_embedding_rank} and \ref{alg:refined_embedding_score} (detailed in Appendix~\ref{app:algo}): it guarantees that with enough embeddings, the empirical ranking of neighbor pairs converges to the true ranking, so that genuinely trustworthy neighbors are reliably identified. The algorithm then selects the top-$k$ most trustworthy neighbors per point and uses them to construct a refined \kNN{} graph, which in turn produces an improved embedding.}

\section{Experiments}
\label{sec:experiments}

% Our setup is as follows.

\paragraph{Datasets}

We conduct experiments across a diverse set of datasets to assess the consistency and reliability of embeddings created from the Rashomon set. For image data, we include MNIST \citep{lecun2010mnist}, F-MNIST \citep{fmnist} and COIL-20  \citep{coil20}. From computational biology, we use several single-cell RNA-sequencing (scRNA-seq) datasets from studies \citep{kang2018multiplexed,stuart2019comprehensive,human_cortex_data,neurips2021_data}, preprocessed following the protocol in \citep{townes2019feature}. Finally, to study embedding stability under controlled structural variations, we include 3D point cloud datasets with known geometry, such as Mammoth \citep{SmithsonianMammoth} and Airplane \citep{airplane}. See Appendix \ref{app:datasets} for more details.

\paragraph{Algorithms}
We conduct experiments using five widely-used parametric DR methods: ParamUMAP (\UMAP) \citep{umap_parametric}, Parametric Info-NC-t-SNE (\texttt{InfoNCE}) \citep{oord2018representation}, Parametric Neg-t-SNE (\texttt{NegtSNE}) \citep{damrich2023from}, Parametric NCVis (\texttt{NCVis}) \citep{damrich2023from,gutmann2010noise} and Parametric PaCMAP (\PaCMAP) \citep{paramrepulsor}.
% We use parametric PaCMAP(\PaCMAP) for the majority of our experiments \citep{paramrepulsor}. We include other parametric neighborhood embedding algorithms: ParamUMAP (\UMAP) \citep{sainburg2021parametric}, Parametric Info-NC-t-SNE (\texttt{InfoNCE}) \citep{oord2018representation}, Parametric Neg-t-SNE (\texttt{NegtSNE}) \citep{damrich2023from}, and Parametric NCVis (\texttt{NCVis}) \citep{damrich2023from,gutmann2010noise}. 
Most implementations are from the \textit{contrastive-ne} \citep{contrastivenegithub}, with modifications to incorporate the additional terms introduced in Sections \ref{subsec:pca_aligned} and \ref{subsubsec:axis_aligned}. Qualitative results for \PaCMAP are presented in the main text, while visualizations for the remaining methods are provided in the appendix. For all parametric methods, we use their default hyperparameters. Since no method-to-method comparison is conducted, keeping the default settings provides a fair basis for evaluation. 

% \textcolor{black}{[Not Done Yet]} To show the performance of our method with, we have compared our common knowledge extraction method with other consensus embedding methods like median consensus embeddings \citep{tomo2025median} and see how the performance behave differently.

\paragraph{Rashomon set construction and evaluation} For each DR algorithm, we construct a Rashomon set comprised of 235 embeddings per task. This set is generated by producing 5 embeddings using distinct random seeds across 47 different label weight configurations ($\lambda_{\textrm{PCA}}$ and $\lambda_{\textrm{Axis}}$ from Section \ref{sec:alignment}). For concept-informed DR, we used the dataset's class labels as the concept that the embedding aligns along. We then analyze the loss curves corresponding to varying $\lambda$ values and establish a cutoff threshold that identifies embeddings with comparable performance. Following \citet{rashomon}, embeddings falling within this threshold are considered part of the Rashomon set and others are excluded. Subsequently, we utilize these selected embeddings to derive a common nearest-neighbor (NN) graph, as detailed in Section \ref{sec:common_knowledge}. This approach identifies structural patterns that are consistently preserved across high-performing embeddings. A detailed process for the experiment is shown in Appendix \ref{app:detailed_experiment}. In the experiment setup, we set the missingness of the concept labels (see Appendix \ref{app:datasets}) as 90\%, which limits the amount of prior knowledge used. Additional experiments in Appendix \ref{app:missing_ratio} show how the embedding quality is influenced by the concept label weight $\lambda_{\text{Axis}}$ and missingness ratio.

The performance of each embedding within the Rashomon set is evaluated using a suite of supervised and unsupervised metrics that assess both local and global structure preservation that has been used in previous work \citep{umap_algorithm, Tang16, paramrepulsor,Linderman19FItSNE}. To further assess the alignment between the DR embeddings and the principal components obtained via Principal Component Analysis (PCA), we introduce the Triplet PCA Score. This metric evaluates the consistency of triplet relationships between the DR embeddings and their corresponding PCA representations, providing insights into the preservation of global structure. A more detailed description of the calculation for each metric is in Appendix \ref{app:metrics}.

\subsection{Embedding Alignment Enhances Global Structure without Compromising Local Structure}
\label{exp1}

\paragraph{PCA-informed alignment.} Here, we examine how our PCA alignment term can align clusters globally without disrupting local structure, as quantified by the soft Jaccard index (see Definition \ref{def:soft_jaccard}). Figure \ref{fig:mnist_comparison} shows that the MNIST embeddings have been effectively aligned to reflect the relative positions observed in the standard PCA embeddings visually and quantitatively. For all the datasets, the experimental results are shown in Table \ref{tab:triplet_pca}.
\begin{figure*}[ht]
    \centering
    \includegraphics[width=0.8\linewidth]{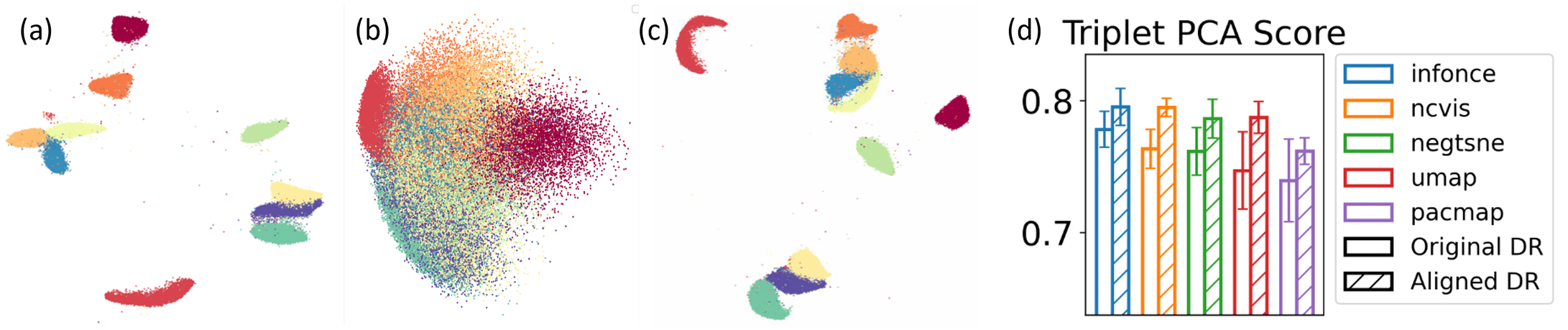}
    \caption{(a) MNIST PaCMAP param embedding, (b) PCA embedding, (c) PCA-informed embedding with $\lambda_{\text{PCA}}$ set to be $ 0.1$. It is nicely aligned with the first two principal components while capturing the detailed cluster structure. (d) Triplet PCA score has improved after PCA-informed alignment.}
    \label{fig:mnist_comparison}

\end{figure*}
% \begin{figure*}[ht!]
%     \centering
%     \includegraphics[width=0.80\linewidth]{fig/metric_comparison_MNIST_pca_new2.png}
%     \caption{Evaluation metrics and losses for MNIST dataset before and after PCA-alignment. The embedding quality does not change substantially as a result of the alignment.
%     }
%     \label{fig:aligned_mnist_metrics}
%     -5mm}
% \end{figure*}

\begin{table*}[ht]
\caption{Triplet PCA Score for embeddings before and after PCA-alignment within different datasets, \textcolor{orange}{orange} are the scores that slightly decrease and the \textcolor{teal}{teal} are the scores that improve. The almost uniform improvement indicates better alignment with the principal components; the fact that the improvements are small is consistent with preservation of local structure, i.e., staying within the Rashomon set.}
\label{tab:triplet_pca}
\centering
\setlength\tabcolsep{1.5pt}
\scalebox{0.93}{
\begin{tabular}{l|ccc|ccc|ccc|ccc|ccc}
\toprule
 & \multicolumn{3}{c}{\UMAP} & \multicolumn{3}{c}{\texttt{NCVis}} & \multicolumn{3}{c}{\texttt{InfoNCE}} & \multicolumn{3}{c}{\texttt{NegtSNE}} & \multicolumn{3}{c}{\PaCMAP} \\
Dataset & Diff & Before &After & Diff & Before &After & Diff & Before &After & Diff & Before & After & Diff & Before & After \\
\midrule
AirPlane & \textcolor{teal}{0.07} & 0.67 & 0.73 & \textcolor{teal}{0.23} & 0.71 & 0.93 & \textcolor{teal}{0.28} & 0.65 & 0.93 & \textcolor{teal}{0.19} & 0.73 & 0.92 & \textcolor{teal}{0.12} & 0.77 & 0.89 \\
COIL20 & \textcolor{teal}{0.12} & 0.68 & 0.8 & \textcolor{teal}{0.11} & 0.68 & 0.79 & \textcolor{teal}{0.10} & 0.73 & 0.83 & \textcolor{teal}{0.12} & 0.68 & 0.8 & \textcolor{teal}{0.05} & 0.74 & 0.79 \\
FMNIST & \textcolor{teal}{0.00} & 0.87 & 0.87 & \textcolor{teal}{0.00} & 0.87 & 0.88 & \textcolor{teal}{0.01} & 0.87 & 0.88 & \textcolor{teal}{0.00} & 0.87 & 0.87 & \textcolor{teal}{0.02} & 0.83 & 0.85 \\
MNIST & \textcolor{teal}{0.04} & 0.75 & 0.79 & \textcolor{teal}{0.03} & 0.76 & 0.79 & \textcolor{teal}{0.02} & 0.78 & 0.8 & \textcolor{teal}{0.02} & 0.76 & 0.79 & \textcolor{teal}{0.02} & 0.74 & 0.76 \\
Mammoth & \textcolor{teal}{0.02} & 0.91 & 0.93 & \textcolor{teal}{0.02} & 0.92 & 0.93 & \textcolor{teal}{0.01} & 0.95 & 0.96 & \textcolor{orange}{-0.01} & 0.91 & 0.9 & \textcolor{orange}{-0.03} & 0.92 & 0.89 \\
USPS & \textcolor{teal}{0.03} & 0.82 & 0.85 & \textcolor{orange}{-0.01} & 0.87 & 0.86 & \textcolor{teal}{0.04} & 0.83 & 0.87 & \textcolor{teal}{0.00} & 0.85 & 0.85 & \textcolor{teal}{0.05} & 0.82 & 0.87 \\
Cortx & \textcolor{teal}{0.01} & 0.85 & 0.86 & \textcolor{teal}{0.02} & 0.84 & 0.87 & \textcolor{teal}{0.06} & 0.86 & 0.92 & \textcolor{teal}{0.01} & 0.86 & 0.88 & \textcolor{teal}{0.02} & 0.86 & 0.88 \\
Kang & \textcolor{teal}{0.01} & 0.61 & 0.62 & \textcolor{teal}{0.01} & 0.63 & 0.64 & \textcolor{teal}{0.01} & 0.64 & 0.65 & \textcolor{teal}{0.01} & 0.63 & 0.64 & \textcolor{orange}{-0.01} & 0.66 & 0.65 \\
CBMC & \textcolor{teal}{0.02} & 0.61 & 0.64 & \textcolor{teal}{0.13} & 0.61 & 0.74 & \textcolor{teal}{0.04} & 0.69 & 0.73 & \textcolor{teal}{0.07} & 0.61 & 0.69 & \textcolor{teal}{0.02} & 0.59 & 0.61 \\
Stuart & \textcolor{teal}{0.02} & 0.59 & 0.6 & \textcolor{teal}{0.04} & 0.56 & 0.6 & \textcolor{teal}{0.10} & 0.57 & 0.68 & \textcolor{teal}{0.02} & 0.58 & 0.6 & \textcolor{orange}{-0.01} & 0.64 & 0.63 \\
\bottomrule
\end{tabular}}

\end{table*}

\paragraph{Concept-aware alignment.}
Here, we evaluate the concept-aware embedding method of Section \ref{subsubsec:axis_aligned}. Qualitatively, on the Fashion MNIST dataset, where the labels correspond to clothing items arranged from head to toe, Figure \ref{fig:fmnist_aligned} shows that the embeddings exhibit a clear alignment along the horizontal axis going from feet to head.

\begin{figure*}[!ht]
    \centering
    \includegraphics[width=0.70\linewidth]{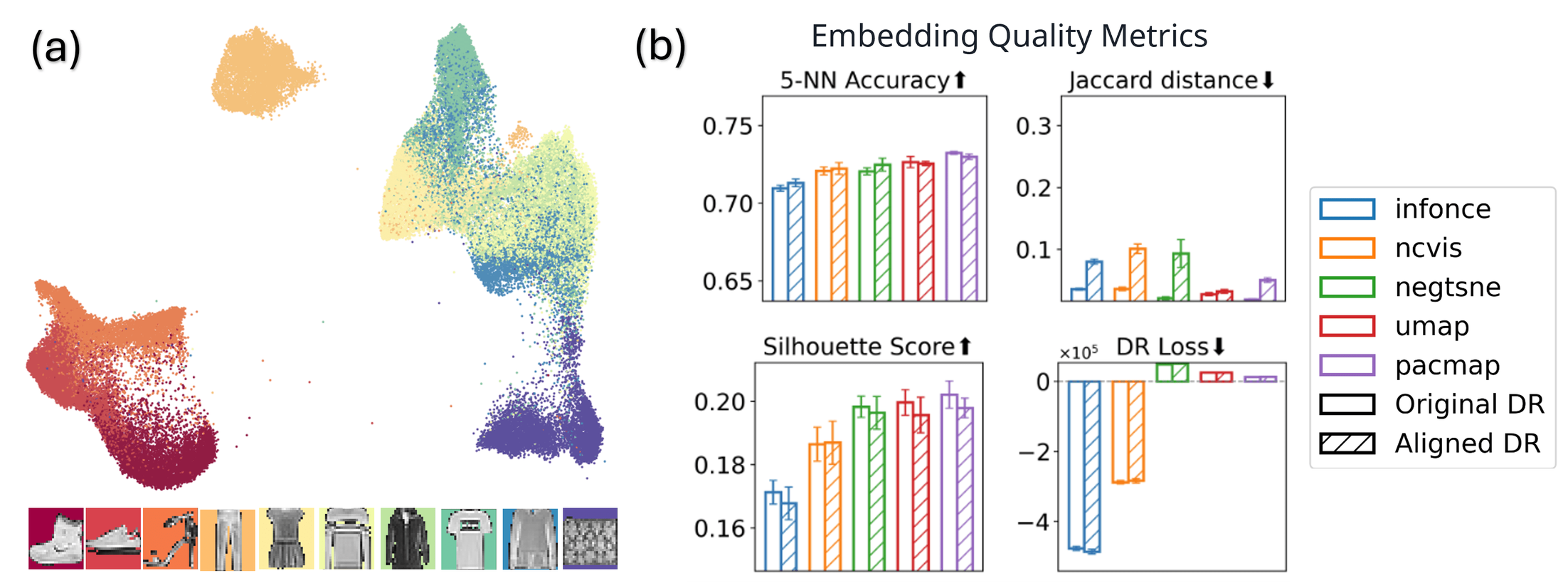}
    \caption{(a) Concept-informed aligned \PaCMAP embedding. Alignment is along the horizontal axis from feet (left) to head (right). Footwear is labeled in shades of red to orange, trousers in yellow, dresses in light yellow, pullovers and coats in green, shirts and t-shirts in blue, handbags in purple. (b) Evaluation metrics and losses for FMNIST before and after concept alignment, which remain generally unchanged.}
    \label{fig:fmnist_aligned}
\end{figure*}

Importantly, if we examine the metrics evaluated before and after alignment in Figure \ref{fig:fmnist_aligned}(b), we observe that for all methods, the metric values remain largely consistent. This indicates that the alignment process does not significantly disrupt the structural properties of the embeddings. For both types of alignment,    neither the PaCMAP loss nor the soft Jaccard similarity showed significant degradation. This is consistent with having high-quality embeddings throughout the Rashomon set. More results can be observed in Appendix \ref{app:pca_results} and \ref{app:concept_results}.

\subsection{Trustworthy Common Knowledge Graphs Improve DR Performance}\label{sub:experiment_common_knowledge}

Here we examine how extracted common knowledge can enhance the performance of dimension reduction based on Algorithm in Appendix \ref{app:algo}. As shown in Figure \ref{fig:usps_combined}, the embedding constructed using only the trustworthy nearest neighbor pairs (b) results in much clearer separation between different digits in USPS \citep{USPS}, compared to the original embedding (a). Furthermore, the quantitative metrics in (c) show that the combined DR approach (hatched bars) consistently outperforms or matches the original embeddings across multiple algorithms. This demonstrates that leveraging stable relationships across the Rashomon set improves both local detail and global organization, leading to more robust and interpretable representations. More examples are in Appendix \ref{app:common}.

\begin{figure*}[ht]
    \centering
    \includegraphics[width=\linewidth]{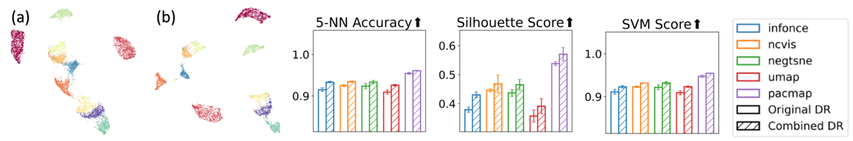}
    \caption{(a) Original \PaCMAP \ embedding of USPS dataset. (b) Common knowledge embedding using only stable neighbor pairs  within the Rashomon set. (c) Quantitative comparison of original vs. combined DR embeddings across three evaluation metrics for five methods.}
    \label{fig:usps_combined}
\end{figure*}

\subsection{Case Study: Common Knowledge Extraction Filters False Positives}

Here we use the MNIST dataset as an example to show how  the scoring selection function optimizes the nearest neighbor graph using the Rashomon set, enabling it to identify more trustworthy neighbor pairs. Figure \ref{fig:impact_nn} illustrates the difference between the original \PaCMAP 
 \ embedding and the refined embedding after incorporating stable nearest neighbors identified through our Rashomon set-based selection.
\begin{figure*}[ht]
    \centering
    
    \includegraphics[width=0.72\linewidth]{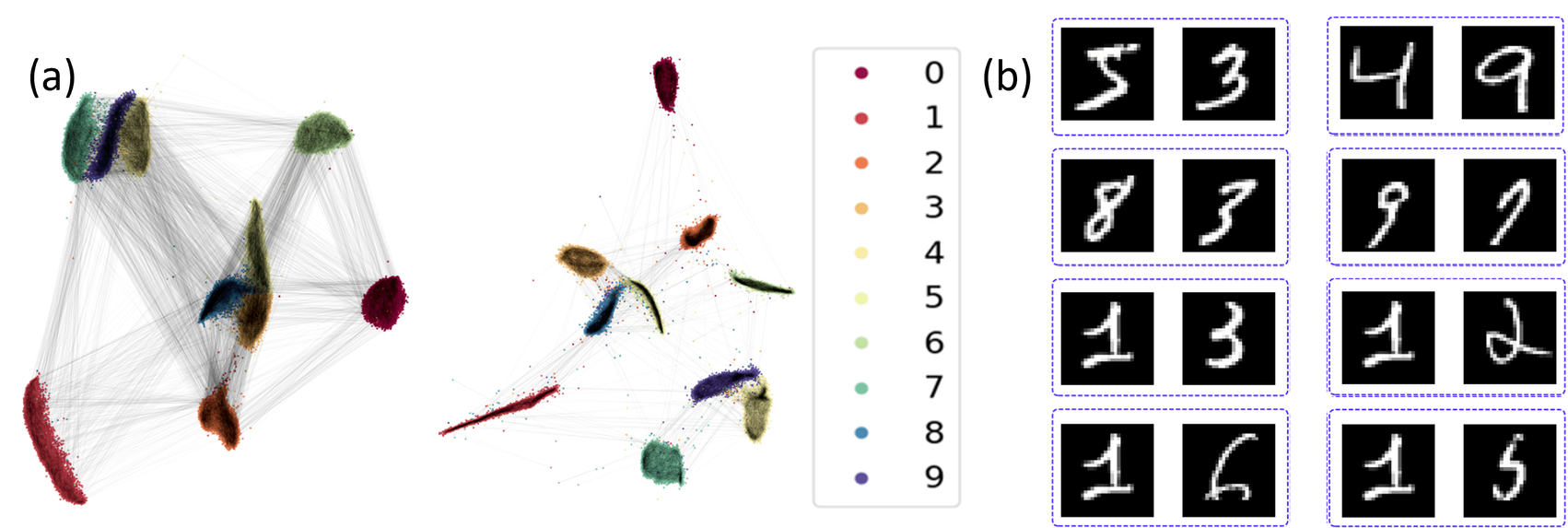}
    \caption{(a) MNIST embedding before (left) and after (right) common NN pairs are selected, (b) Examples of rejected NN pairs when finding common knowledge through the Rashomon set. These pairs are close in pixel space (but not close along the manifold of any digit).}
    \label{fig:impact_nn}
    
\end{figure*}

Figure~\ref{fig:impact_nn} shows that refining the nearest neighbor (NN) pairs leads to an embedding that achieves better class separability while still preserving meaningful global relationships among digits. In Figure \ref{fig:impact_nn} (a), the left panel shows the original PaCMAP embedding, where some digit clusters are entangled and boundaries between similar classes -- such as 3, 5, and 8 -- are less clear, due to false neighbor relationships. After selecting only the consistent NN pairs across the Rashomon set, the right panel reveals a cleaner embedding: digit clusters are more compact and well-separated, yet their relative positions remain consistent with semantic similarities (e.g., curved digits remain near each other). This suggests that the global structure is maintained even as local ambiguities are resolved.

Figure \ref{fig:impact_nn}(b) provides a closer look at several rejected NN pairs. These pairs often span across visually dissimilar digits, such as connecting a 5 with a 3, or a 9 with a 7. Their rejection suggests that they were unreliable connections in the original embedding; neighbors in pixel space are not equivalent to neighbors along the manifold of data. This visual and empirical evidence demonstrates that pruning away inconsistent neighbor relationships enhances the reliability of both local and global structure in the embedding. The result is a more interpretable representation where each digit class forms a tighter, more coherent cluster with fewer spurious links.

\section{Discussion and Limitations}
\label{sec:discussion_limitation}

Our work reframes dimension reduction as a problem with multiple valid solutions, bringing the Rashomon set into the DR paradigm. This perspective enables more interpretable, robust, and customizable embeddings. By aligning embeddings with principal components or user-defined concepts (e.g, developmental trajectories in biology), our approach offers a flexible way to inject interpretable structure into low-dimensional representations. By extracting common information across high-quality embeddings, we enhance the trustworthiness of DR.

One caveat is computation. Similar to random forests or boosting, generating approximations of a Rashomon set requires multiple runs of DR, each with different seeds, sub-samples, or regularization settings, leading to a higher computational cost during training. This overhead is often justified by the robustness and insights gained. Another possible caveat is that the user may try to align the axes with concepts that contradict the natural DR layout. This is not difficult to detect as $\L_\DR$ will increase. Formalizing this is a useful direction for future work.
%In addition, applying the alignment regularization requires some prior knowledge about the dataset such as meaningful concepts, class hierarchies, or interpretable axes of variation. In domains where such concepts are unknown or poorly defined, the benefits of alignment may be harder to realize.

Overall, this work provides a new lens through which to view dimension reduction—one that makes uncertainty visible, interpretable alignment possible, and common structure extractable. As DR continues to play a central role in data analysis pipelines and hypothesis generation for scientific domains, methods that leverage the space of valid embeddings will become increasingly valuable.

\section*{Acknowledgement}

This material is based upon work supported by the National Institutes of Health/NIDA grant number R01DA054994 and  National Science Foundation grant number CMMI-2323978.

\bibliographystyle{apalike}
\bibliography{mybib}

@inproceedings{localmap,
title={Dimension Reduction with Locally Adjusted Graphs}, 
booktitle={Association for the Advancement of Artificial Intelligence  ({AAAI}) Annual Conference on Artificial Intelligence},
year={2024},
author={Yingfan Wang and Yiyang Sun and Haiyang Huang and Cynthia Rudin}
}

@inproceedings{paramrepulsor,
  title={Navigating the Effect of Parametrization for Dimensionality Reduction},
  author={Huang, Haiyang and Wang, Yingfan and Rudin, Cynthia},
  booktitle={Neural Information Processing Systems},
  year={2024},
}

@article{tsne,
  title={Visualizing data using {t-SNE}},
  author={Van der Maaten, Laurens and Hinton, Geoffrey},
  journal={Journal of Machine Learning Research},
  volume={9},
  number={11},
  year={2008}
}

@article{umap_algorithm,
  title={Uniform manifold approximation and projection},
  author={Healy, John and McInnes, Leland},
  journal={Nature Reviews Methods Primers},
  volume={4},
  number={1},
  pages={82},
  year={2024},
  publisher={Nature Publishing Group UK London}
}

@article {umap_parametric,
    author = {Sainburg, Tim and McInnes, Leland and Gentner, Timothy Q.},
    title = {Parametric {UMAP}: learning embeddings with deep neural networks for representation and semi-supervised learning},
    journal = {ArXiv e-prints},
    archivePrefix = "arXiv",
    eprint = {2009.12981},
    primaryClass = "stat.ML",
    year = 2020,
}

@article{rashomon,
  title={Statistical modeling: The two cultures (with comments and a rejoinder by the author)},
  author={Breiman, Leo},
  journal={Statistical Science},
  volume={16},
  number={3},
  pages={199--231},
  year={2001},
  publisher={Institute of Mathematical Statistics}
}

@article{pacmap,
      title={Understanding {H}ow {D}imension {R}eduction {T}ools {W}ork: {A}n {E}mpirical {A}pproach to {D}eciphering t-{SNE}, {UMAP}, {TriMAP}, and {PaCMAP} for {D}ata {V}isualization}, 
      author={Yingfan Wang and Haiyang Huang and Cynthia Rudin and Yaron Shaposhnik},
      year={2021},
      journal={Journal of Machine Learning Research},
    volume={22}
}

@article{wattenberg2016use,
  title={How to use t-{SNE} effectively},
  author={Wattenberg, Martin and Vi{\'e}gas, Fernanda and Johnson, Ian},
  journal={Distill},
  volume={1},
  number={10},
  pages={e2},
  year={2016}
}

@article{nguyen2019ten,
  title={Ten quick tips for effective dimensionality reduction},
  author={Nguyen, Lan Huong and Holmes, Susan},
  journal={PLoS computational Biology},
  volume={15},
  number={6},
  year={2019},
  publisher={Public Library of Science}
}

@article{cao2017automatic,
  title={Automatic selection of t-{SNE} Perplexity},
  author={Cao, Yanshuai and Wang, Luyu},
  journal={arXiv preprint arXiv:1708.03229},
  year={2017}
}

@inproceedings{Tang16,
 author    = {Jian Tang and Jingzhou Liu and Ming Zhang and Qiaozhu Mei},
 title     = {Visualizing {L}arge-{S}cale and {H}igh-{D}imensional {D}ata},
 year      = {2016},
 pages     = {287--297},
 booktitle     = {Proceedings of the 25th International Conference on the World Wide Web}
}

@article{pca,
author = {Karl Pearson},
title = {{LIII}. On lines and planes of closest fit to systems of points in space},
journal = {The London, Edinburgh, and Dublin Philosophical Magazine and Journal of Science},
volume = {2},
number = {11},
pages = {559-572},
year  = {1901},
publisher = {Taylor & Francis},
doi = {10.1080/14786440109462720},
}

@inproceedings{Belkin01,
 author    = {Mikhail Belkin and Partha Niyogi},
 title     = {Laplacian {E}igenmaps and {S}pectral {T}echniques for {E}mbedding and {C}lustering},
 year      = {2001},
 pages     = {585--591},
 volume    = {14},
 booktitle     = {Advances in Neural Information Processing Systems},
 publisher = {MIT Press}
}

@Article{Torgerson52,
  title={Multidimensional scaling: I {T}heory and method },
  author={Warren Torgerson},
  journal={Psychometrika},
  year={1952},
  volume={17},
  number={4},
  pages={401-419}
}

@Article{Tenenbaum00,
  title={A {G}lobal {G}eometric {F}ramework for {N}onlinear {D}imensionality {R}eduction},
  author={Joshua B. Tenenbaum and Vin de Silva and John C. Langford},
  journal={Science},
  year={2000},
  volume={290},
  number={5500},
  pages={2319-2323}
}

@Article{Roweis00,
  title={Nonlinear {D}imensionality {R}eduction by {L}ocally {L}inear {E}mbedding},
  author={Sam T. Roweis and Lawrence K. Saul},
  journal={Science},
  year={2000},
  volume={290},
  number={5500},
  pages={2323-2326}
}

@Article{Linderman19FItSNE,
  title={Fast interpolation-based {t-SNE} for improved visualization of single-cell RNA-seq data},
  author={George C. Linderman and Manas Rachh and Jeremy G. Hoskins and Stefan Steinerberger and Yuval Kluger },
  journal={Nature Methods},
  year={2019},
  volume={16},
  pages={243-245}
}

@Article{Belkina19OptSNE,
  title={Automated optimized parameters for T-distributed stochastic neighbor embedding improve visualization and analysis of large datasets},
  author={Anna C. Belkina and Christopher O. Ciccolella and Rina Anno and Richard Halpert and Josef Spidlen and Jennifer E. Snyder-Cappione},
  journal={Nature Communications},
  year={2019},
  volume={10},
  number={5415},
}

@ARTICLE{TriMAP,
       author = {{Amid}, Ehsan and {Warmuth}, Manfred K.},
        title = "{{TriMAP}: Large-scale Dimensionality Reduction Using Triplets}",
      journal = {arXiv e-prints},
         year = 2019,
          eid = {arXiv:1910.00204},
        pages = {arXiv:1910.00204},
archivePrefix = {arXiv},
       eprint = {1910.00204}
}

@article{lecun2010mnist,
  title={{MNIST} handwritten digit database},
  author={LeCun, Yann and Cortes, Corinna and Burges, CJ},
  journal={ATT Labs [Online]. Available: http://yann.lecun.com/exdb/mnist},
  volume={2},
  year={2010}
}

@misc{UmapMammoth,
  author = {Coenen, Andy and Pearce, Adam},
  title = {Understanding {UMAP}},
  year = {2019},
  publisher = {Google Research},
  howpublished = {\url{https://pair-code.github.io/understanding-umap/}},
}

@misc{SmithsonianMammoth,
  author = {{The Smithsonian Institute}},
  title = {Mammuthus primigenius (Blumbach)},
  publisher = {The Smithsonian Institute},
  year = {2020},
  howpublished = {\url{https://3d.si.edu/object/3d/mammuthus-primigenius-blumbach:341c96cd-f967-4540-8ed1-d3fc56d31f12}},
}

@TECHREPORT{coil20,
    author = {Sameer A. Nene and Shree K. Nayar and Hiroshi Murase},
    title = {Columbia {O}bject {I}mage {L}ibrary (COIL-20)},
    institution = {Technical Report CUCS-005-96},
    year = {1996}
}

@article{fmnist,
  title={Fashion-{MNIST}: a novel image dataset for benchmarking machine {L}earning algorithms},
  author={Xiao, Han and Rasul, Kashif and Vollgraf, Roland},
  journal={arXiv preprint arXiv:1708.07747},
  year={2017}
}

@Manual{Annoy,
  author =       {Erik Bernhardsson},
  title =        {Annoy: Approximate Nearest Neighbors in {C++/Python}},
  year =         2019,
  note =  {Python package version 1.16.3},
  url =          {https://pypi.org/project/annoy/},
}

@ARTICLE{USPS,  author={J. J. {Hull}},  journal={IEEE Transactions on Pattern Analysis and Machine Intelligence},   title={A database for handwritten text recognition research},   year={1994},  volume={16},  number={5},  pages={550-554},}

@article{kang2018multiplexed,
  title={Multiplexed droplet single-cell {RNA}-sequencing using natural genetic variation},
  author={Kang, Hyun Min and Subramaniam, Meena and Targ, Sasha and Nguyen, Michelle and Maliskova, Lenka and McCarthy, Elizabeth and Wan, Eunice and Wong, Simon and Byrnes, Lauren and Lanata, Cristina M and others},
  journal={Nature Biotechnology},
  volume={36},
  number={1},
  pages={89},
  year={2018},
  publisher={Nature Publishing Group}
}

@article{stuart2019comprehensive,
  title={Comprehensive integration of single-cell data},
  author={Stuart, Tim and Butler, Andrew and Hoffman, Paul and Hafemeister, Christoph and Papalexi, Efthymia and Mauck III, William M and Hao, Yuhan and Stoeckius, Marlon and Smibert, Peter and Satija, Rahul},
  journal={Cell},
  volume={177},
  number={7},
  pages={1888--1902},
  year={2019},
  publisher={Elsevier}
}

@article{townes2019feature,
  title={Feature selection and dimension reduction for single-cell RNA-Seq based on a multinomial model},
  author={Townes, F William and Hicks, Stephanie C and Aryee, Martin J and Irizarry, Rafael A},
  journal={Genome Biology},
  volume={20},
  number={1},
  pages={1--16},
  year={2019},
  publisher={Springer}
}

@article{lee1999learning,
  title={Learning the parts of objects by non-negative matrix factorization},
  author={Lee, Daniel D and Seung, H Sebastian},
  journal={Nature},
  volume={401},
  number={6755},
  pages={788--791},
  year={1999},
  publisher={Nature Publishing Group}
}

@inproceedings{moor2020topological,
  title={Topological {A}utoencoders},
  author={Moor, Michael and Horn, Max and Rieck, Bastian and Borgwardt, Karsten},
  booktitle={International {C}onference on {M}achine {L}earning},
  pages={7045--7054},
  year={2020},
  organization={PMLR}
}

@inproceedings{bohm2022unsupervised,
  title={Unsupervised visualization of image datasets using contrastive learning},
  author={B{\"o}hm, Jan Niklas and Berens, Philipp and Kobak, Dmitry},
  booktitle={International Conference on Learning Representations},
  year={2023},
}

@inproceedings{van2009learning,
  title={Learning a {P}arametric {E}mbedding by {P}reserving {L}ocal {S}tructure},
  author={van der Maaten, Laurens},
  booktitle={Artificial {I}ntelligence and {S}tatistics},
  pages={384--391},
  year={2009},
  organization={PMLR}
}

@inproceedings{damrich2023from,
  title={From $t$-{SNE} to {UMAP} with contrastive learning},
  author={Damrich, Sebastian and B{\"o}hm, Jan Niklas  and Hamprecht, Fred A and Kobak, Dmitry},
  booktitle={International Conference on Learning Representations},
  year={2023},
}

@article{bohm2022attraction,
  title={Attraction-{R}epulsion {S}pectrum in {N}eighbor {E}mbeddings},
  author={B{\"o}hm, Jan Niklas and Berens, Philipp and Kobak, Dmitry},
  journal={Journal of Machine Learning Research},
  volume={23},
  number={1},
  pages={4118--4149},
  year={2022},
  publisher={JMLRORG}
}

@inproceedings{nazari2023geometric,
  title={Geometric Autoencoders--What You See is What You Decode},
  author={Nazari, Philipp and Damrich, Sebastian and Hamprecht, Fred A},
  booktitle={Proceedings of International Conference on Machine Learning},
  year={2023},
pages={25834-25857},
  organization={PMLR},

}

@inproceedings{artemenkov2020ncvis,
  title={{NCVis}: {N}oise {C}ontrastive {A}pproach for {S}calable {V}isualization},
  author={Artemenkov, Aleksandr and Panov, Maxim},
  booktitle={Proceedings of The Web Conference},
  pages={2941--2947},
  year={2020}
}

@inproceedings{gutmann2010noise,
  title={Noise-contrastive estimation: A new estimation principle for unnormalized statistical models},
  author={Gutmann, Michael and Hyv{\"a}rinen, Aapo},
  booktitle={Proceedings of the {I}nternational {C}onference on {A}rtificial {I}ntelligence and {S}tatistics},
  pages={297--304},
  year={2010}
}

@article{oord2018representation,
  title={Representation {L}earning with {C}ontrastive {P}redictive {C}oding},
  author={Oord, Aaron van den and Li, Yazhe and Vinyals, Oriol},
  journal={arXiv preprint arXiv:1807.03748},
  year={2018}
}

@article{neurips2021_data,
  title={Simultaneous epitope and transcriptome measurement in single cells},
  author={Stoeckius, Marlon and Hafemeister, Christoph and Stephenson, William and Houck-Loomis, Brian and Chattopadhyay, Pratip K and Swerdlow, Harold and Satija, Rahul and Smibert, Peter},
  journal={Nature Methods},
  volume={14},
  number={9},
  pages={865--868},
  year={2017},
  publisher={Nature Publishing Group}
}

@article{human_cortex_data,
  title={Multi-omic profiling of the developing human cerebral cortex at the single-cell level},
  author={Zhu, Kaiyi and Bendl, Jaroslav and Rahman, Samir and Vicari, James M and Coleman, Claire and Clarence, Tereza and Latouche, Ovaun and Tsankova, Nadejda M and Li, Aiqun and Brennand, Kristen J and others},
  journal={Science Advances},
  volume={9},
  number={41},
  pages={eadg3754},
  year={2023},
  publisher={American Association for the Advancement of Science}
}

@article{van2022probabilistic,
  title={A probabilistic graph coupling view of dimension reduction},
  author={Van Assel, Hugues and Espinasse, Thibault and Chiquet, Julien and Picard, Franck},
  journal={Advances in Neural Information Processing Systems},
  volume={35},
  pages={10696--10708},
  year={2022}
}

@article{colange2020steering,
  title={Steering distortions to preserve classes and neighbors in supervised dimensionality reduction},
  author={Colange, Beno{\^\i}t and Peltonen, Jaakko and Aupetit, Micha{\"e}l and Dutykh, Denys and Lespinats, Sylvain},
  journal={Advances in Neural Information Processing Systems},
  volume={33},
  pages={13214--13225},
  year={2020}
}

@inproceedings{zu2022spacemap,
  title={{SpaceMAP}: Visualizing High-Dimensional Data by Space Expansion.},
  author={Zu, Xinrui and Tao, Qian},
  booktitle={International Conference on Machine Learning},
  pages={27707--27723},
  year={2022}
}

@inproceedings{sarfraz2022hnne,
  title={Hierarchical nearest neighbor graph embedding for efficient dimensionality reduction},
  author={Sarfraz, Saquib and Koulakis, Marios and Seibold, Constantin and Stiefelhagen, Rainer},
  booktitle={Proceedings of the IEEE/CVF Conference on Computer Vision and Pattern Recognition},
  pages={336--345},
  year={2022}
}

@article{tomo2025median,
  title={Median Consensus Embedding for Dimensionality Reduction},
  author={Tomo, Yui and Yoneoka, Daisuke},
  journal={arXiv preprint arXiv:2503.08103},
  year={2025}
}

@article{viswanath2012consensus,
  title={Consensus embedding: theory, algorithms and application to segmentation and classification of biomedical data},
  author={Viswanath, Satish and Madabhushi, Anant},
  journal={BMC Bioinformatics},
  volume={13},
  pages={1--20},
  year={2012},
  publisher={Springer}
}

@inproceedings{tiwari2008consensus,
  title={Consensus-locally linear embedding (C-LLE): application to prostate cancer detection on magnetic resonance spectroscopy},
  author={Tiwari, Pallavi and Rosen, Mark and Madabhushi, Anant},
  booktitle={Medical Image Computing and Computer-Assisted Intervention--MICCAI 2008: 11th International Conference, New York, NY, USA, September 6-10, 2008, Proceedings, Part II 11},
  pages={330--338},
  year={2008},
  organization={Springer}
}

@article{li2022consensus,
  title={Consensus embedding for multiple networks: Computation and applications},
  author={Li, Mengzhen and Co{\c{s}}kun, Mustafa and Koyut{\"u}rk, Mehmet},
  journal={Network Science},
  volume={10},
  number={2},
  pages={190--206},
  year={2022},
  publisher={Cambridge University Press}
}

@article{dalmia2021clustering,
  title={Clustering with {UMAP}: Why and how connectivity matters},
  author={Dalmia, Ayush and Sia, Suzanna},
  journal={arXiv preprint arXiv:2108.05525},
  year={2021}
}

@misc{contrastivenegithub,
  author       = {Berens Lab},
  title        = {Contrastive Neighbor Embeddings},
  year         = {2022},
  howpublished = {\url{https://github.com/berenslab/contrastive-ne}},
  note         = {Accessed: 2025-05-04}
}

@article{kobak2021initialization,
  title={Initialization is critical for preserving global data structure in both {t-SNE} and {UMAP}},
  author={Kobak, Dmitry and Linderman, George C},
  journal={Nature Biotechnology},
  volume={39},
  number={2},
  pages={156--157},
  year={2021},
  publisher={Nature Publishing Group, New York}
}

@inproceedings{airplane,
title={3d shapenets: A deep representation for volumetric shapes},
author={Wu, Zhirong and Song, Shuran and Khosla, Aditya and Yu, Fisher and Zhang, Linguang and Tang, Xiaoou and Xiao, Jianxiong},
booktitle={Proceedings of the IEEE Conference on Computer Vision and Pattern Recognition},
pages={1912--1920},
year={2015}
}

@misc{fico,
  author = {FICO},
  title = {Explainable Machine Learning Challenge},
  howpublished={\url {https://community.fico.com/s/explainable-machine-learning-challenge}},
  year = {2018}
}

%%%%%%%%%%%%%%%%%%%%%%%%%%%%%%%%%%%%%%%%%%%%%%%%%%%%%%%%%%%%
\section*{Checklist}

\begin{enumerate}

  \item For all models and algorithms presented, check if you include:
  \begin{enumerate}
    \item A clear description of the mathematical setting, assumptions, algorithm, and/or model. [Yes]
    \item An analysis of the properties and complexity (time, space, sample size) of any algorithm. [Yes]
    \item (Optional) Anonymized source code, with specification of all dependencies, including external libraries. [Yes]
  \end{enumerate}

  \item For any theoretical claim, check if you include:
  \begin{enumerate}
    \item Statements of the full set of assumptions of all theoretical results. [Yes]
    \item Complete proofs of all theoretical results. [Yes]
    \item Clear explanations of any assumptions. [Yes]     
  \end{enumerate}

  \item For all figures and tables that present empirical results, check if you include:
  \begin{enumerate}
    \item The code, data, and instructions needed to reproduce the main experimental results (either in the supplemental material or as a URL). [Yes]
    \item All the training details (e.g., data splits, hyperparameters, how they were chosen). [Yes]
    \item A clear definition of the specific measure or statistics and error bars (e.g., with respect to the random seed after running experiments multiple times). [Yes]
    \item A description of the computing infrastructure used. (e.g., type of GPUs, internal cluster, or cloud provider). [Yes]
  \end{enumerate}

  \item If you are using existing assets (e.g., code, data, models) or curating/releasing new assets, check if you include:
  \begin{enumerate}
    \item Citations of the creator If your work uses existing assets. [Yes]
    \item The license information of the assets, if applicable. [Yes]
    \item New assets either in the supplemental material or as a URL, if applicable. [Not Applicable]
    \item Information about consent from data providers/curators. [Yes]
    \item Discussion of sensible content if applicable, e.g., personally identifiable information or offensive content. [Yes]
  \end{enumerate}

  \item If you used crowdsourcing or conducted research with human subjects, check if you include:
  \begin{enumerate}
    \item The full text of instructions given to participants and screenshots. [Not Applicable]
    \item Descriptions of potential participant risks, with links to Institutional Review Board (IRB) approvals if applicable. [Not Applicable]
    \item The estimated hourly wage paid to participants and the total amount spent on participant compensation. [Not Applicable]
  \end{enumerate}

\end{enumerate}

\clearpage
\appendix
\thispagestyle{empty}

% Supplementary material: To improve readability, you must use a single-column format for the supplementary material.
\onecolumn
\aistatstitle{Supplementary Materials for The Rashomon Effect for Visualizing High-Dimensional Data}

\section{Proof of Theorem}
\label{theorem_proof}

{\color{black}
\begin{theorem}
Let $(i, j^*, j')$ be a triplet of points. Let $\mathcal{D}$ be the distribution of valid embeddings within the Rashomon set. Let $\{y^{(1)}, \dots, y^{(T)}\}$ be $T$ i.i.d.\ low-dimensional embeddings sampled from $\mathcal{D}$. Define the population scoring function $\Psi_{ij}(y) = d^y_{ij} + \lambda \cdot \sigma^*_{ij}$ if $d^y_{ij} \le M^*$, and $\Psi_{ij}(y) = m^*_{ij} + \lambda \cdot \sigma^*_{ij}$ otherwise, where $\sigma^*_{ij}, m^*_{ij}, M^*$ are fixed population constants estimated from a separate calibration set. Define the margin variable for the $t$-th embedding:
$$
Z^{(t)}_{i,j^*,j'} := \Psi_{ij^*}(y^{(t)}) - \Psi_{ij'}(y^{(t)}).
$$
Since $\Psi$ depends on the random sample $y^{(t)}$ through $d^{y^{(t)}}_{ij}$ while using only fixed population constants, and the $y^{(t)}$ are i.i.d., the variables $Z^{(t)}_{i,j^*,j'}$ are strictly i.i.d.\ real-valued random variables. Let $\Delta = \mathbb{E}[Z^{(t)}_{i,j^*,j'}]$ be the true expected margin. Define the empirical margin:
$$\hat{\Delta}^{(T)} = \frac{1}{T}\sum_{t=1}^T Z^{(t)}_{i,j^*,j'}.$$
Let $V = \mathrm{Var}(Z^{(t)}_{i,j^*,j'})$ and $B$ a constant such that $Z^{(t)}_{i,j^*,j'} \le B$ almost surely. Then with probability at least $1 - \delta$:
\[
\Delta \le \hat{\Delta}^{(T)} + \frac{B\log(1/\delta)}{3T} + \sqrt{\frac{2V\log(1/\delta)}{T}}.
\]
We say that the neighbor $j^*$ is strictly more trustworthy than $j'$ if the upper bound of this confidence interval is negative:
\[
\hat{\Delta}^{(T)} + \frac{B\log(1/\delta)}{3T} + \sqrt{\frac{2V\log(1/\delta)}{T}} < 0.
\]
\end{theorem}

\begin{proof}
The sequence $\{Z^{(t)}_{i,j^*,j'}\}_{t=1}^T$ consists of i.i.d.\ real-valued random variables with mean $\Delta = \mathbb{E}[Z^{(t)}_{i,j^*,j'}]$, variance $V = \mathrm{Var}(Z^{(t)}_{i,j^*,j'})$, and $Z^{(t)}_{i,j^*,j'} \le B$ almost surely. By the one-sided Bernstein inequality for bounded i.i.d.\ random variables:
\[
\mathbb{P}\left( \frac{1}{T}\sum_{t=1}^T Z^{(t)} - \Delta \le -\epsilon \right)
\le \exp\left( -\frac{T\epsilon^2}{2V + \frac{2B\epsilon}{3}} \right).
\]
Setting the right-hand side equal to $\delta$ and solving the resulting quadratic in $\epsilon$ exactly (without the simplification $\sqrt{a+b} \le \sqrt{a} + \sqrt{b}$), we obtain:
\[
\Delta \le \hat{\Delta}^{(T)} + \frac{B\log(1/\delta)}{3T} + \sqrt{\frac{2V\log(1/\delta)}{T} + \frac{B^2\log^2(1/\delta)}{9T^2}}.
\]
Since $\sqrt{a + b} \le \sqrt{a} + \sqrt{b}$ for $a, b > 0$, this is bounded by:
\[
\Delta \le \hat{\Delta}^{(T)} + \frac{B\log(1/\delta)}{3T} + \sqrt{\frac{2V\log(1/\delta)}{T}},
\]
which is the stated bound. Since $\hat{\Delta}^{(T)} \to \Delta$ by the law of large numbers, for sufficiently large $T$, the confidence interval will certify the sign of $\Delta$.
\end{proof}
}

\clearpage
\section{Generating Embedding through Common Knowledge}
\label{app:algo}

\textcolor{black}{In this section, we provide detailed algorithmic procedures for extracting a common knowledge graph from a Rashomon set of embeddings and generating a refined embedding. Both algorithms below use the \emph{calibration-test split} described in Section~\ref{sec:common_knowledge}: a separate calibration set of embeddings is used to estimate the population statistics ($\mu_{ij}, \sigma_{ij}, m_{ij}, M_{\text{global}}$), and these fixed estimates are then used to score neighbor pairs in an independent test (Rashomon) set of embeddings. The two approaches differ in how the per-embedding scores are aggregated to produce a final neighbor ranking:}

\textcolor{black}{\textbf{Indicator (Average Rank) Approach} (Algorithm~\ref{alg:refined_embedding_rank}): For each test embedding, neighbors are ranked by their penalized score $\Psi_{ij}^{(t)}$. The final selection uses the average rank $\bar{r}_{ij}$ across all test embeddings. This ordinal aggregation is robust to outlier scores since it discards magnitude information.}

\textcolor{black}{\textbf{Mean Score Approach} (Algorithm~\ref{alg:refined_embedding_score}): The penalized scores $\Psi_{ij}^{(t)}$ are averaged directly across test embeddings, and the $k$ neighbors with the lowest mean score $\bar{\Psi}_{ij}$ are selected. This approach retains the full distance-aware information from $\Psi$ and is directly connected to the Bernstein concentration bound in Theorem~\ref{theo:conv}, where the empirical margin $\hat{\Delta}^{(T)}$ is the difference of two such averages. In our experiments, we use the Mean Score Approach as the default.}

% \begin{algorithm}[H]
% \caption{Refined Embedding via Common Knowledge Graph}
% \label{alg:refined_embedding}
% \begin{algorithmic}[1]
% \REQUIRE Rashomon set of embeddings $\{Y^{(1)}, Y^{(2)}, \dots, Y^{(T)}\}$, number of trustworthy neighbors $k$ (e.g., $k=10$ for \PaCMAP), the original dataset $\mathcal{X} \in\mathbb{R}^{n\times d}$
% \OUTPUT Refined $k$-NN graph $\kNN_{\text{refined}}$, reconstructed embedding $\tilde{Y}$

% \STATE \textbf{Compute pairwise distances:} For each embedding $Y^{(t)}$, compute the pairwise distances $d_{ij}^{(t)}$ between all points.

% \STATE \textbf{Construct common $m$-NN graph:} Construct the $m$-NN graph in the $d$ dimensional space from the original dataset $\mathcal{X}$. Any kind of nearest-neighbor methods are applicable here.

% \FOR{each data point $i$}
%     \STATE Retrieve $m$ neighbors $\mathcal{N}_{m}(i)$ from $m$-NN
%     \FOR{each neighbor $j \in \mathcal{N}_{m}(i)$}
%         \STATE Compute edge score $\Psi_{ij}$ using Equation~\ref{eq:edge_importance}, which aggregates over $t = 1$ to $T$
%     \ENDFOR
%     \STATE Select $k$ neighbors with the lowest $\Psi_{ij}$: $\mathcal{N}_k(i) \gets \texttt{BottomK}(\mathcal{N}_{m}(i), \Psi_{ij}, k)$
% \ENDFOR

% \STATE Construct refined $k$-NN graph $\kNN_{\text{refined}}$ using $\mathcal{N}_k(i)$ for all $i$

% \STATE Reconstruct the embedding $\tilde{Y}$ using the DR method with default parameters and $\kNN_{\text{refined}}$ as the input graph

% \RETURN $\kNN_{\text{refined}}, \tilde{Y}$
% \end{algorithmic}
% \end{algorithm}

\begin{algorithm}[ht]
\caption{Refined Embedding -- \textcolor{black}{Indicator (Average Rank) Approach}}
\label{alg:refined_embedding_rank}
\begin{algorithmic}[1]
\REQUIRE Rashomon set of embeddings $\{Y^{(1)}, Y^{(2)}, \dots, Y^{(T)}\}$, Calibration set of embeddings $\{C^{(1)}, C^{(2)}, \dots, C^{(M)}\}$, number of trustworthy neighbors $k$, original dataset $\mathcal{X} \in\mathbb{R}^{n\times d}$, penalty weight $\lambda$
\OUTPUT Refined $k$-NN graph $\kNN_{\text{refined}}$, consensus embedding $\tilde{Y}$

\STATE \textbf{Construct Candidate Neighbor Set:} Construct the $m$-NN graph ($m > k$) from the original dataset $\mathcal{X}$ to define candidate neighbors $\mathcal{N}_{m}(i)$ for each point $i$.

\STATE \textbf{Estimate Calibration Statistics:}
\FOR{each point $i$ within the dataset $\mathcal{X}$ and each neighbor $j \in \mathcal{N}_{m}(i)$}
    \STATE Collect distances from Calibration set: $d_{ij}^C = \{ \|C^{(r)}_i - C^{(r)}_j\|_2 \}_{r=1}^M$
    \STATE Estimate mean $\mu_{ij} \leftarrow \text{mean}_{C}(d_{ij}^C)$, $m_{ij}\leftarrow\text{max}_{C}(d_{ij}^C)$ and standard deviation $\sigma_{ij} \leftarrow \text{std}_{C}(d_{ij}^C)$
\ENDFOR
\STATE Compute global threshold $M_{\text{global}} \leftarrow \text{mean}_{ij\ \text{pairs}} (\mu_{ij})$

\STATE \textbf{Rank Edges in Rashomon set:}
\FOR{each embedding $t = 1$ to $T$ in Rashomon set $Y$}
    \FOR{each point $i$}
        \FOR{each neighbor $j \in \mathcal{N}_{m}(i)$}
            \STATE Compute raw distance: $d_{ij}^{(t)} = \|Y^{(t)}_i - Y^{(t)}_j\|_2$
            \STATE Compute penalized score: 
            \IF{$\mu_{ij} <  M_{\text{global}}$}
            \STATE $\Psi_{ij}^{(t)} \leftarrow d_{ij}^{(t)} + \lambda \cdot \sigma_{ij}$
            \ELSE
            \STATE $\Psi_{ij}^{(t)} \leftarrow m_{ij} + \lambda \cdot \sigma_{ij}$
            \ENDIF
        \ENDFOR
        \STATE Compute rank $r_{ij}^{(t)}$: Rank of neighbor $j$ among $\mathcal{N}_{m}(i)$ based on score $\Psi_{ij}^{(t)}$ (ascending)
    \ENDFOR
\ENDFOR

\STATE \textbf{Select Trustworthy Neighbors via Average Ranking:}
\FOR{each point $i$}
    \FOR{each neighbor $j \in \mathcal{N}_{m}(i)$}
        \STATE Compute average rank: $\bar{r}_{ij} \leftarrow \frac{1}{T} \sum_{t=1}^{T} r_{ij}^{(t)}$
    \ENDFOR
    \STATE Select $k$ neighbors with the lowest $\bar{r}_{ij}$: $\mathcal{N}_k(i) \gets \texttt{BottomK}(\mathcal{N}_{m}(i), \bar{r}_{ij}, k)$
\ENDFOR

\STATE \textbf{Reconstruction:}
\STATE Construct refined $k$-NN graph $\kNN_{\text{refined}}$ using $\mathcal{N}_k(i)$ for all $i$
\STATE Create consensus embedding $\tilde{Y}$ using the DR method with $\kNN_{\text{refined}}$

\RETURN $\kNN_{\text{refined}}, \tilde{Y}$
\end{algorithmic}
\end{algorithm}

\begin{algorithm}[ht]
\caption{Refined Embedding -- \textcolor{black}{Mean Score Approach}}
\label{alg:refined_embedding_score}
\begin{algorithmic}[1]
\REQUIRE Rashomon set of embeddings $\{Y^{(1)}, Y^{(2)}, \dots, Y^{(T)}\}$, Calibration set of embeddings $\{C^{(1)}, C^{(2)}, \dots, C^{(M)}\}$, number of trustworthy neighbors $k$, original dataset $\mathcal{X} \in\mathbb{R}^{n\times d}$, penalty weight $\lambda$
\OUTPUT Refined $k$-NN graph $\kNN_{\text{refined}}$, consensus embedding $\tilde{Y}$

\STATE \textbf{Construct Candidate Neighbor Set:} Construct the $m$-NN graph ($m > k$) from the original dataset $\mathcal{X}$ to define candidate neighbors $\mathcal{N}_{m}(i)$ for each point $i$.

\STATE \textbf{Estimate Calibration Statistics:}
\FOR{each point $i$ and each neighbor $j \in \mathcal{N}_{m}(i)$}
    \STATE Collect distances from Calibration set: $d_{ij}^C = \{ \|C^{(r)}_i - C^{(r)}_j\|_2 \}_{r=1}^M$
    \STATE Estimate mean $\mu_{ij} \leftarrow \text{mean}_{C}(d_{ij}^C)$, $m_{ij}\leftarrow\text{max}_{C}(d_{ij}^C)$ and standard deviation $\sigma_{ij} \leftarrow \text{std}_{C}(d_{ij}^C)$
\ENDFOR
\STATE Compute global threshold $M_{\text{global}} \leftarrow \text{mean}_{ij\ \text{pairs}} (\mu_{ij})$

\STATE \textbf{Compute Penalized Scores in Rashomon set:}
\FOR{each embedding $t = 1$ to $T$ in Rashomon set $Y$}
    \FOR{each point $i$}
        \FOR{each neighbor $j \in \mathcal{N}_{m}(i)$}
            \STATE Compute raw distance: $d_{ij}^{(t)} = \|Y^{(t)}_i - Y^{(t)}_j\|_2$
            \STATE Compute penalized score: 
            \IF{$\mu_{ij} <  M_{\text{global}}$}
            \STATE $\Psi_{ij}^{(t)} \leftarrow d_{ij}^{(t)} + \lambda \cdot \sigma_{ij}$
            \ELSE
            \STATE $\Psi_{ij}^{(t)} \leftarrow m_{ij} + \lambda \cdot \sigma_{ij}$
            \ENDIF
        \ENDFOR
    \ENDFOR
\ENDFOR

\STATE \textbf{Select Trustworthy Neighbors via Mean Score:}
\FOR{each point $i$}
    \FOR{each neighbor $j \in \mathcal{N}_{m}(i)$}
        \STATE Compute average score: $\bar{\Psi}_{ij} \leftarrow \frac{1}{T} \sum_{t=1}^{T} \Psi_{ij}^{(t)}$
    \ENDFOR
    \STATE Select $k$ neighbors with the lowest $\bar{\Psi}_{ij}$: $\mathcal{N}_k(i) \gets \texttt{BottomK}(\mathcal{N}_{m}(i), \bar{\Psi}_{ij}, k)$
\ENDFOR

\STATE \textbf{Reconstruction:}
\STATE Construct refined $k$-NN graph $\kNN_{\text{refined}}$ using $\mathcal{N}_k(i)$ for all $i$
\STATE Reconstruct consensus embedding $\tilde{Y}$ using the DR method with $\kNN_{\text{refined}}$

\RETURN $\kNN_{\text{refined}}, \tilde{Y}$
\end{algorithmic}
\end{algorithm}

\clearpage
\section{Evaluation Metrics}
\label{app:metrics}

To assess the quality of learned embeddings under both supervised and unsupervised settings, we evaluate them using the following metrics.

\subsection{Soft Jaccard distance}

This metric measures the distance between two weighted matrices originating from two embeddings, particularly in terms of how consistently they preserve pairwise relationships across a shared large nearest neighbor (NN) graph.

\begin{enumerate}
    \item Let $\kNN$ be a fixed large NN graph derived from high-dimensional data using an NN algorithm (e.g., ANNOY\citep{Annoy}). In this experiment, we have set up a 50-NN graph.
    \item For each data point pair $(i,j)$, define a similarity weight $w_{ij}^{y^1}$ in the baseline embedding and $w_{ij}^{y^2}$ in the compared embedding $y^1$ and $y^2$, where $W_{ij} = \frac{(\|y_i-y_j\|_2^2 + \delta)}{(\|y_i-y_j\|_2^2+\delta)+1}$, and $y$ is the low-dimensional embedding.
    \item Compute the soft Jaccard similarity:
    $$
    d(W^{y^1}, W^{y^2}) := 1-
    \frac{\sum_{i,j} \frac{\min \left( W^{y^1}_{ij}, W^{y^2}_{ij} \right)}{W^{y^1}_{ij} + W^{y^2}_{ij}}}
         {\sum_{i,j} \frac{\max \left( W^{y^1}_{ij}, W^{y^2}_{ij} \right)}{W^{y^1}_{ij} + W^{y^2}_{ij}}}
    $$
    \item Lower values indicate better consistency between neighborhood structures of the two embeddings.
\end{enumerate}

\subsection{PCA-aligned Triplet Score (\texttt{TripletPCA})}

This metric evaluates whether the embedding preserves the global inter-class relationships revealed by a linear projection (PCA). Specifically, we compare the relative distances between class centroids in PCA space versus those in the embedding.

\begin{enumerate}
    \item Project the original data $\X$ into PCA space to obtain $\mathbf{y}_{\textrm{PCA}}$.
    
    \item For each class $c$, compute its centroid in PCA space and in the evaluated embedding $y$:
    \[
    \mu_c^{\text{PCA}} = \frac{1}{|C_c|} \sum_{i \in C_c} y_{\text{PCA},i}, \quad
    \mu_c^{y} = \frac{1}{|C_c|} \sum_{i \in C_c} y_i
    \]
    
    \item For all unordered pairs of classes $(i, j)$ with $i < j$, compute the Euclidean distance between their centroids in PCA and in the embedding:
    \[
    D^{\text{PCA}}_{ij} = \| \mu_i^{\text{PCA}} - \mu_j^{\text{PCA}} \|, \quad D^{y}_{ij} = \| \mu_i^{y} - \mu_j^{y} \|
    \]
    
    \item For all unordered centroid triplets $(i,j,k)$ with $i < j < k$, compare the ordering of distances in PCA and in the embedding:
    \[
    \text{A triplet is \emph{preserved} if } \operatorname{sign}(D^{\text{PCA}}_{ij} - D^{\text{PCA}}_{ik}) = \operatorname{sign}(D^{y}_{ij} - D^{y}_{ik})
    \]
    
    \item The final PCA-guided triplet agreement score is the fraction of triplets with consistent ordering:
    \[
    \text{Score} = \frac{\text{\# of preserved triplets}}{\text{Total number of triplets}}
    \]
\end{enumerate}

This metric captures whether the embedding respects the global inter-class structure suggested by a linear reference model (PCA), without relying on individual point-level distances.

\subsection{Random Triplet PCA Score (\texttt{RandomTripletPCA})}

This metric evaluates whether global relationships between randomly sampled class-level centroids are preserved from PCA space to the embedding.

\begin{enumerate}
    \item Project the original data $\X$ into PCA space to obtain $\mathbf{y}_{\textrm{PCA}}$.
    
    \item For each class $c$, compute its centroid in PCA space and in the evaluated embedding $y$:
    \[
    \mu_c^{\text{PCA}} = \frac{1}{|C_c|} \sum_{i \in C_c} y_{\text{PCA},i}, \quad
    \mu_c^{y} = \frac{1}{|C_c|} \sum_{i \in C_c} y_i
    \]
    
    \item Randomly sample multiple triplets of distinct class indices $(i, j, k)$.
    
    \item For each triplet, compute the Euclidean distances between class centroids in PCA space and in the embedding:
    \[
    D^{\text{PCA}}_{ij} = \| \mu_i^{\text{PCA}} - \mu_j^{\text{PCA}} \|_2, \quad D^{y}_{ij} = \| \mu_i^{y} - \mu_j^{y} \|_2
    \]
    
    \item For each triplet, determine the relative ordering of distances:
    \[
    \text{Label}(i,j,k) = \mathbb{I}(D^{\text{PCA}}_{ij} < D^{\text{PCA}}_{ik}), \quad
    \text{Prediction}(i,j,k) = \mathbb{I}(D^{y}_{ij} < D^{y}_{ik})
    \]
    
    \item Compute the final agreement score as the proportion of triplets where the ordering is preserved:
    \[
    \text{Score} = \frac{1}{T} \sum_{t=1}^{T} \mathbb{I} \left( \text{Label}(t) = \text{Prediction}(t) \right)
    \]
\end{enumerate}

A higher score indicates that the embedding preserves the inter-class distance relationships suggested by PCA. Unlike the full triplet PCA score, this metric uses a randomized subset of triplets to provide a scalable, global evaluation. When preserving the local structure, this metric is slightly worse than the triplet PCA score since the higher score would ruin the local structure to be similar to PCA embedding. Therefore, when evaluating the performance of PCA alignment, we consider both.

\subsection{Silhouette Score}

The silhouette score was originally used for evaluating cluster quality in unsupervised settings. In this work, we apply it using true class labels as proxy clusters to assess class cohesion and separation.

\begin{enumerate}
    \item For each data point $i$ in the embedding $y$:
    \begin{itemize}
        \item Calculate average distance $a_i$ between $i$ and all other data points within the same class $C_i$,
        \[
        a_i = \frac{1}{|C_i| - 1} \sum_{j \in C_i, j \neq i} \|y_i-y_j\|_2
        \]
        \item Calculate the minimum average distance $b_i$ of $i$ to all points in other classes:
        \[
        b_i = \min_{C_k \neq C_i} \frac{1}{|C_k|} \sum_{j \in C_k} \|y_i-y_j\|_2
        \]
    \end{itemize}
    \item Compute silhouette score for point $i$:
    \[
    s_i = \frac{b_i - a_i}{\max(a_i, b_i)}
    \]
    \item The overall silhouette score is the average over all $N$ data points:
    \[
    S = \frac{1}{N} \sum_{i=1}^{N} s_i
    \]
    \item The higher the values, the better quality the embedding is.
\end{enumerate}

\subsection{SVM Classification Accuracy}

This metric evaluates how well the embedding supports non-linear classification by training a Support Vector Machine (SVM) with an RBF kernel and measuring its prediction accuracy. To improve efficiency, we apply a kernel approximation method.

\begin{enumerate}
    \item Apply the Nystr\"{o}m method, which approximates the
kernel matrix by a low rank matrix, using \texttt{sklearn.kernel\_approximation.Nystroem} to transform the embedding $\mathbf{Y} \in \mathbb{R}^{n \times d}$ into a higher-dimensional feature space $\Phi(\mathbf{Y}) \in \mathbb{R}^{n \times D}$ such that:
    \[
    K_{\text{RBF}}(\mathbf{y}_i, \mathbf{y}_j) \approx \langle \Phi(\mathbf{y}_i), \Phi(\mathbf{y}_j) \rangle
    \]
    
    \item Train a linear SVM classifier on the transformed features $\Phi(\mathbf{Y})$ using a one-vs-rest strategy for multi-class problems.

    \item Compute the classification accuracy over all $n$ data points:
    \[
    \text{Accuracy} = \frac{1}{n} \sum_{i=1}^{n} \mathbb{I}(\hat{y}_i = y_i)
    \]
    where $\hat{y}_i$ is the predicted label and $\mathbb{I}(\cdot)$ is the indicator function.
    
    \item Higher accuracy indicates that the embedding supports better class separation under non-linear decision boundaries.
\end{enumerate}

This metric is done under a 5-fold setup in the experiments (each time using 4 folds as the training data for the SVM model and using the remaining fold for the evaluation of accuracy), which captures the global separability of classes in the embedding space using a kernelized classifier.

\clearpage
\section{Data Description}
\label{app:datasets}

In our experiments, we evaluated a diverse collection of datasets spanning text, image, and biological domains to examine the effects of different alignment regularizations. Table~\ref{tab:data_description} summarizes each dataset along with the number of samples, features, task types, and label descriptions. Datasets such as MNIST \citep{lecun2010mnist}, FMNIST \citep{fmnist}, USPS \citep{USPS}, and COIL20 \citep{coil20} represent image data, while Human Cortex \citep{human_cortex_data}, Kang et al. \citep{kang2018multiplexed}, CBMC \citep{neurips2021_data}, and Stuart et al. \citep{stuart2019comprehensive} focus on single-cell transcriptomic measurements. Tabular datasets such as FICO \citep{fico} were also included to assess generalizability. The alignment tasks applied to each dataset fall into two categories: PCA alignment, which enforces global structure consistency with principal component axes, and concept alignment, which aligns the embedding with user-defined or interpretable axes. Most datasets were evaluated under both PCA and concept alignment settings, while others like Airplane and Mammoth only used PCA alignment due to their clear spatial or structural geometry. This comprehensive selection allows us to test the robustness and utility of our methods across a variety of domains and structures.
\begin{table}[ht]

\centering
\caption{Data Descriptions and Task Types each dataset corresponds to}
\label{tab:data_description}
\begin{tabular}{p{5cm}llll}
\hline
\textbf{Dataset} & \textbf{\# Samples} & \textbf{\# features} & \textbf{Tasks Types} & \textbf{Labels}     \\ \hline
Airplane \citep{airplane}        & 24,141                     & 3                           & PCA                  & Airplane Structure  \\
COIL20  \citep{coil20}         & 1,440                      & 16,384                      & Concept, PCA         & 20 Objects          \\
FICO \citep{fico}           & 9,861                      & 23                          & Concept              & External Risk Score \\
FMNIST  \citep{fmnist}         & 70,000                     & 784                         & Concept, PCA         & 10 Clothes Type     \\
Human Cortex \citep{human_cortex_data}    & 43,349                     & 100                         & Concept, PCA         & 9 Cell Types        \\
Kang et al. \citep{kang2018multiplexed}            & 13,999                     & 100                         & Concept, PCA         & 13 Cell Types       \\
Mammoth   \citep{UmapMammoth,SmithsonianMammoth}       & 10,000                     & 3                           & PCA                  & Mammoth Structure   \\
MNIST  \citep{lecun2010mnist}          & 70,000                     & 784                         & Concept, PCA         & Digits 0-9          \\
CBMC  \citep{neurips2021_data}           & 67686                      & 100                         & Concept, PCA         & 9 Cell Types        \\
Stuart et al. \citep{stuart2019comprehensive}   & 30,672                     & 100                         & Concept, PCA         & 27 Cell Types       \\
USPS \citep{USPS}            & 9298                       & 256                         & Concept, PCA         & Digits 0-9          \\ \hline
\end{tabular}
\end{table}

\clearpage
\section{Details of Experimental Design}
\label{app:detailed_experiment}

To construct the Rashomon set for each DR algorithm, we generate a comprehensive pool of 235 embeddings per alignment task. This is achieved by varying the alignment regularization strength—denoted as the label weight parameters $\lambda_{\textrm{PCA}}$ and $\lambda_{\textrm{Axis}}$ (defined in Section \ref{sec:alignment})—across 47 candidate values and sampling five different random seeds for each configuration. The list of $\lambda$ candidates spans a wide range of values from near-zero to large magnitudes, specifically
$\{0.0, 0.001, 0.002, \dots, 0.009, 0.01, 0.02, \dots, 0.1, 0.2, \dots, 1.0, 2.0, \dots, 100.0\}.
$ For concept-aware regularizer, we are using 10\% of the data are labeled. For the influence of missingness ratio, label weights and their corresponding loss function, we have shown an example of MNIST in Appendix \ref{app:missing_ratio} using \PaCMAP. 

For all DR methods other than PaCMAP param, we apply a scaling factor of 10,000 to the label weight values during implementation, ensuring consistent influence across algorithms with different objective function scales. In concept-informed DR, the dataset’s class labels serve as the alignment concept to guide embedding formation.

To determine which embeddings are included in the Rashomon set, we evaluate the $\L_\DR$ across all runs and identify a threshold beyond which embeddings show a significant degradation in performance. This threshold is defined as the point where the loss curve demonstrates a statistically significant increase compared to the minimum observed loss—measured by a consistent gap between the stable region and the rising region of the loss curve. Embeddings whose loss values fall within the low-loss plateau are retained as members of the Rashomon set, while those in the degraded region are excluded.

From the selected Rashomon set of embeddings, we first extract a common 50-nearest neighbor (50-NN) graph following the procedure described in Section~\ref{sec:common_knowledge}. For each data point, we then compute edge scores $\Psi$ using Equation~\ref{eq:edge_importance} to evaluate the reliability of its neighbors. Among the 50 neighbors, we retain the number of edges $k$ (the $k$ is defined as the default setting for each of the methods, e.g. for \PaCMAP, $k=10$) with the lowest scores—those deemed most trustworthy—and construct a refined $k$-NN graph from these pairs. We then reconstruct a new embedding using this filtered neighbor graph. All dimensionality reduction (DR) methods are run with their default settings, as our goal is not to compare across DR methods but rather to assess the improvement in embedding quality before and after incorporating common knowledge graph. A detailed version of the algorithm is shown in Algorithm~\ref{alg:refined_embedding_rank} and \ref{alg:refined_embedding_score}.Although different datasets and methods may yield Rashomon sets of varying sizes, current experiments have shown that as few as five embeddings are sufficient to demonstrate improvements through common knowledge extraction.

All experiments were conducted on a machine with Intel(R) Xeon(R) Gold 5317 CPU @ 3.00GHz CPU and an NVIDIA A5000 GPU with 128GB Memory.

\paragraph{Common Knowledge Embedding Runtime Breakdown}The overall runtime of the common knowledge extraction pipeline can be broken down into three main stages: (i) \textbf{Nearest Neighbor Graph Construction}: This step is computed once and reused across all runs. (ii) \textbf{Embedding Construction}: This step is performed separately for each DR configuration, but it is highly parallelizable. Multiple embeddings can be processed concurrently per GPU, so the runtime does not scale linearly with the number of embeddings. (iii) \textbf{Common Knowledge Extraction}: This step aggregates stable neighbors across embeddings and is relatively lightweight compared to the other stages. In practice, constructing a Rashomon set of several embeddings and performing common knowledge extraction introduces only moderate overhead compared to a single DR run, while yielding significant robustness and interpretability benefits.

\section{Additional PCA-aligned DR results}
\label{app:pca_results}

Here we provide PCA-aligned and original DR embeddings across five methods—InfoNCE,
NCVis, Neg-tSNE, \UMAP, and \PaCMAP—on a set of datasets discussed in Appendix \ref{app:datasets}. Each subplot shows a 2D embedding colored by class labels, with the left column representing
original embeddings and the middle column showing the PCA embedding and the right column showing aligned counterparts. The alignment enforces
PCA-aware positioning while preserving local structure. The bar plots (bottom) quantify
performance using multiple metrics: soft Jaccard distance, triplet satisfaction, random
triplet discrimination, and $\L_\DR$. Aligned DR embeddings consistently maintain
structure quality across all metrics, indicating enhanced trustworthiness and interpretability.

\begin{figure}
    \centering
    \includegraphics[width=0.8\linewidth]{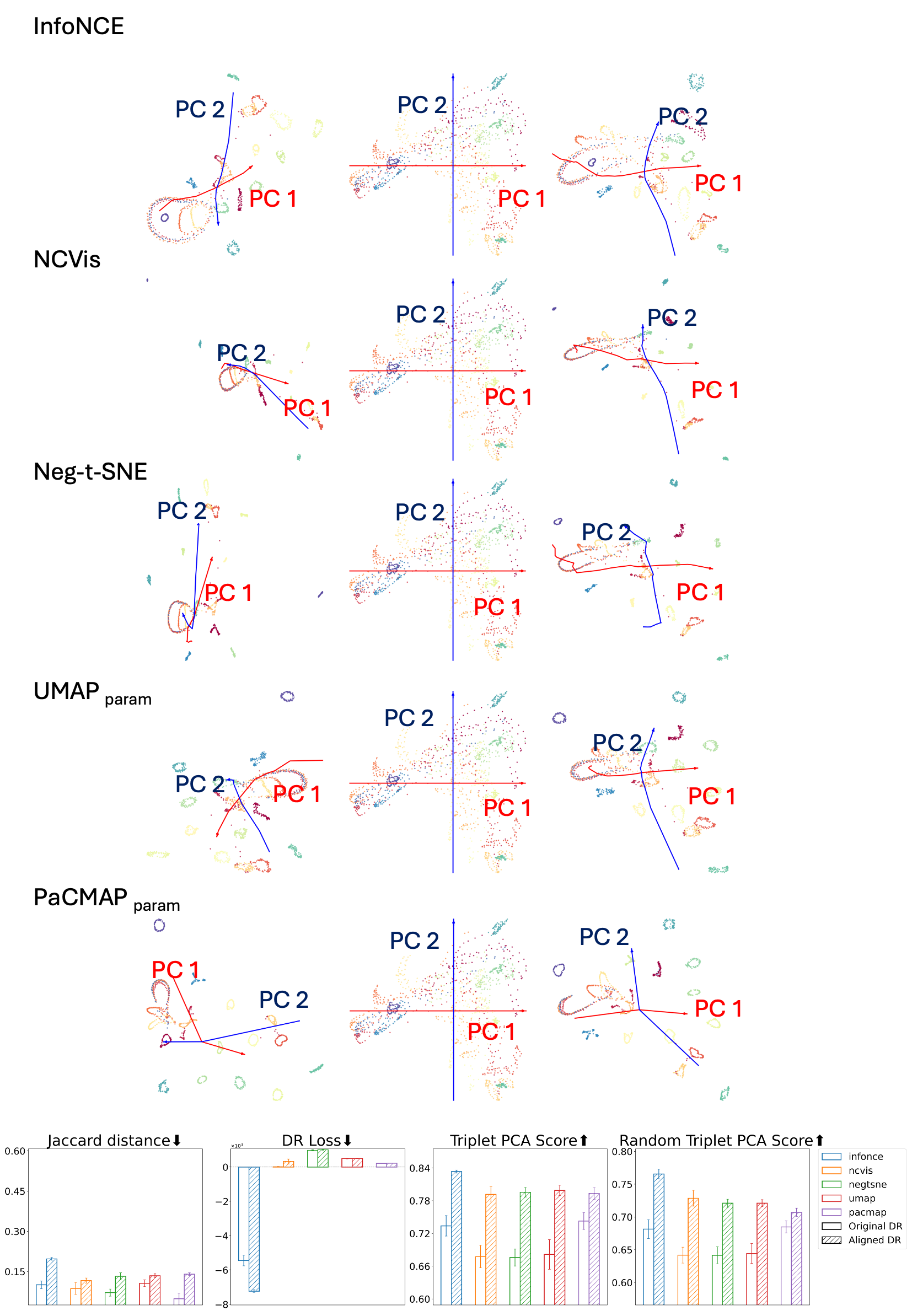}
    \caption{Comparison of original COIL20 embedding (left) PCA embedding (middle) and PCA informed embeddings (right) across different methods. We see alignment to principal components across all methods while preserving structure. We show that soft Jaccard distance, triplet PCA score, random triplet PCA score, and $\L_\DR$ (bottom) remain mostly unchanged and that aligned embeddings consistently maintain
structure. }
    \label{fig:appendix_e_coil20}
\end{figure}

\begin{figure}
    \centering
    \includegraphics[width=0.8\linewidth]{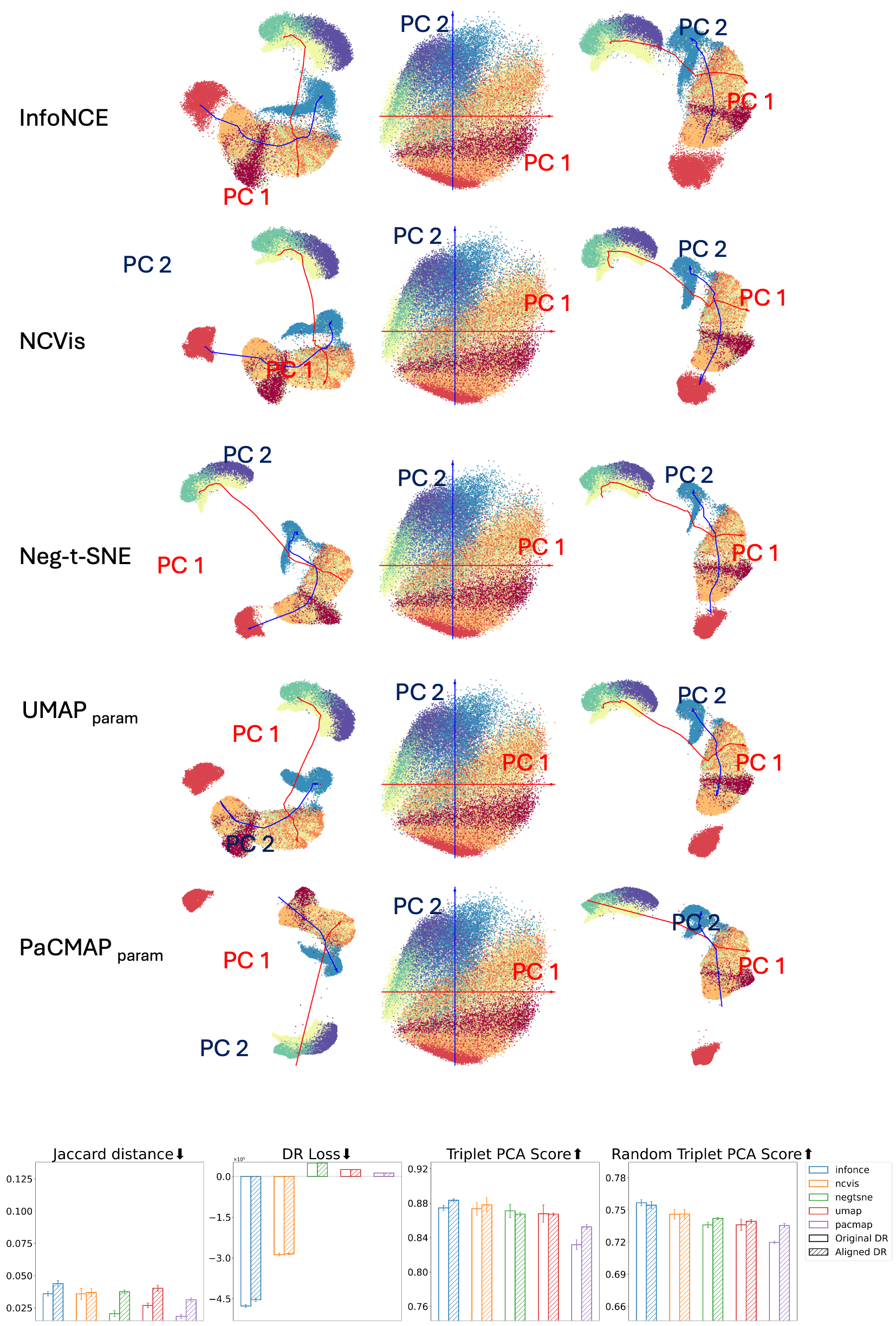}
    \caption{Comparison of original FMNIST embedding (left) PCA embedding (middle) and PCA informed embeddings (right) across different methods. We see alignment to principal components across all methods while preserving structure. We show that soft Jaccard distance and $\L_\DR$ (bottom) remain mostly unchanged and that aligned embeddings consistently maintain structure. Random Triplet PCA score and Triplet PCA score have shown an improvement after the alignment. }
    \label{fig:appendix_e_FMNIST}
\end{figure}

\begin{figure}
    \centering
    \includegraphics[width=0.8\linewidth]{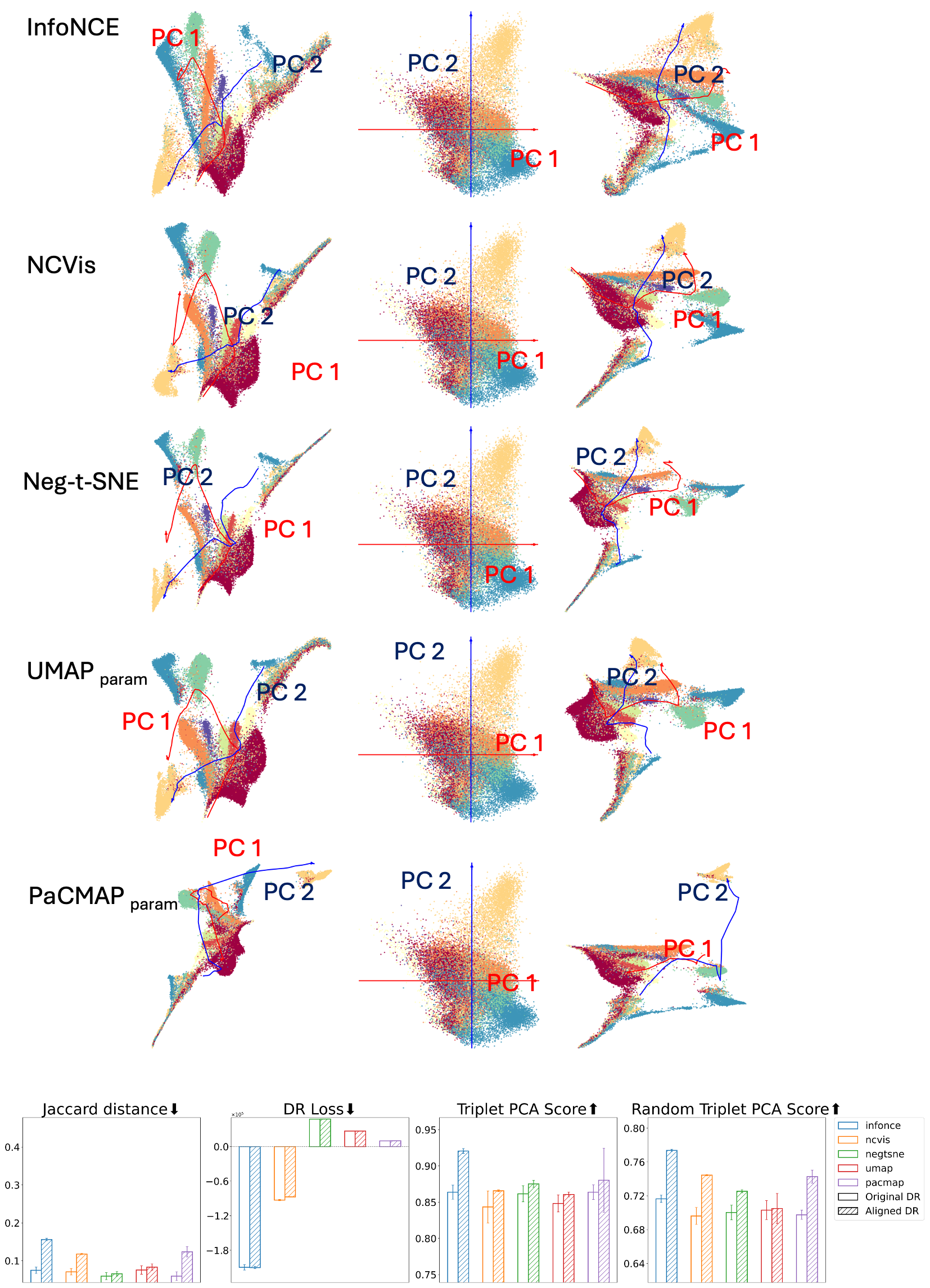}
    \caption{Comparison of original Human Cortex embedding (left) PCA embedding (middle) and PCA informed embeddings (right) across different methods. We see alignment to principal components across all methods while preserving structure. We show that soft Jaccard distance and $\L_\DR$ (bottom) remain mostly unchanged and that aligned embeddings consistently maintain structure. Random Triplet PCA score and Triplet PCA score have shown an improvement after the alignment. }
    \label{fig:appendix_e_human_cortex}
\end{figure}

\begin{figure}
    \centering
    \includegraphics[width=0.8\linewidth]{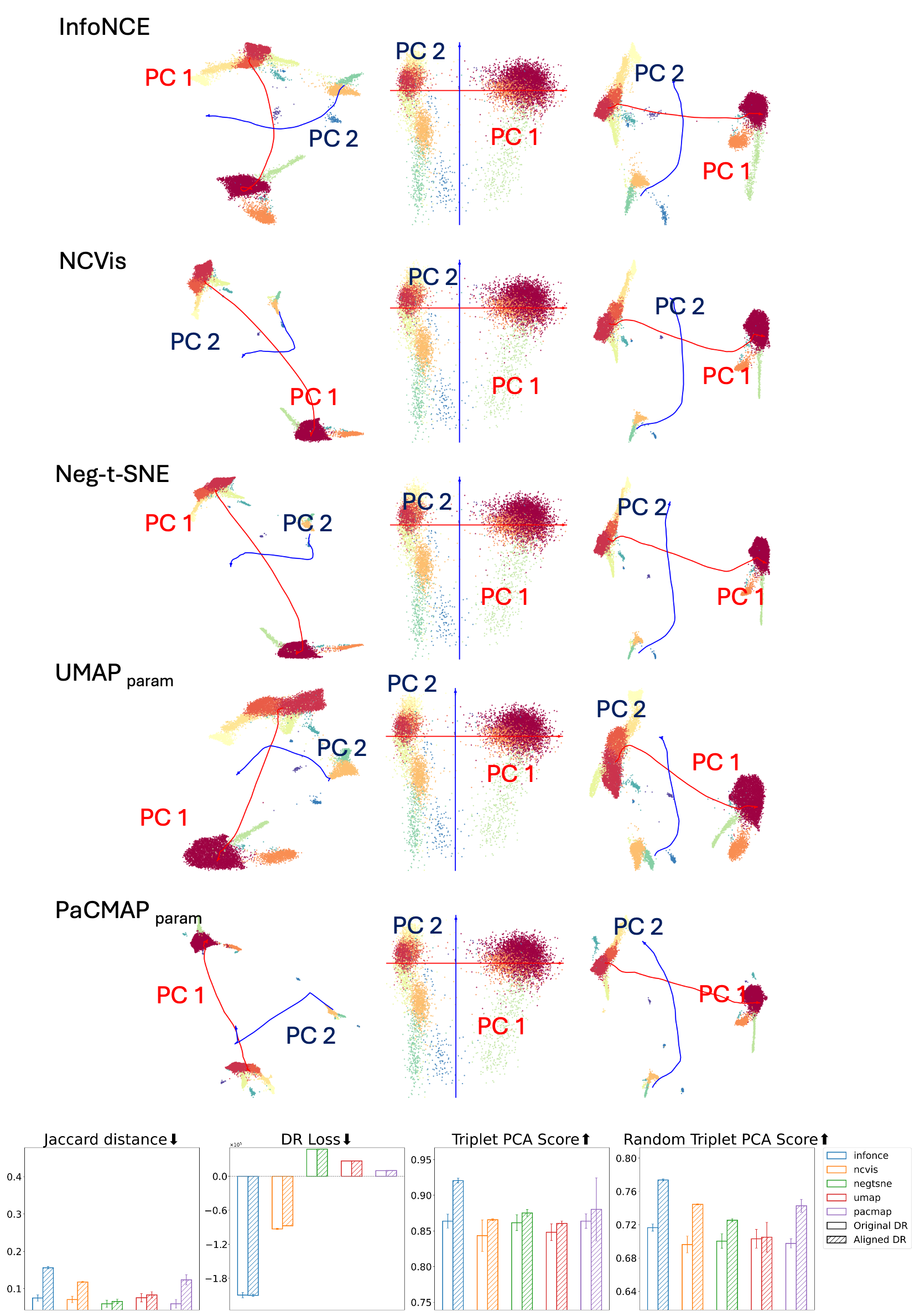}
    \caption{Comparison of original Kang et al. embedding (left) PCA embedding (middle) and PCA informed embeddings (right) across different methods. We see alignment to principal components across all methods while preserving structure. We show that soft Jaccard distance and $\L_\DR$ (bottom) remain mostly unchanged and that aligned embeddings consistently maintain structure. Random Triplet PCA score and Triplet PCA score have shown an improvement after the alignment.}
    \label{fig:appendix_e_kang}
\end{figure}

\begin{figure}
    \centering
    \includegraphics[width=0.8\linewidth]{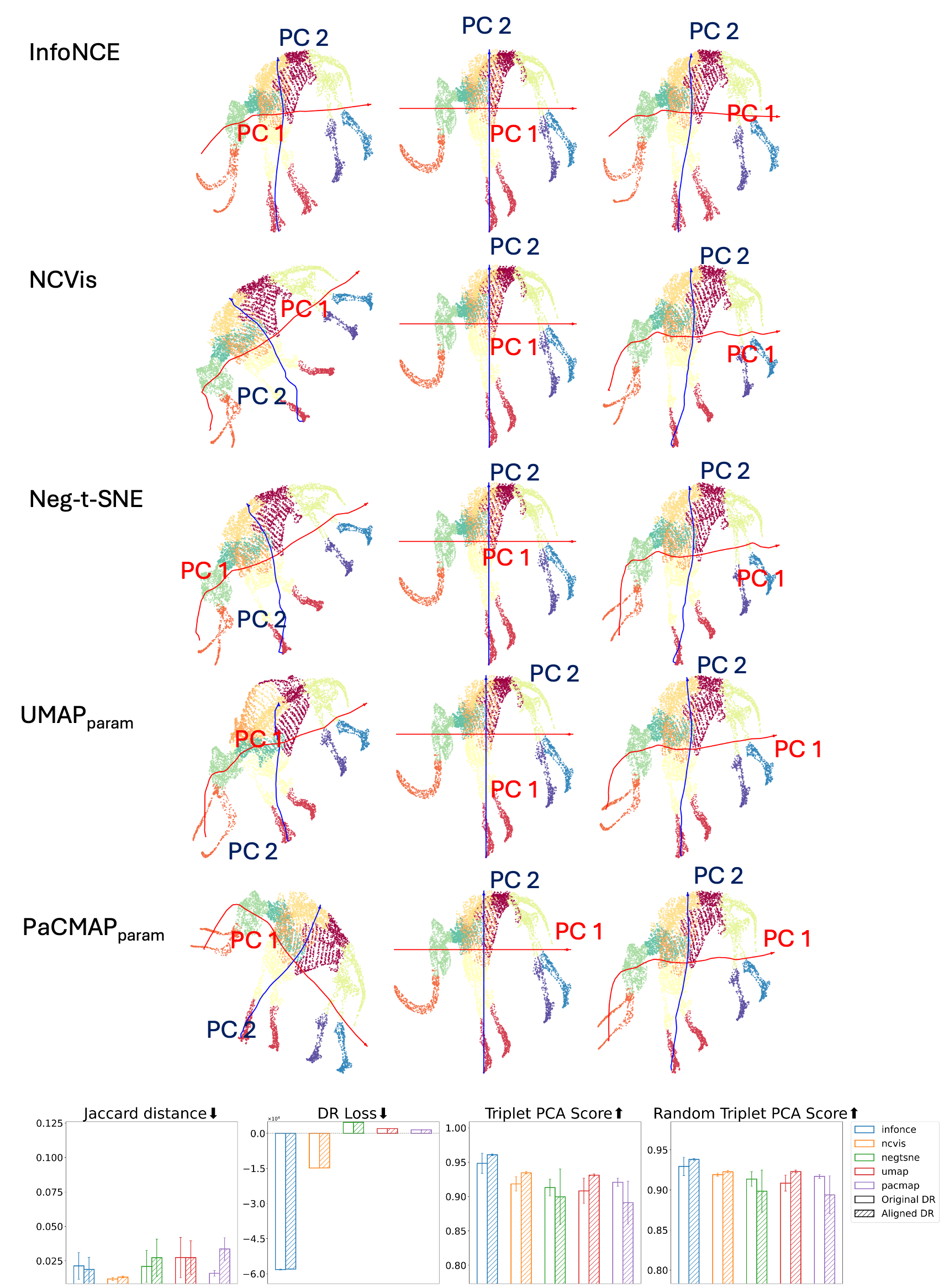}
    \caption{Comparison of original Mammoth embedding (left) PCA embedding (middle) and PCA informed embeddings (right) across different methods. We see alignment to principal components across all methods while preserving structure. We show that soft Jaccard distance and $\L_\DR$ (bottom) remain mostly unchanged and that aligned embeddings consistently maintain structure. Random Triplet PCA score and Triplet PCA score have shown an improvement after the alignment. }
    \label{fig:appendix_e_Mammoth}
\end{figure}

\begin{figure}
    \centering
    \includegraphics[width=0.8\linewidth]{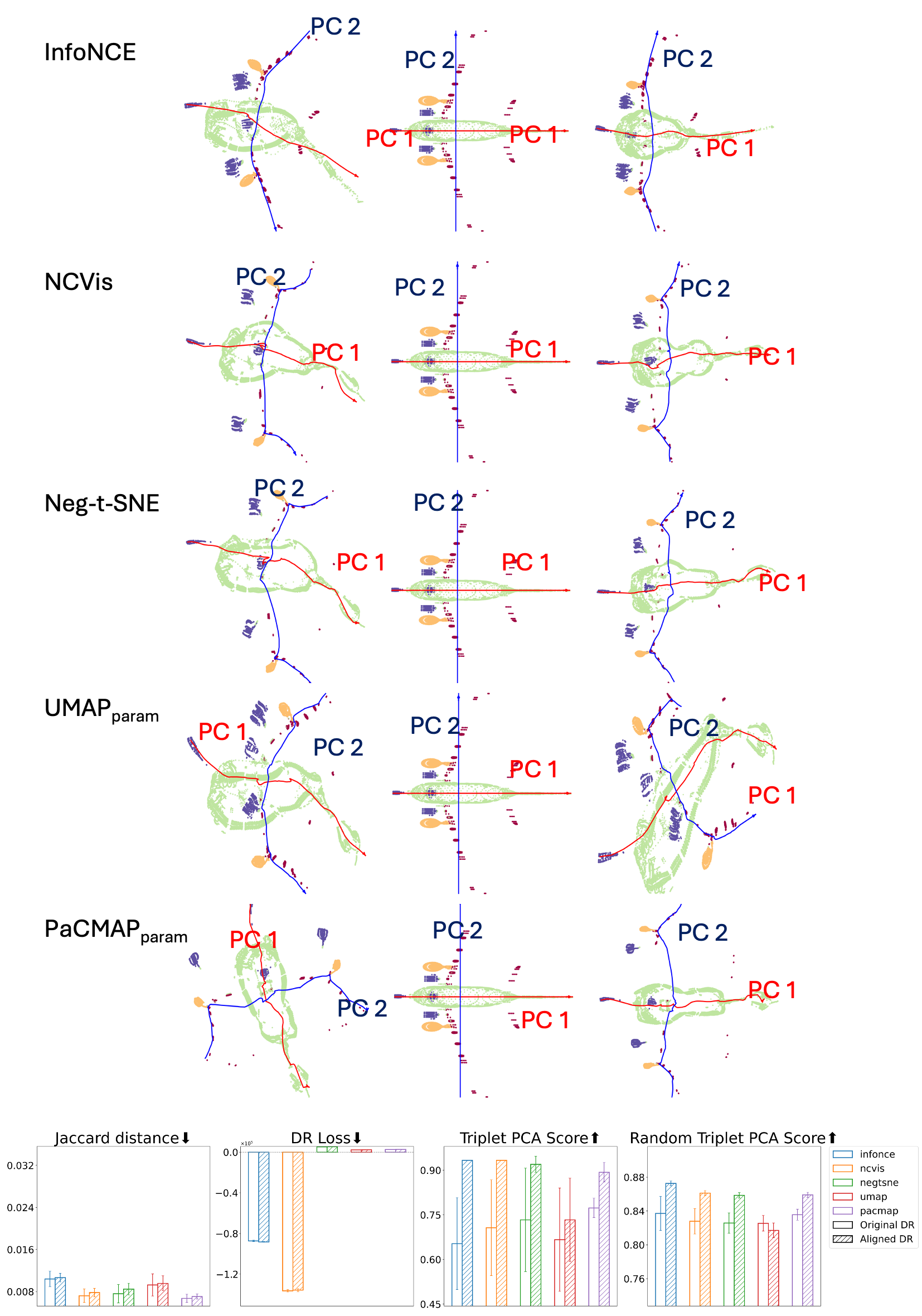}
    \caption{Comparison of original Airplane embedding (left) PCA embedding (middle) and PCA informed embeddings (right) across different methods. We see alignment to principal components across all methods while preserving structure. We show that soft Jaccard distance and $\L_\DR$ (bottom) remain mostly unchanged and that aligned embeddings consistently maintain structure. Random Triplet PCA score and Triplet PCA score have shown an improvement after the alignment.}
    \label{fig:appendix_e_airplane}
\end{figure}

\begin{figure}
    \centering
    \includegraphics[width=0.8\linewidth]{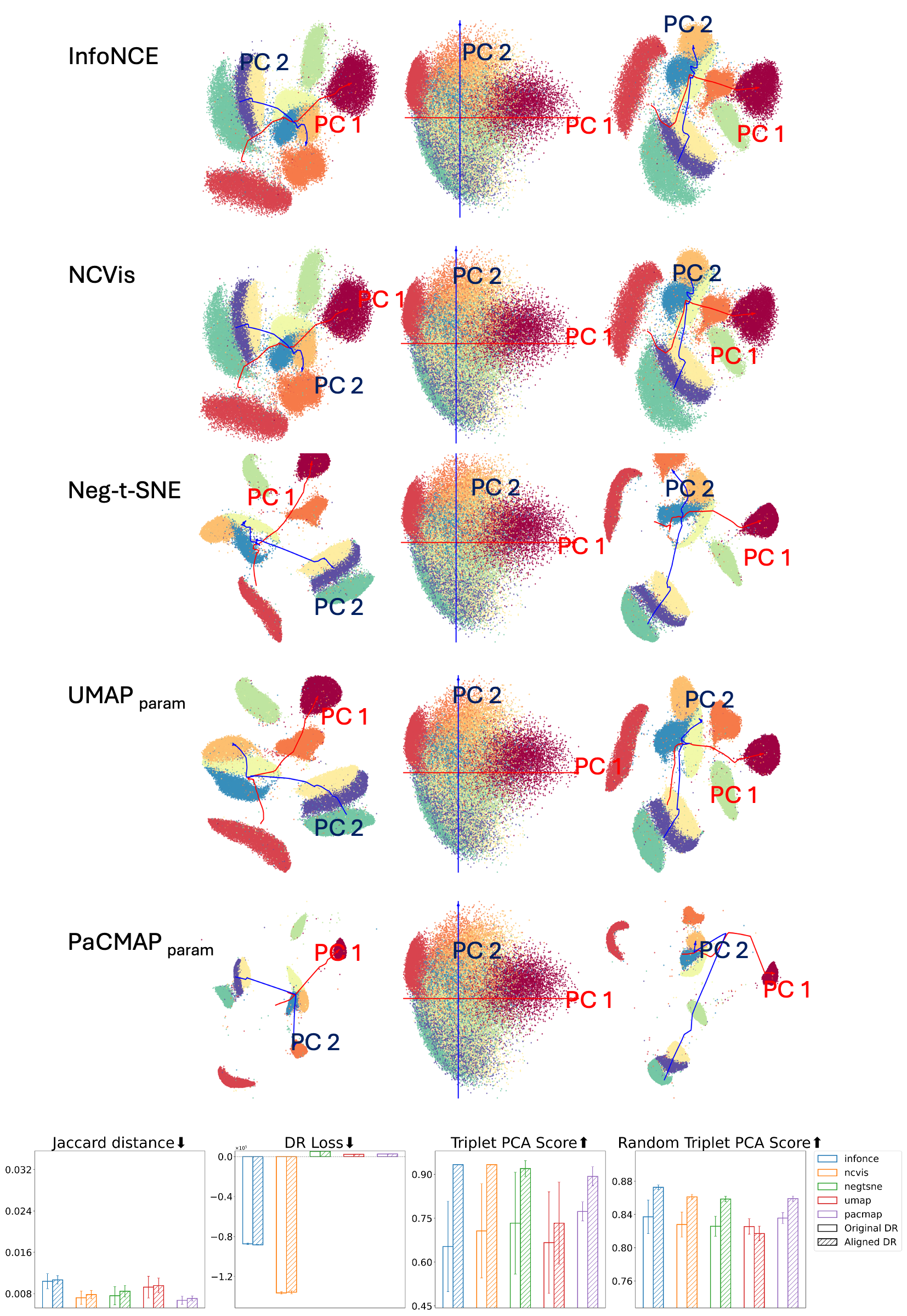}
    \caption{Comparison of original MNIST embedding (left) PCA embedding (middle) and PCA informed embeddings (right) across different methods. We see alignment to principal components across all methods while preserving structure. We show that soft Jaccard distance, and $\L_\DR$ (bottom) remain mostly unchanged and that aligned embeddings consistently maintain structure. }
    \label{fig:appendix_e_MNIST}
\end{figure}

\begin{figure}
    \centering
    \includegraphics[width=0.8\linewidth]{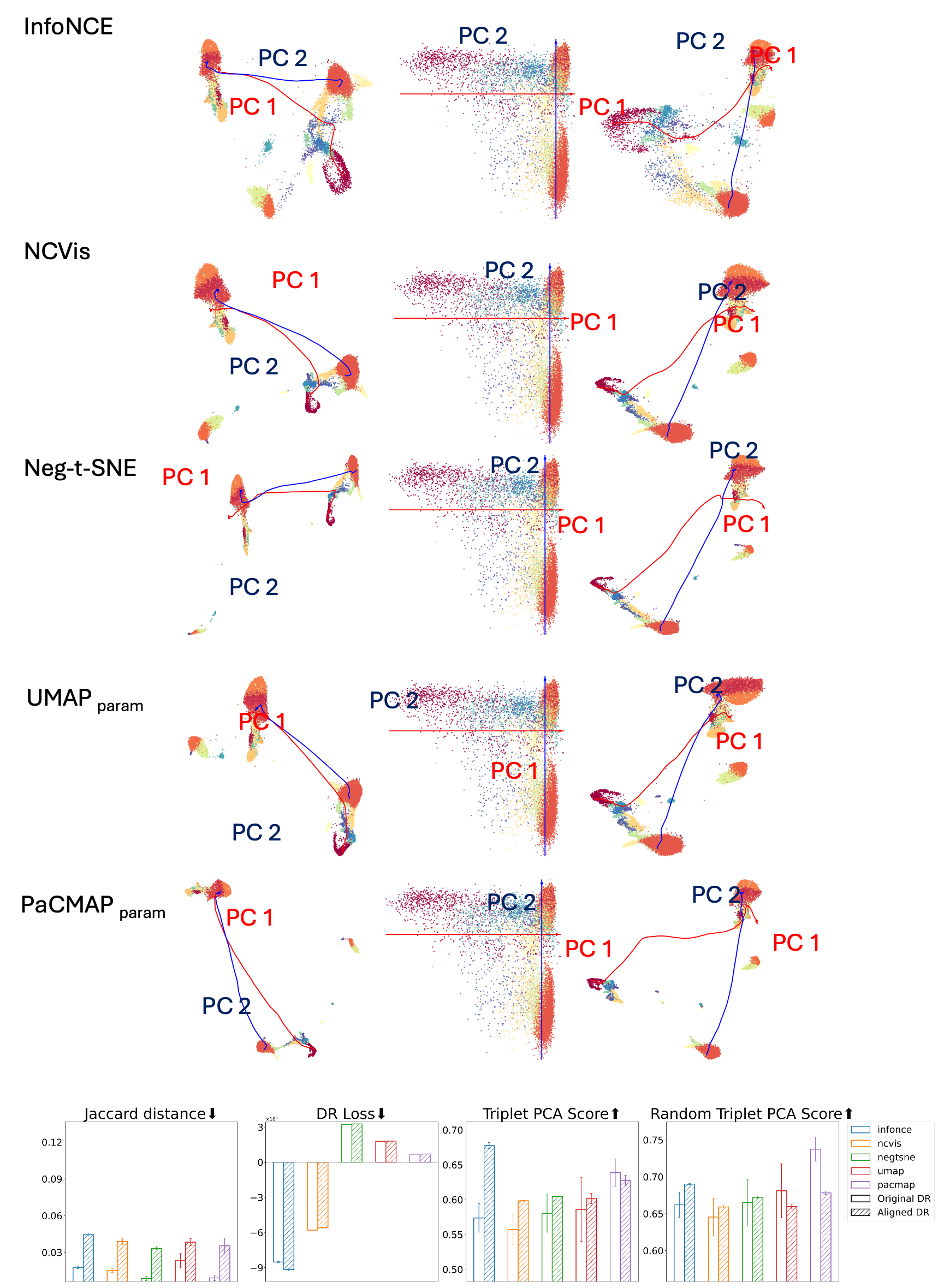}
    \caption{Comparison of original Stuart et al$.$ embedding (left) PCA embedding (middle) and PCA informed embeddings (right) across different methods. We see alignment to principal components across all methods while preserving structure. We show that soft Jaccard distance and $\L_\DR$ (bottom) remain mostly unchanged and that aligned embeddings consistently maintain structure. Random Triplet PCA score and Triplet PCA score have shown an improvement after the alignment.}
    \label{fig:appendix_e_seurat}
\end{figure}

\begin{figure}
    \centering
    \includegraphics[width=0.8\linewidth]{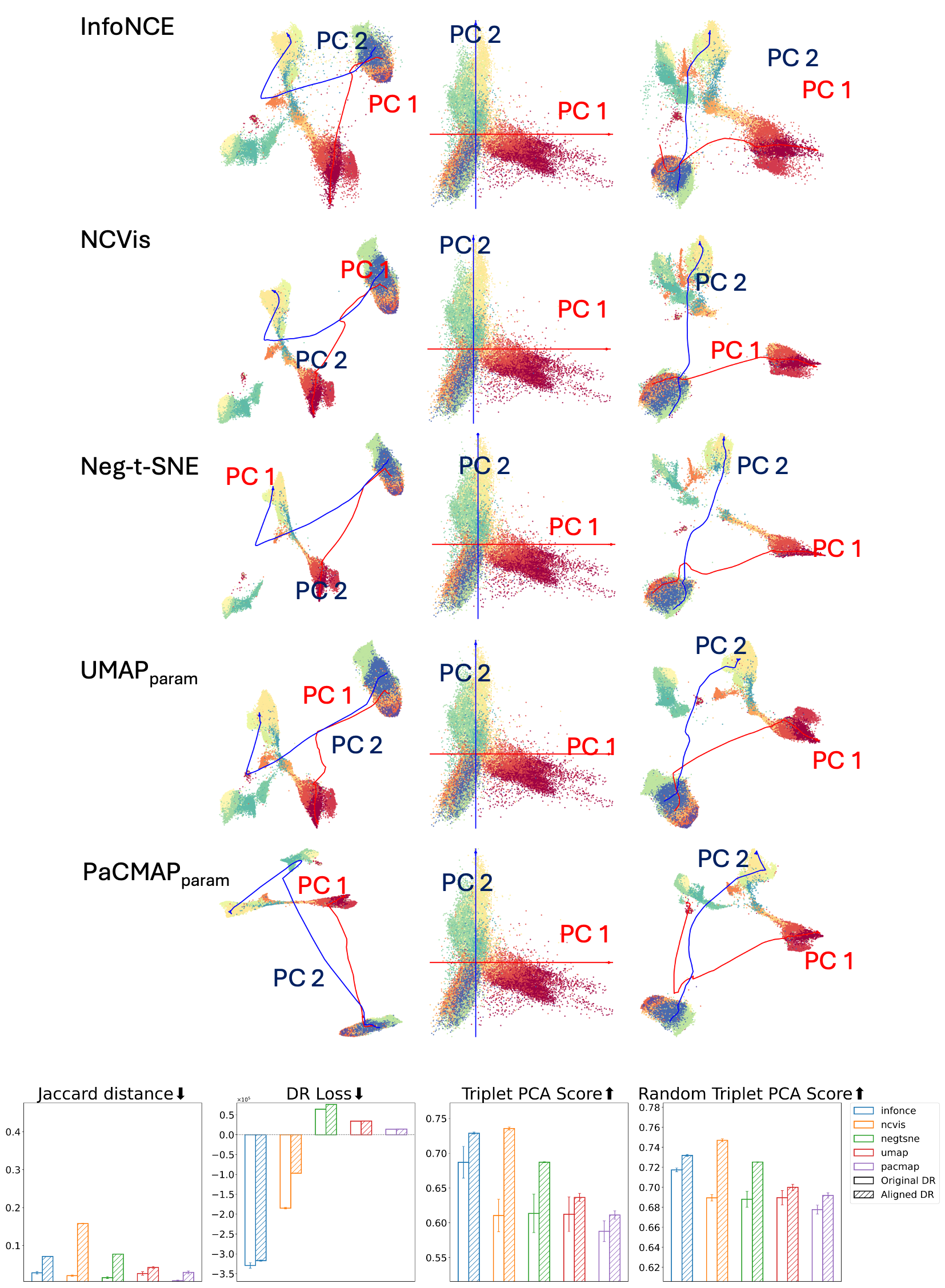}
    \caption{Comparison of original CBMC embedding (left) PCA embedding (middle) and PCA informed embeddings (right) across different methods. We see alignment to principal components across all methods while preserving structure. We show that soft Jaccard distance, triplet PCA score, random triplet PCA score, and $\L_\DR$ (bottom) remain mostly unchanged and that aligned embeddings consistently maintain
structure. }
    \label{fig:appendix_e_neurips}
\end{figure}

\begin{figure}
    \centering
    \includegraphics[width=0.8\linewidth]{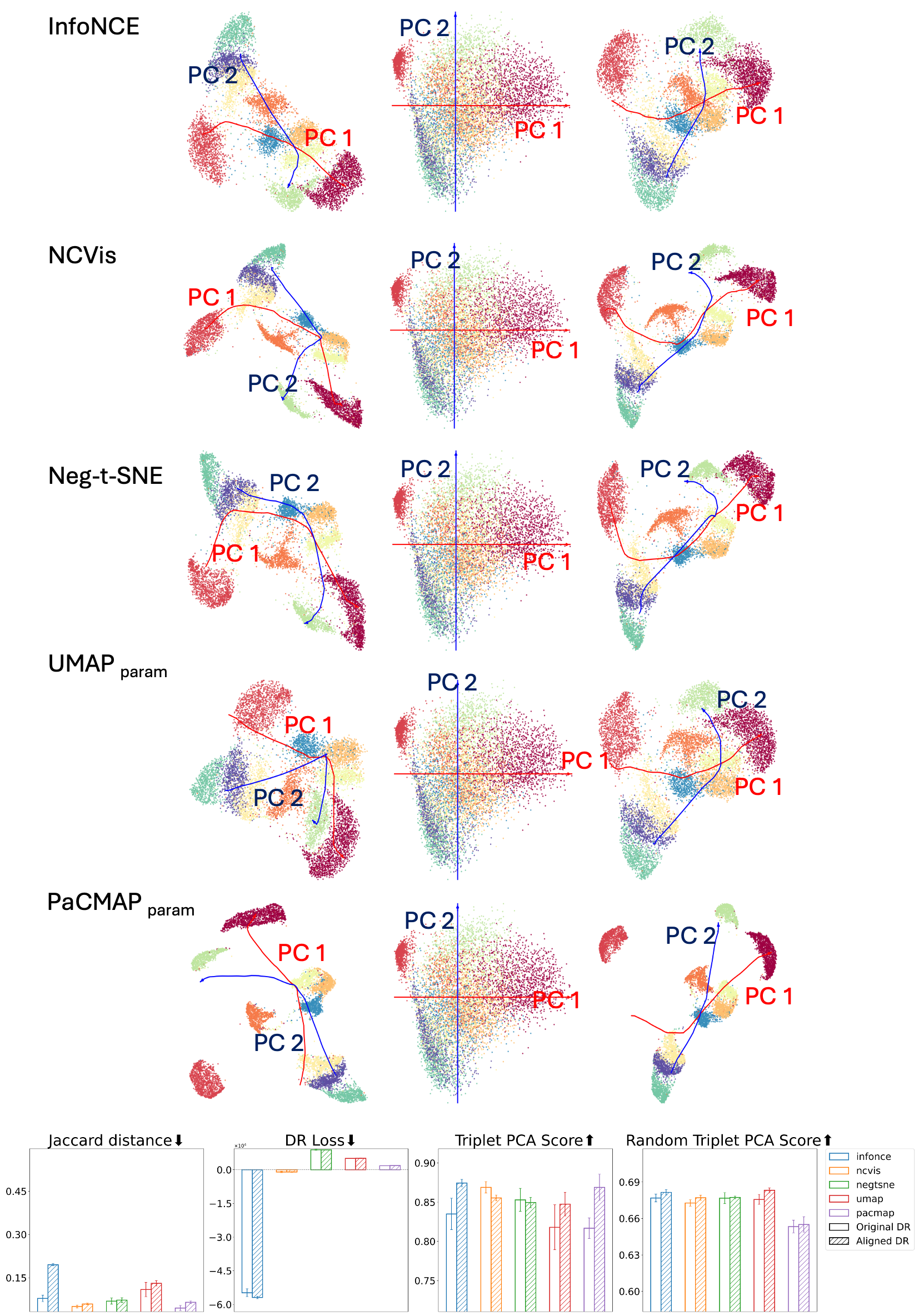}
    \caption{Comparison of original USPS embedding (left) PCA embedding (middle) and PCA informed embeddings (right) across different methods. We see alignment to principal components across all methods while preserving structure and $\L_\DR$$\L_\DR$ (bottom) remaining mostly unchanged and that aligned embeddings consistently maintain structure. Random Triplet PCA score and Triplet PCA score have shown an improvement after the alignment.}
    \label{fig:appendix_e_USPS}
\end{figure}

\section{Additional Concept-aware DR Results}
\label{app:concept_results}

Here we are showing a list of aligned and original DR embeddings across five methods—InfoNCE, NCVis, Neg-tSNE, UMAP, and PaCMAP -- on a set of datasets that have been mentioned in Appendix \ref{app:datasets}. Each subplot shows a 2D embedding colored by class label, with the left column representing original embeddings and the right column showing aligned counterparts. The alignment enforces concept-aware positioning while preserving local structure. The bar plots (bottom right) quantify performance using multiple metrics: 5-NN accuracy, soft Jaccard distance, triplet satisfaction, random triplet discrimination, silhouette score, and $\L_\DR$. Aligned DR embeddings consistently maintain structure quality across all metrics, indicating enhanced trustworthiness and interpretability.

\begin{figure}[ht]
    \centering
    \includegraphics[width=0.8\linewidth]{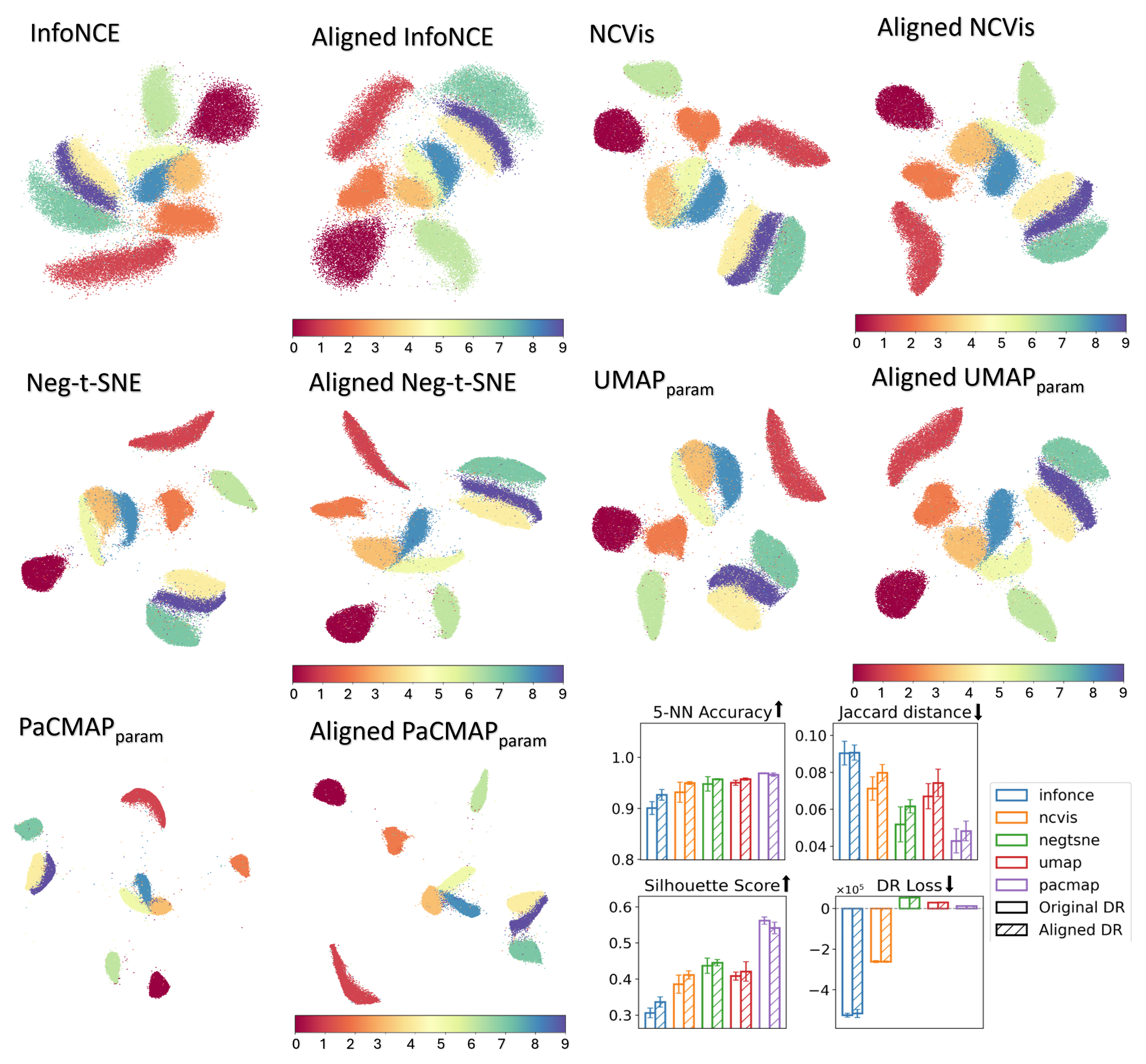}
    \caption{Comparison of original and aligned embeddings for MNIST using a concept-aware regularizer. The embeddings' horizontal axis is aligned with the digit index and embedding evaluation metrics haven't changed significantly after alignment.}
    \label{fig:mnist_axis}
\end{figure}

\begin{figure}[ht]
    \centering
    \includegraphics[width=0.8\linewidth]{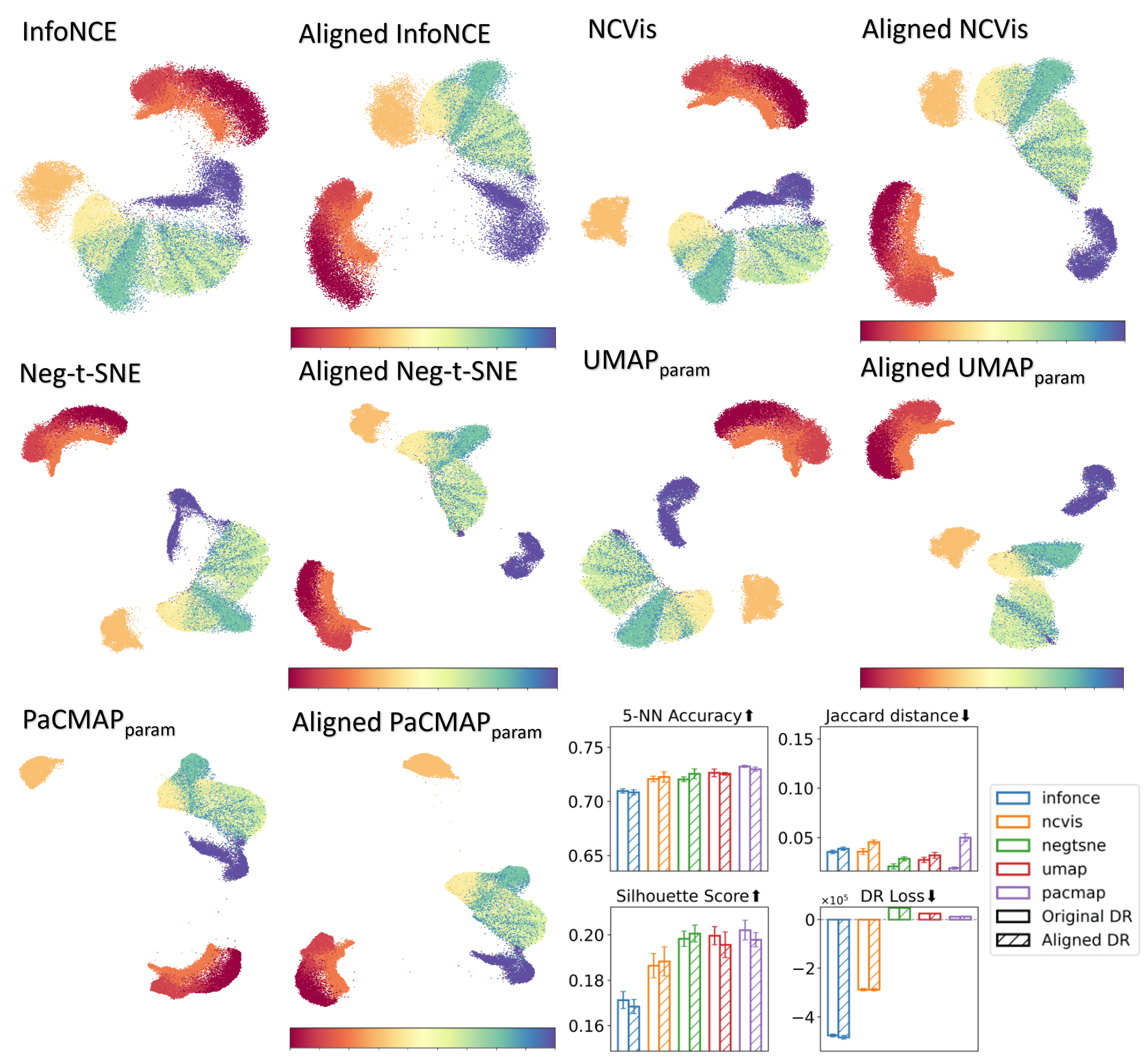}
    \caption{Comparison of original and aligned embeddings for FMNIST using a concept-aware regularizer. The embeddings' horizontal axis is aligned with the label index (from head to toe as mentioned in Section \ref{exp1}) and embedding evaluation metrics haven't changed significantly after alignment.}
    \label{fig:fmnist_axis}
\end{figure}

\begin{figure}[ht]
    \centering
    \includegraphics[width=0.8\linewidth]{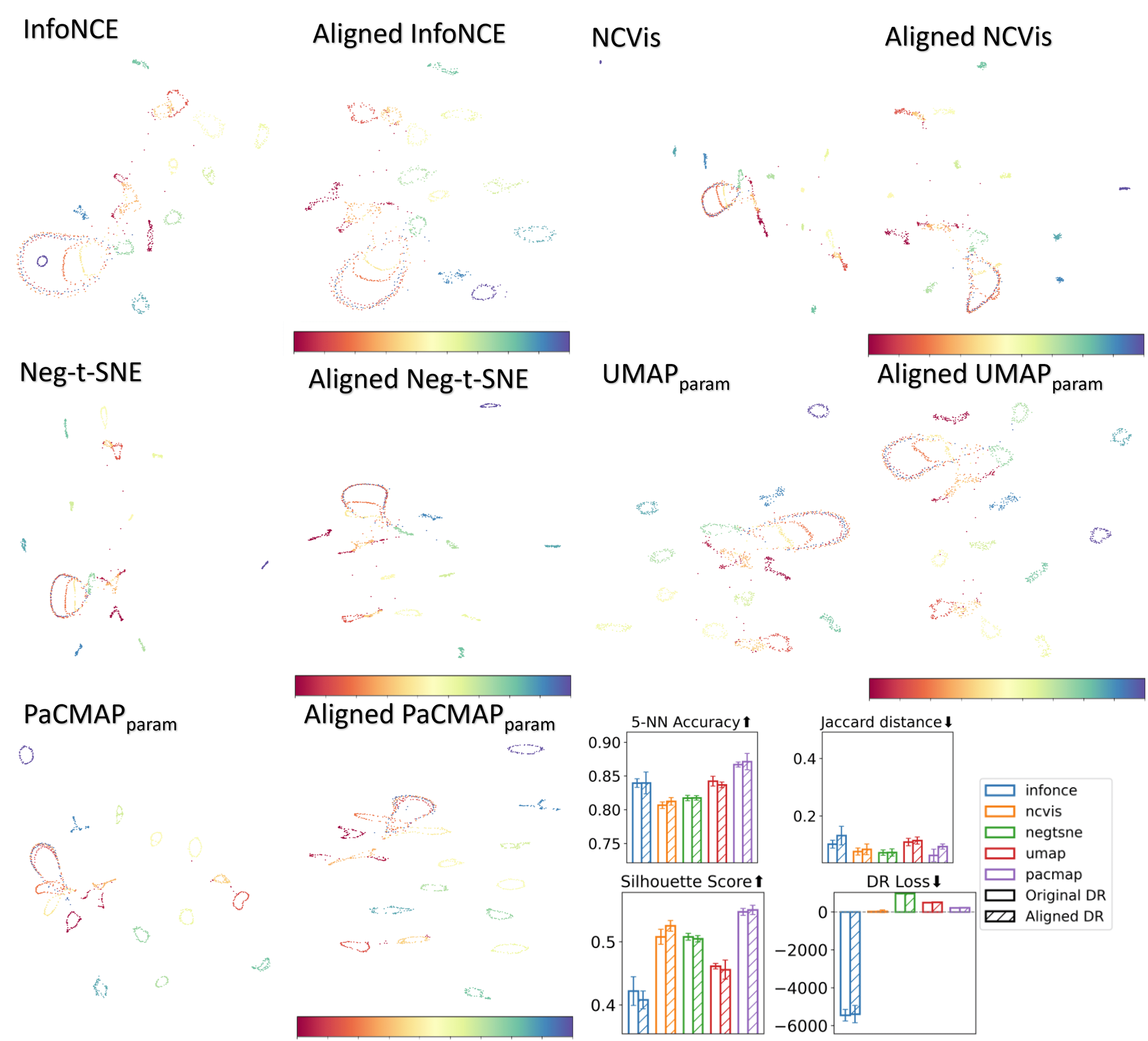}
    \caption{Comparison of original and aligned embeddings for COIL20 using a concept-aware regularizer. The embeddings' horizontal axis is aligned with the label index and embedding evaluation metrics haven't changed significantly after alignment.}
    \label{fig:coil20_axis}
\end{figure}

\begin{figure}[ht]
    \centering
    \includegraphics[width=0.8\linewidth]{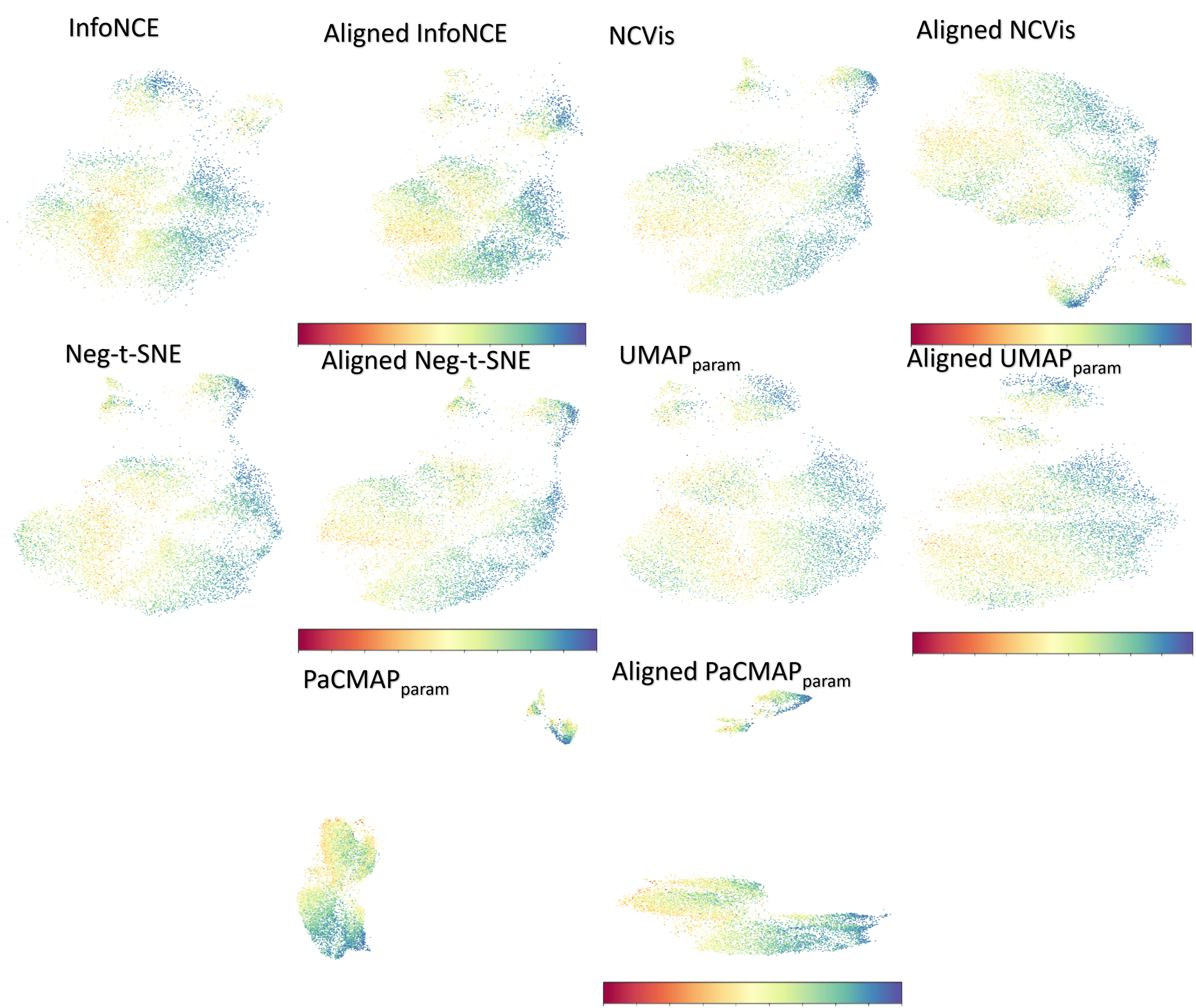}
    \caption{Comparison of original and aligned embeddings for FICO using a concept-aware regularize. The embeddings' horizontal axis is aligned with the external risk score. Since the label is a continuous feature, we are not displaying barplot here.}
    \label{fig:fico_axis}
\end{figure}

\begin{figure}[ht]
    \centering
    \includegraphics[width=0.8\linewidth]{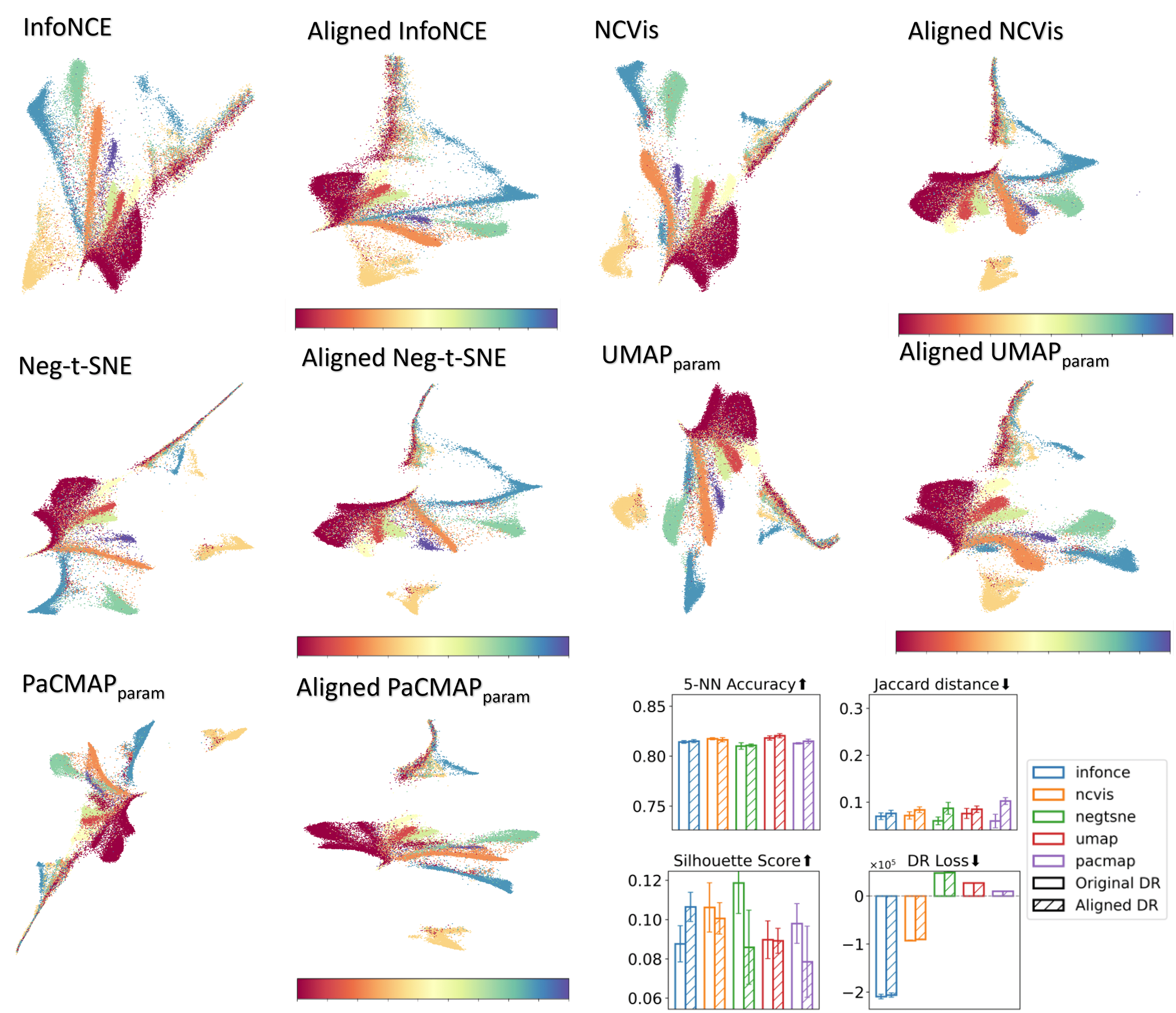}
    \caption{Comparison of original and aligned embeddings for Human Cortex Single Cell dataset using a concept-aware regularizer. The embeddings' horizontal axis is aligned with the cell type index, and the embedding performance hasn't been changed significantly.}
    \label{fig:human_cortex_axis}
\end{figure}

\begin{figure}[ht]
    \centering
    \includegraphics[width=0.8\linewidth]{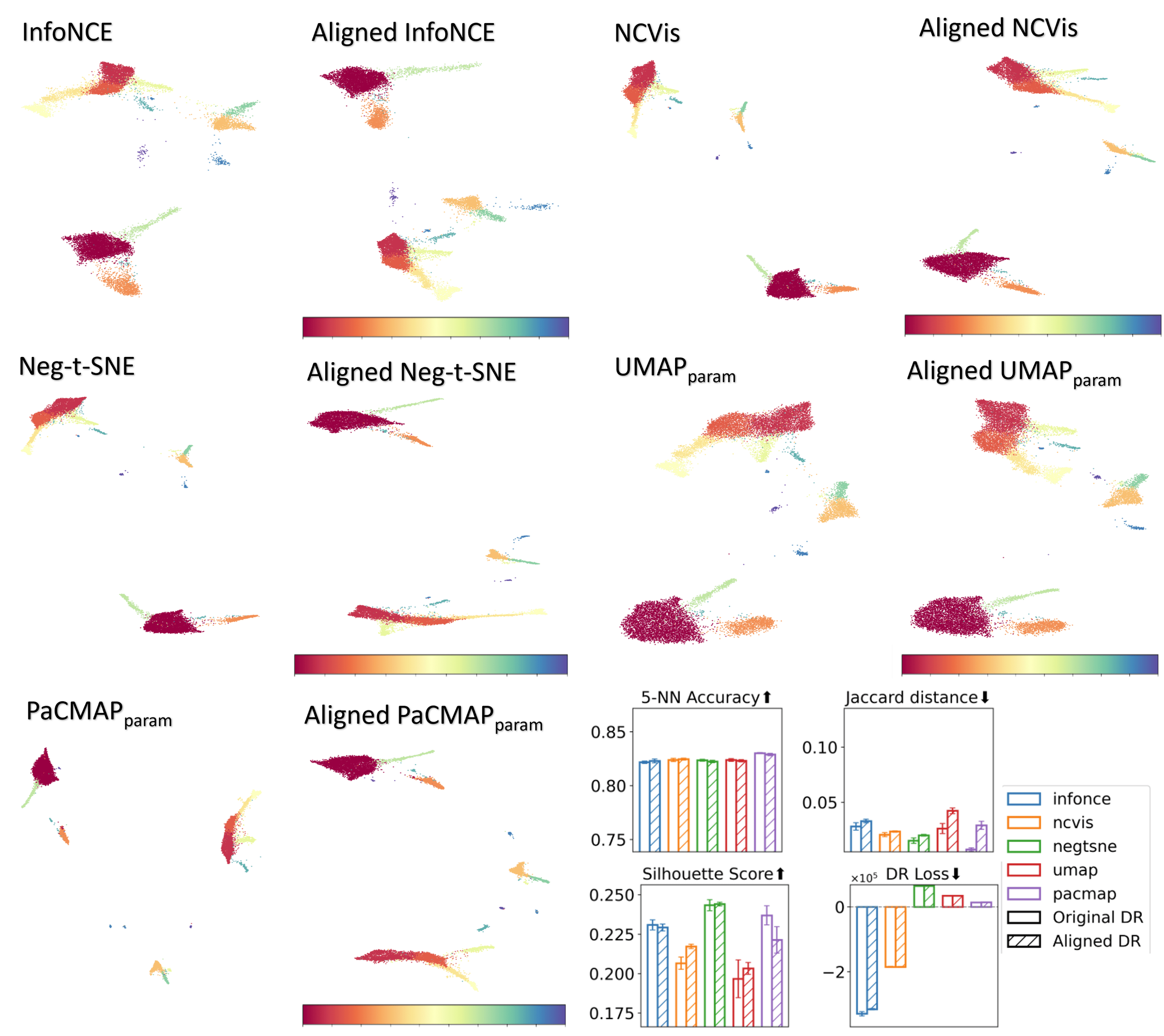}
    \caption{Comparison of original and aligned embeddings for Kang et al$.$ dataset using a concept-aware regularizer. The embeddings' horizontal axis is aligned with the cell type index, and the embedding performance hasn't been changed significantly.}
    \label{fig:kang_axis}
\end{figure}

\begin{figure}[ht]
    \centering
    \includegraphics[width=0.8\linewidth]{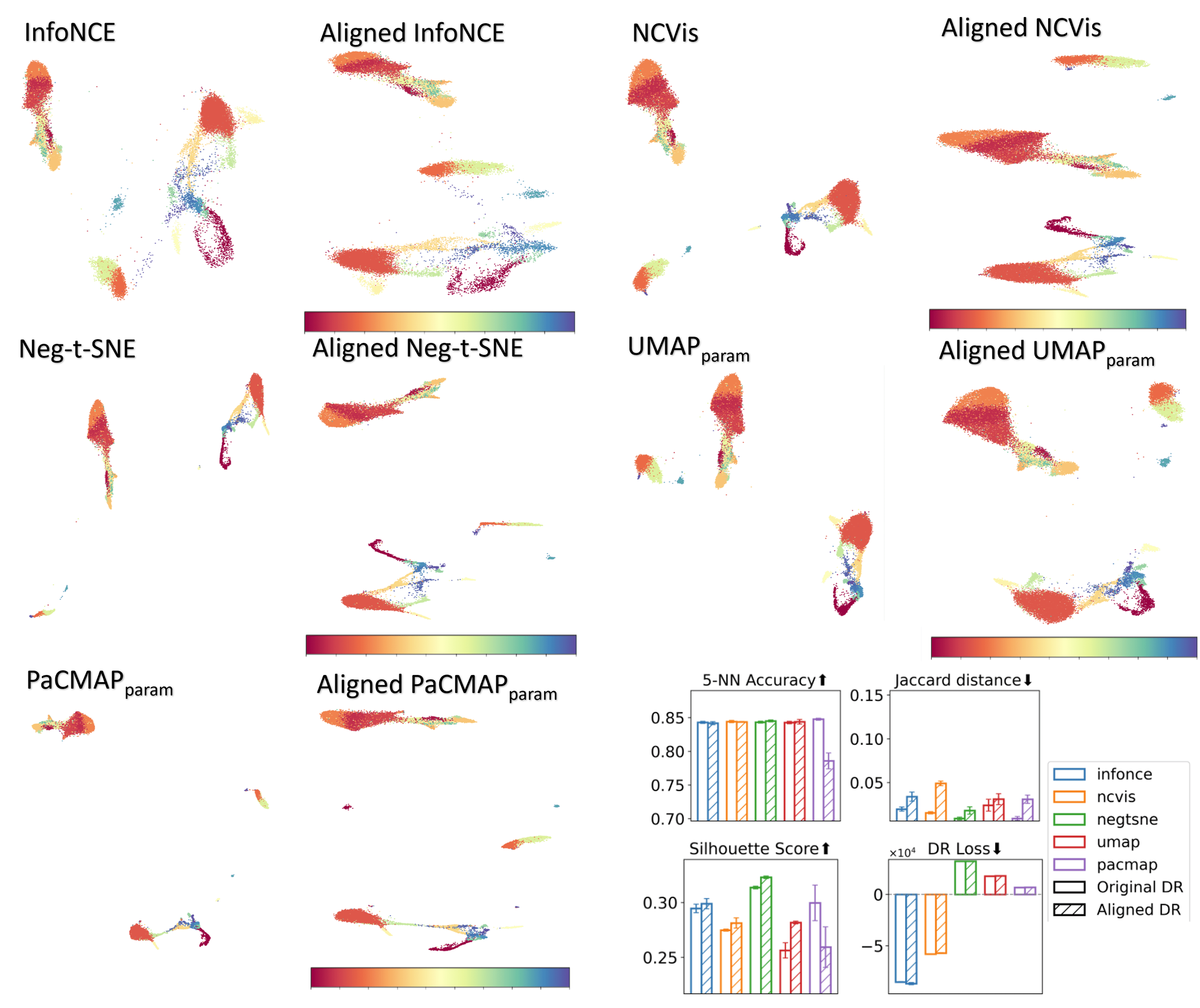}
    \caption{Comparison of original and aligned embeddings for the Stuart dataset using a concept-aware regularizer. The embeddings' horizontal axis is aligned with the cell type index, and the embedding performance hasn't been changed significantly.}
    \label{fig:seurat_axis}
\end{figure}

\begin{figure}
    \centering
    \includegraphics[width=0.8\linewidth]{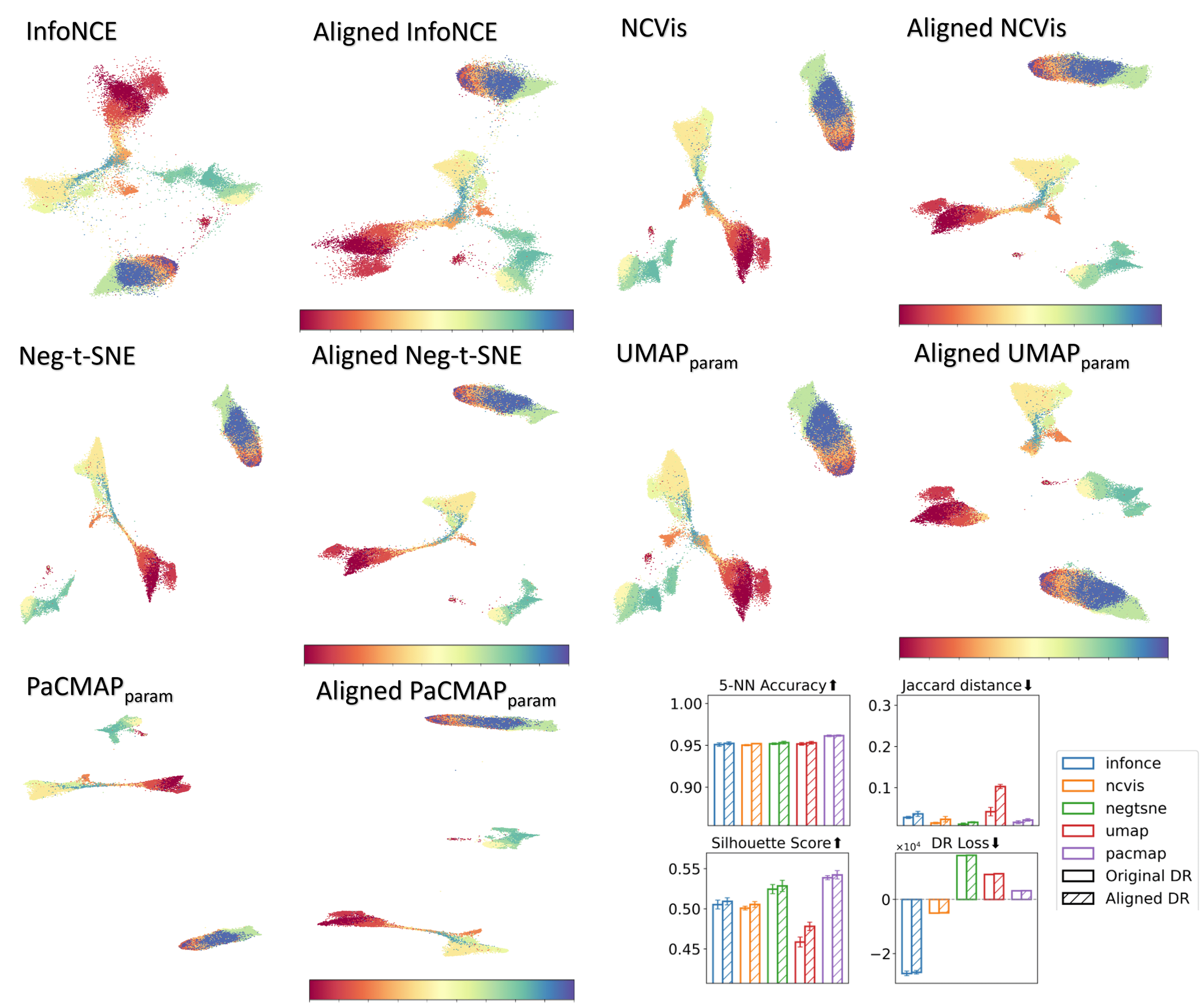}
    \caption{Comparison of original and aligned embeddings for the CMBC dataset using a concept-aware regularizer. The embeddings' horizontal axis is aligned with the cell type index, and the embedding performance hasn't been changed significantly.}
    \label{fig:cbmc_axis}
\end{figure}

\begin{figure}[ht]
    \centering
    \includegraphics[width=0.8\linewidth]{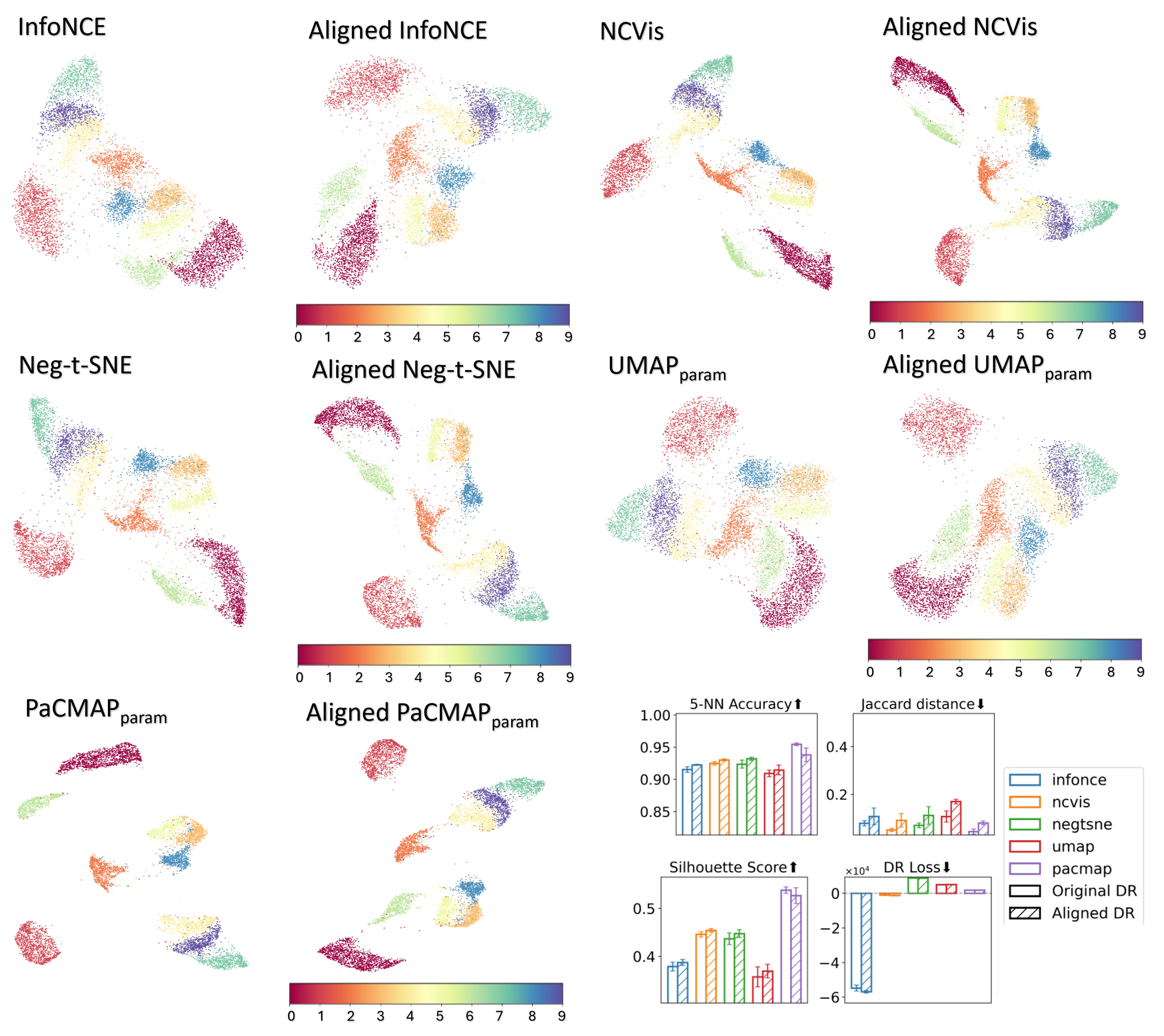}
    \caption{Comparison of original and aligned embeddings for USPS dataset using a concept-aware regularizer. Clusters are shown to align with the digits index and structure preservation across multiple DR methods.}
    \label{fig:usps_axis}
\end{figure}

\clearpage
\section{Influence of Missingness Ratio and Label Weights on Embedding Structure}
\label{app:missing_ratio}

Figure~\ref{fig:missing_label_grid_MNIST} illustrates how the embedding evolves across varying missingness ratios (rows) and label alignment weights (columns). From left to right, we increase the label weight $\lambda_{\text{Axis}}$, while from top to bottom, the missingness ratio increases. When the missingness ratio is low (top rows), moderate increases in $\lambda_{\text{Axis}}$ produce embeddings that gradually align with the label order while preserving local structure. However, as the missingness ratio increases (moving down the rows), the embeddings become increasingly sensitive to the label weight. In particular, when both $\lambda_{\text{Axis}}$ and missingness ratio are high (bottom-right corner), the embedding quality sharply degrades—the original structure breaks down, resulting in distorted or fragmented clusters. This shows that under sparse supervision, overly strong label alignment can dominate and disrupt the geometry learned from the  intrinsic structure of the data.

Figure~\ref{fig:missing_ratio_label_relationship} illustrates the relationship between $\L_\DR$ and the axis alignment weight $\lambda_{\text{Axis}}$ under varying levels of missingness. As the missingness ratio increases, the threshold of $\lambda_{\text{Axis}}$ at which the loss and distance scores begin to rise significantly becomes lower. This trend confirms our earlier qualitative observations: with higher missingness, the concept-based regularizer increasingly dominates the training objective, leading to a degradation of the embedding structure. This curve trend comparison between the$\L_\DR$ and the Jaccard distance also shows that they have similar trend.

\begin{figure}[ht]
    \centering
    \includegraphics[width=\linewidth]{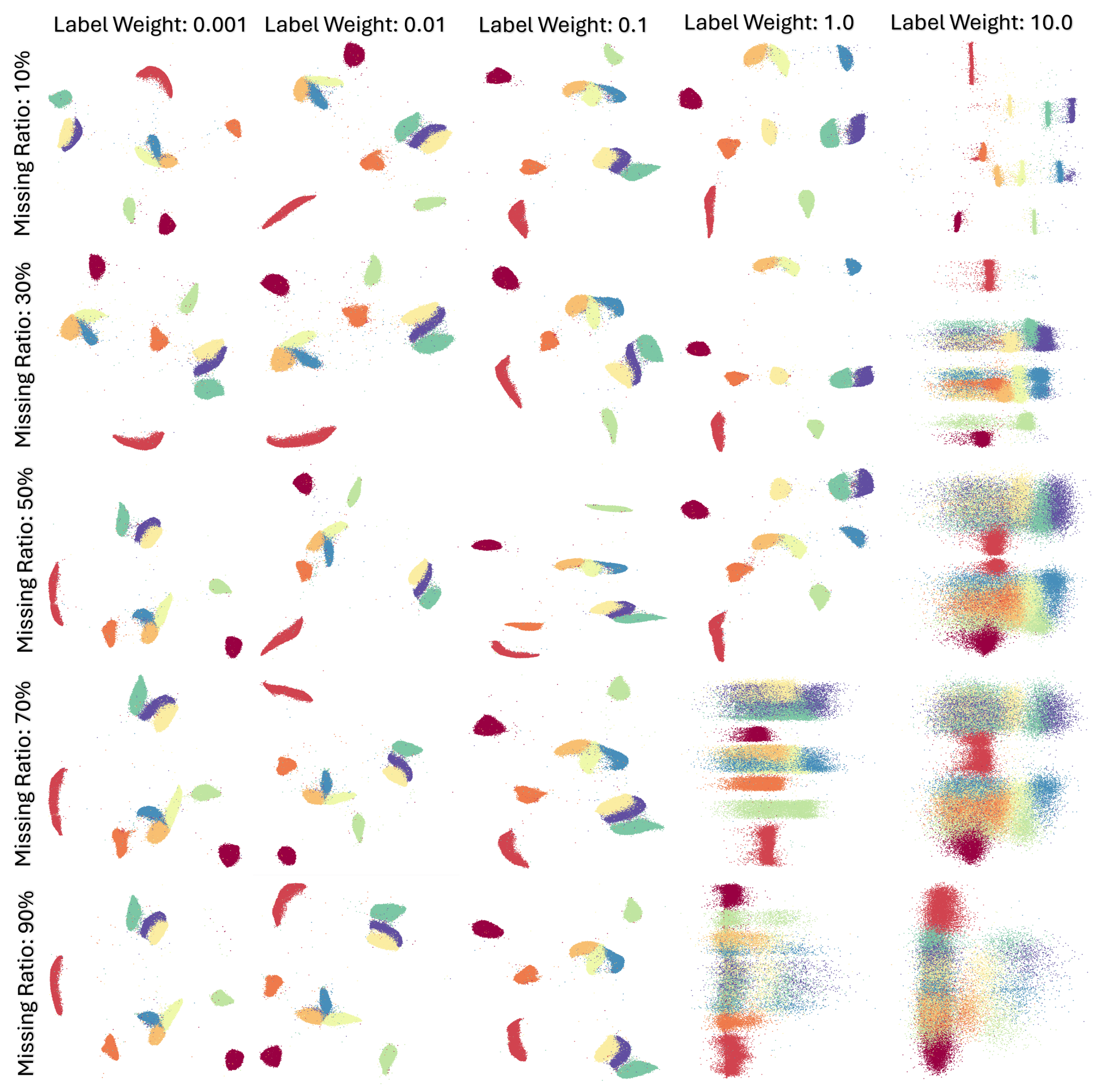}
    \caption{MNIST \PaCMAP Embedding under different label missingness ratios (rows) and label weights (columns). High label weight with high missingness ratio breaks the original structure.}
    \label{fig:missing_label_grid_MNIST}
\end{figure}

\begin{figure}
    \centering
    \includegraphics[width=\linewidth]{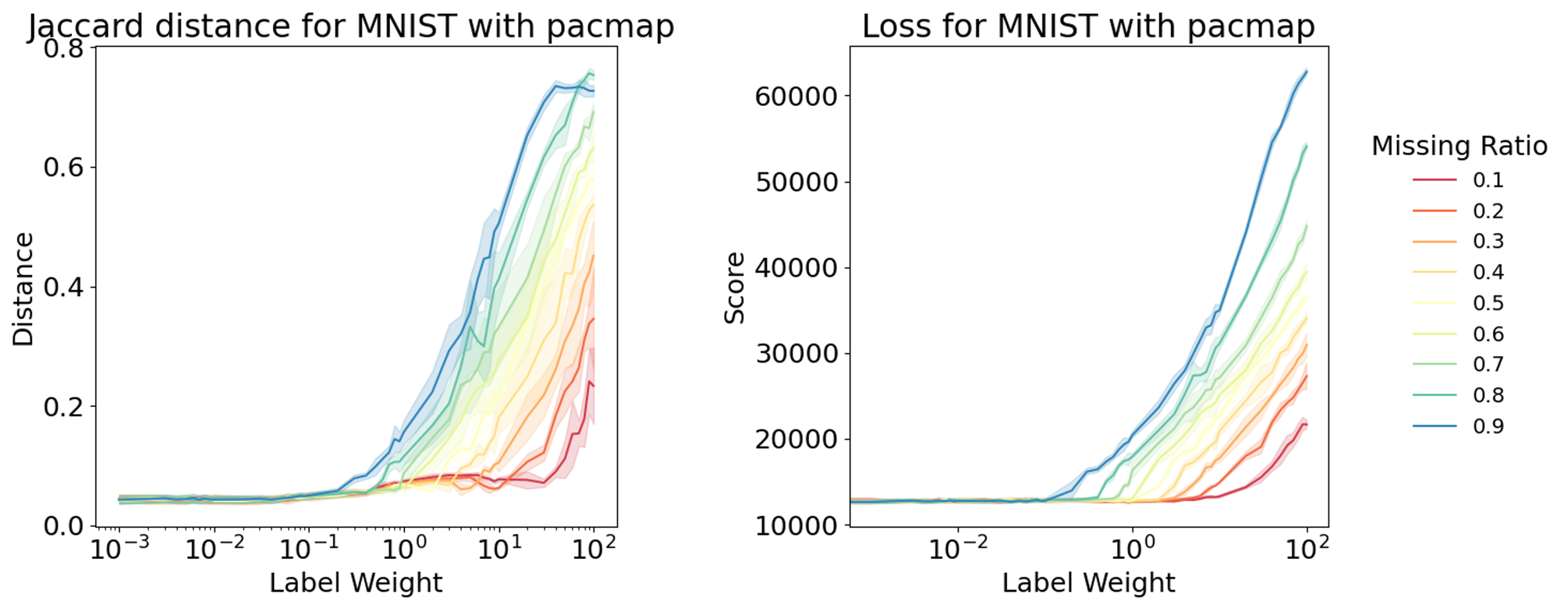}
    \caption{$\L_\DR$ and Jaccard Distance under different label missingness ratios. }
    \label{fig:missing_ratio_label_relationship}
\end{figure}

\clearpage
\section{Common Knowledge Extraction and Aggregation Results}
\label{app:common}

To evaluate the effectiveness of our common knowledge extraction mentioned in Section \ref{sec:common_knowledge}, we compile results across different datasets, embedding methods, and three evaluation metrics, which are in the same format as they are shown in Section \ref{sub:experiment_common_knowledge}. The following figures present a comprehensive overview of how the proposed approach performs under varying data characteristics and dimensionality reduction settings. For each dataset, we report one of the embeddings obtained with different DR methods, along with relevant metrics: k-NN Accuracy, Silhouette Score, and SVM Accuracy. These results demonstrate the generalizability and robustness of our method across tasks.

\begin{figure}[ht]
    \centering
    \includegraphics[width=\linewidth]{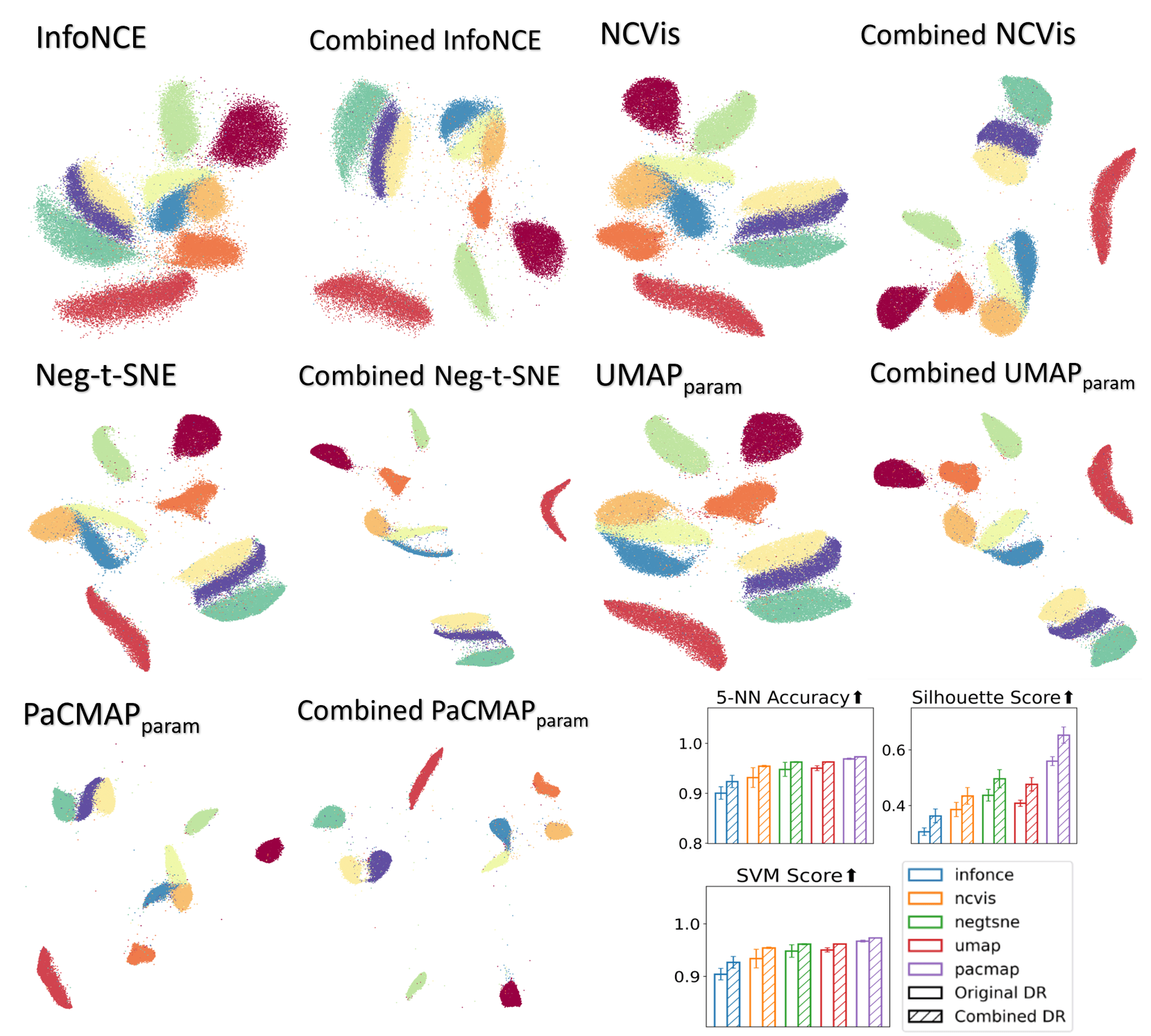}
    \caption{MNIST embeddings and improved embedding using common knowledge. Combined DR improves structure and boosts downstream metrics.}
    \label{fig:common_MNIST}
\end{figure}

\begin{figure}[ht]
    \centering
    \includegraphics[width=\linewidth]{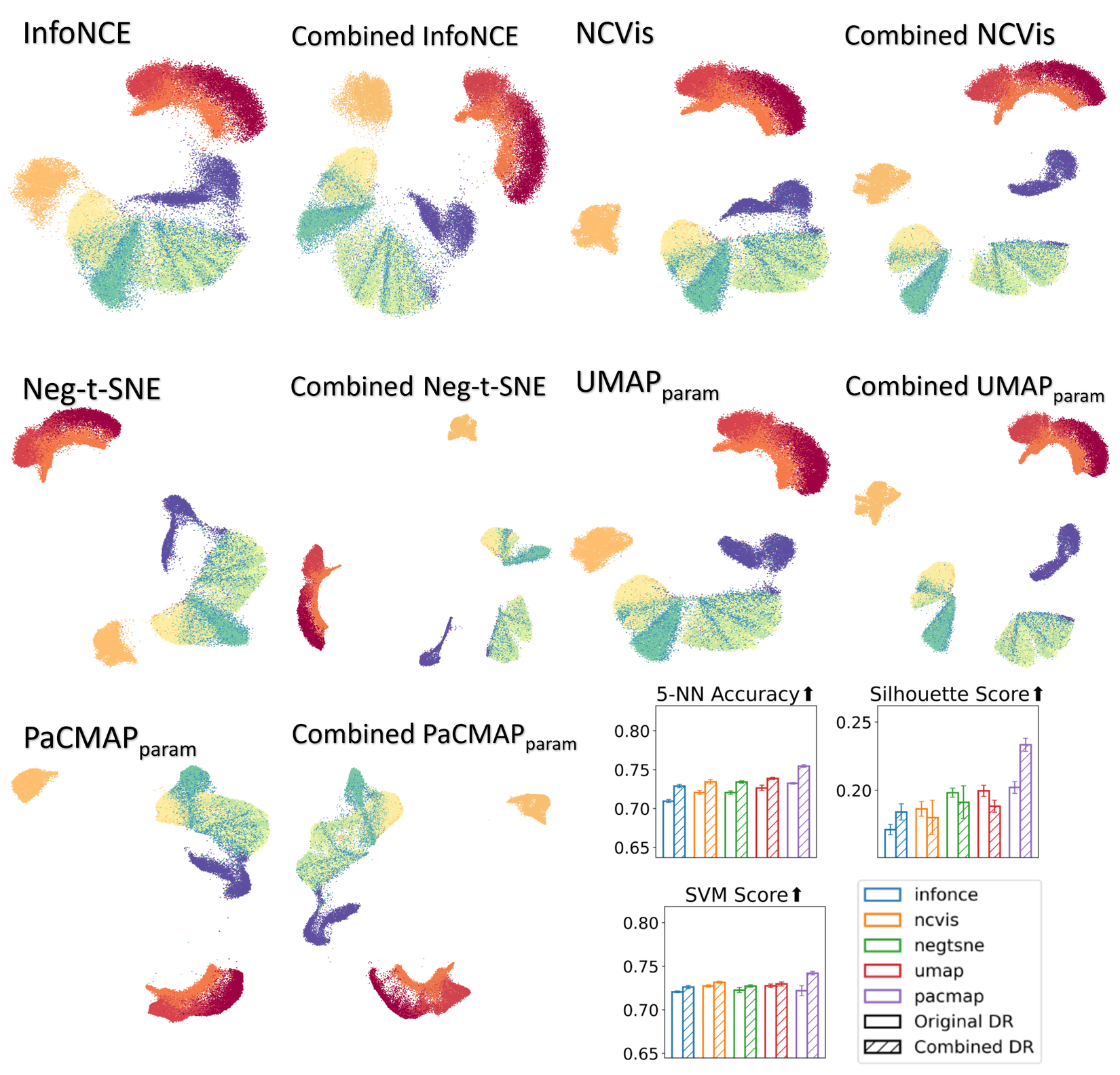}
    \caption{FMNIST embeddings and improved embedding using common knowledge. Combined DR improves structure and boosts downstream metrics. NCVis and Neg-t-SNE and UMAP are showing a bit lower silhouette score due to the mixed area at the bottom right corner}
    \label{fig:common_FMNIST}
\end{figure}

\begin{figure}[ht]
    \centering
    \includegraphics[width=\linewidth]{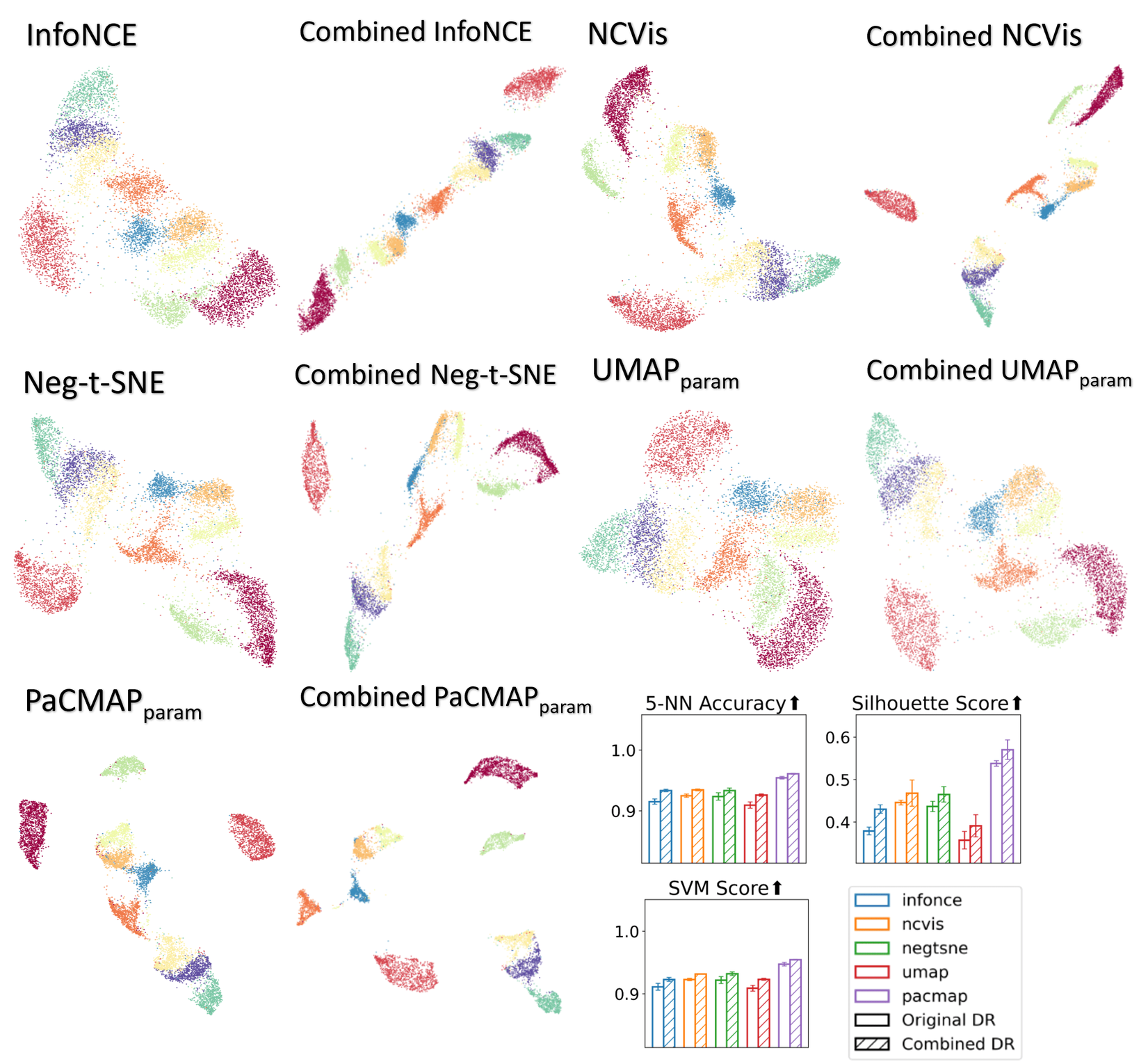}
    \caption{USPS embeddings and improved embedding using common knowledge. Combined DR improves structure and boosts downstream metrics.}
    \label{fig:common_USPS}
\end{figure}

\begin{figure}[ht]
    \centering
    \includegraphics[width=\linewidth]{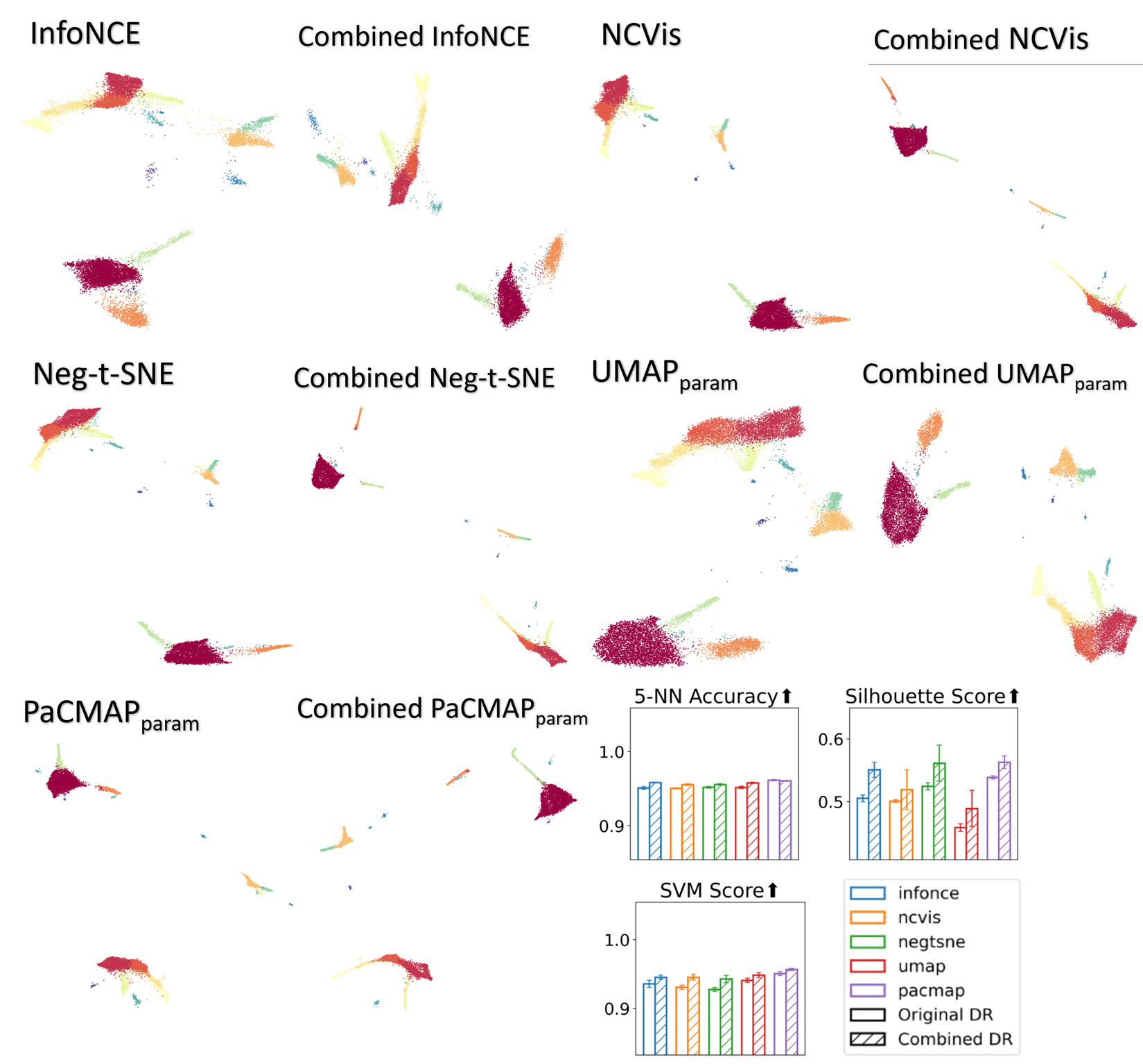}
    \caption{Kang et al. embeddings and improved embedding using common knowledge. Combined DR improves structure and boosts downstream metrics.}
    \label{fig:common_kang}
\end{figure}

\begin{figure}[ht]
    \centering
    \includegraphics[width=\linewidth]{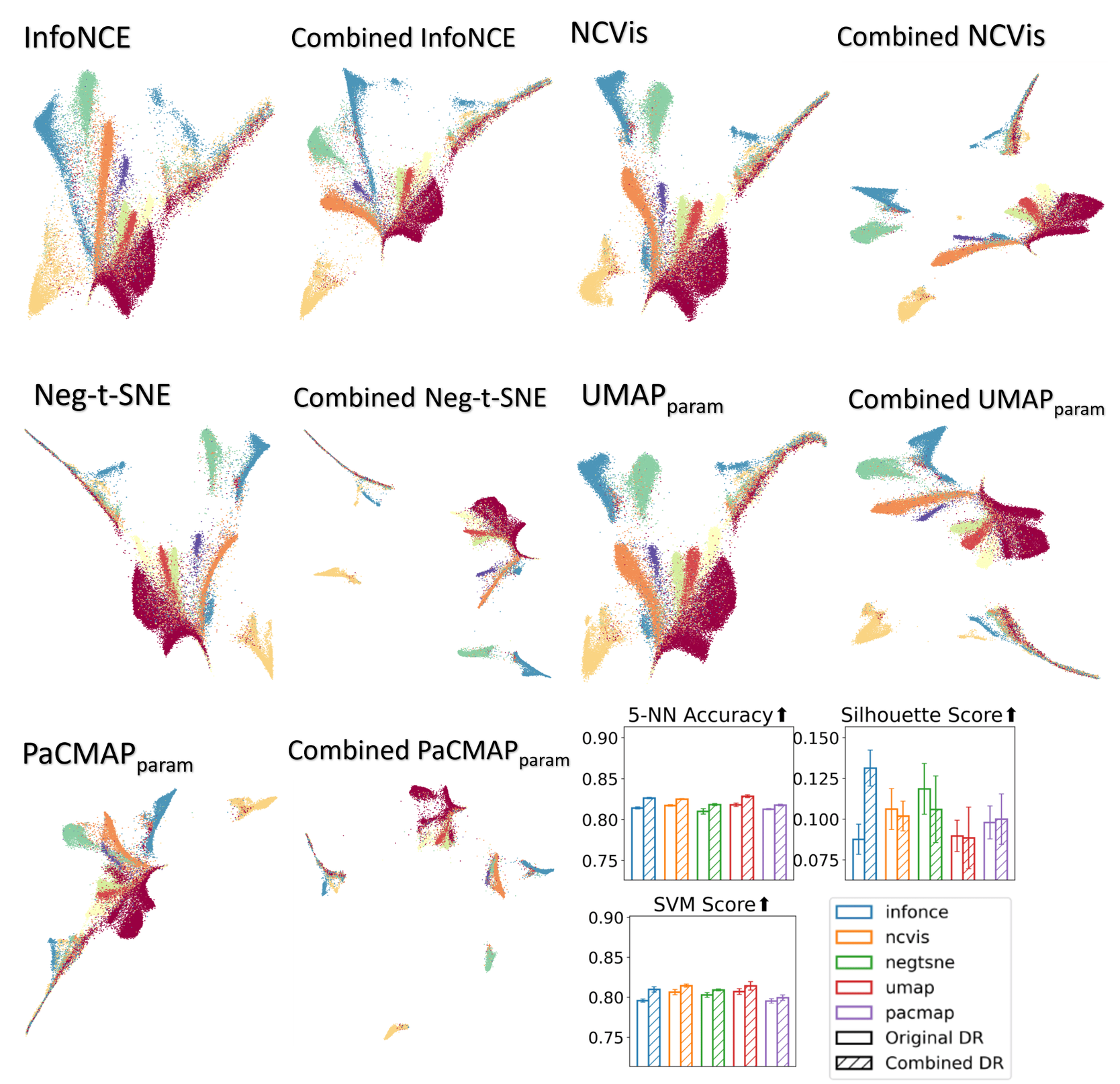}
    \caption{Human Cortex embeddings and improved embedding using common knowledge. Combined DR improves structure and boosts downstream metrics. NCVis and Neg-t-SNE and UMAP are showing a bit lower silhouette score due to the mixed area at the bottom right corner}
    \label{fig:common_human_cortex}
\end{figure}

\begin{figure}[ht]
    \centering
    \includegraphics[width=\linewidth]{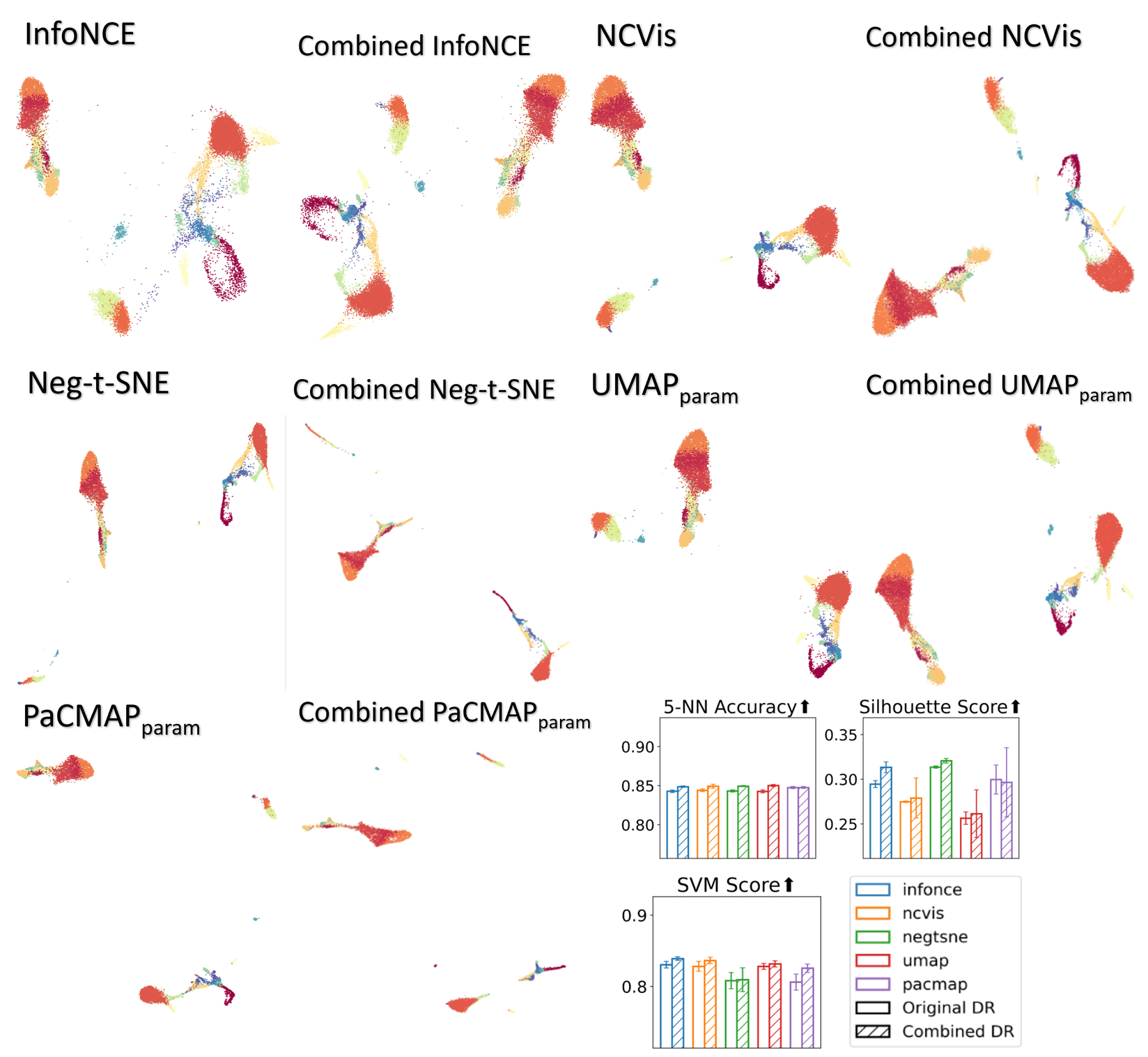}
    \caption{Stuart et. al. embeddings and improved embedding using common knowledge. Combined DR improves structure and boosts downstream metrics. }
    \label{fig:common_seurat}
\end{figure}

\begin{figure}[ht]
    \centering
    \includegraphics[width=\linewidth]{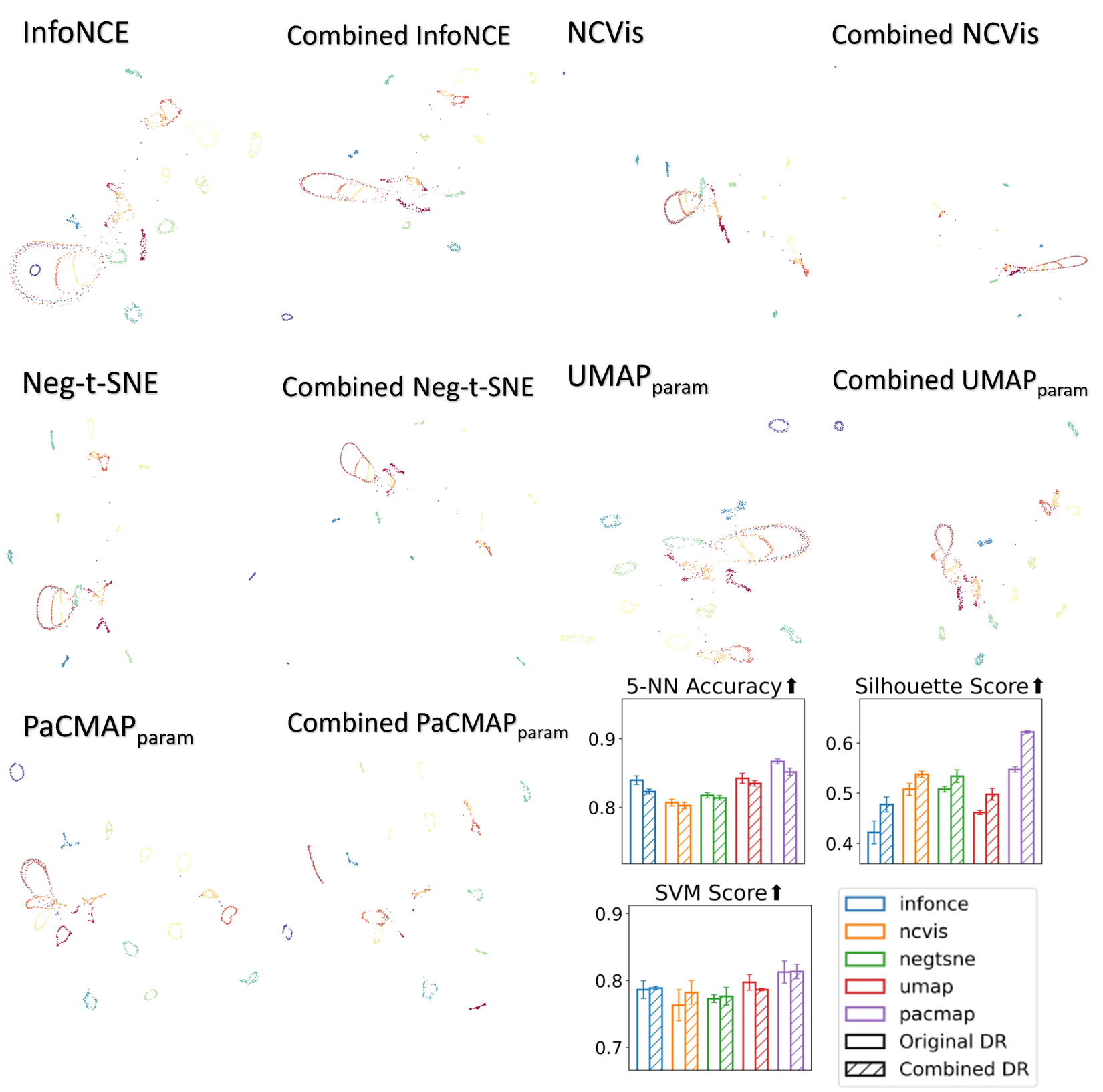}
    \caption{COIL20 embeddings and improved embedding using common knowledge. Combined DR improves structure and boosts downstream metrics.}
    \label{fig:common_coil20}
\end{figure}

\end{document}